\documentclass{article}
\PassOptionsToPackage{numbers,compress,sort&compress}{natbib}
\usepackage[main,final]{neurips_2026}
\renewcommand{\acksection}{\section*{Acknowledgements}}

\usepackage[utf8]{inputenc} % allow utf-8 input
\usepackage[T1]{fontenc}    % use 8-bit T1 fonts
\usepackage{url}            % simple URL typesetting

\usepackage{microtype}
\usepackage{graphicx}
\graphicspath{{./}}
\usepackage{subcaption}
\usepackage{booktabs} % for professional tables

\usepackage[dvipsnames,table]{xcolor}

\usepackage{hyperref}

\usepackage{algorithm}
\usepackage{algorithmic}

\usepackage{amsmath}
\usepackage{amssymb}
\usepackage{mathtools}
\usepackage{amsthm}

\usepackage{color-edits}
\addauthor{yz}{Apricot} % Yizhou's comments
\addauthor{ab}{red}
\addauthor{sw}{blue}

\usepackage{float}
\usepackage{nicefrac}       % compact symbols for \frac{1}{2}, etc.
\usepackage{enumitem}

\definecolor{blue}{HTML}{000080}
\hypersetup{
    colorlinks=true,
    linkcolor=blue,
    filecolor=blue,
    urlcolor=blue,
    citecolor=blue,
}

\DeclareRobustCommand{\rev}[1]{#1}

\usepackage{multirow}
\usepackage{diagbox}
\usepackage{pifont}
\usepackage{adjustbox}
\usepackage{array}
\newcolumntype{R}[1]{>{\raggedleft\arraybackslash}p{#1}}

\newcommand{\fastclock}{\raisebox{-0.2\baselineskip}{\includegraphics[height=1.1em]{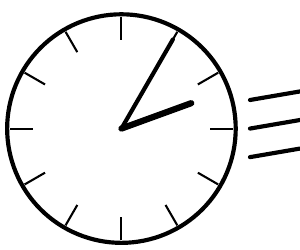}}}
\newcommand{\slowclock}{\raisebox{-0.2\baselineskip}{\includegraphics[height=1.1em]{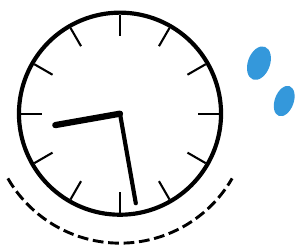}}}

\definecolor{blkA}{RGB}{255,228,225}  % Misty Rose (slightly deeper)
\definecolor{blkB}{RGB}{218,247,232}  % Aquamarine (slightly deeper)
\definecolor{blkC}{RGB}{230,244,251}  % Alice Blue (slightly deeper)
\definecolor{blkD}{RGB}{255,249,214}  % Light Goldenrod (slightly deeper)
\definecolor{blkE}{RGB}{246,233,255}  % Lavender Purple (slightly deeper)
\definecolor{blkF}{RGB}{255,182,120}  % Coral
\definecolor{mydeepblue}{HTML}{0000FF}

\newcommand{\gaussmark}{\textsc{GaussMark}}
\newcommand{\marktune}{\textsc{MarkTune}}
\newcommand{\cX}{\mathcal{X}}
\newcommand{\cY}{\mathcal{Y}}
\newcommand{\cV}{\mathcal{V}}

\newcommand{\hnull}{\mathbf{H_0}}
\newcommand{\halt}{\mathbf{H_A}}

\newcommand{\pref}{p_{\theta_\mathrm{ref}}}
\newcolumntype{L}{>{\raggedright\arraybackslash}p{3.2cm}}
\newcommand{\err}[1]{_{\scriptscriptstyle #1}}

\usepackage{wrapfig}
\usepackage{needspace}
\usepackage{placeins}
\usepackage{tikz}
\usepackage[most]{tcolorbox}
\usepackage{dashrule}
\usepackage{pgfplots}
\pgfplotsset{compat=1.18}
\usetikzlibrary{decorations.pathreplacing}

\usetikzlibrary{arrows.meta, positioning}
\usepackage[capitalize,noabbrev]{cleveref}

\theoremstyle{plain}
\newtheorem{theorem}{Theorem}[section]
\newtheorem{proposition}[theorem]{Proposition}
\newtheorem{lemma}[theorem]{Lemma}

\theoremstyle{definition}

\theoremstyle{remark}
\newtheorem{remark}[theorem]{Remark}

\usepackage[disable,textsize=tiny]{todonotes}

\title{MarkTune: Improving the Quality-Detectability Trade-off in Model-Embedded LLM Watermarking}

\author{
   Yizhou Zhao\\
  Georgia Institute of Technology\\
  University of Pennsylvania\\
  \texttt{yzhao845@gatech.edu}\\
  \And
  Zhiwei Steven Wu\\
  Carnegie Mellon University\\
  \texttt{zstevenwu@cmu.edu}\\
  \And
  Adam Block\\
  Columbia University\\
  \texttt{adam.block@columbia.edu}\\
 }
\begin{document}

  \maketitle

\begin{abstract}
  Language model watermarking schemes fall into two broad categories: \emph{inference-time} methods, which modify the decoding process at generation time, and \emph{model-embedded} methods, which encode a secret signal directly into the model weights. Inference-time watermarking methods incur additional inference overhead and are not applicable in the increasingly prevalent open-weight model settings. In contrast, model-embedded watermarks introduce no additional inference latency---generation uses the standard sampling pipeline---and are particularly well-suited to open-weight settings. However, existing model-embedded approaches such as \gaussmark\ face a fundamental quality-detectability trade-off: achieving strong detection power typically requires weight perturbations that noticeably degrade generation quality. Our key insight is that regularizing toward a high-quality reference distribution decouples quality preservation from proximity to the base model, enabling a better quality-detectability trade-off. Building on this insight, we introduce \marktune, a theoretically grounded on-policy fine-tuning framework that uses the \gaussmark\ detection statistic as a reward while preserving generation quality through cross-entropy on high-quality human text or KL divergence to a stronger oracle model. Empirically, \marktune\ substantially improves the quality-detectability frontier of \gaussmark, approaching the detectability of strong inference-time schemes while preserving generation quality and downstream task performance. \marktune\ also remains robust under the evaluated text-editing and open-weight attacks and generalizes across datasets, retaining substantial detection power on corpora unseen during training.
  Code is available at \url{https://github.com/zhyzmath/MarkTune-LLM-Watermarking}.
\end{abstract}

% \vspace{-1.0em}

\section{Introduction} \label{sec:intro}

% \vspace{-0.5em}

 Language model watermarking aims to determine whether a given text was generated by a language model, a capability that is increasingly critical for maintaining academic integrity~\citep{2022ai,2023chatgpt,2023large}, mitigating large-scale misinformation~\citep{2019defending,2020deep,2021ethical}, and protecting intellectual property in generative AI systems~\citep{2024generative,2024intellectual,xu2025copyright}. Most prior work formalizes watermarking as a statistical hypothesis testing problem~\citep{2023towards,2025gaussmark,2025statistical}, assuming a joint distribution over text and a secret watermarking key. Under this formulation, the goal is to generate text conditioned on a key such that the resulting dependence is reliably detectable, while satisfying quality constraints on the generated text. These constraints are frequently imposed via strict information-theoretic notions of non-distortion~\citep{2023Robust,hu2023unbiased,christ2024undetectable}, for example by requiring watermarked and unwatermarked text distributions to be close in total variation.

Language model watermarking schemes fall into two broad categories:
Inference-time schemes are primarily designed for \emph{API-based deployment}, where the service provider controls the decoding process and can achieve strong statistical detectability with minimal generation quality degradation~\citep{2023Aaronson,2023KGW,2024scalable}.  However, these methods introduce additional decoding overhead. Moreover, with the increasingly prevalent use of \emph{open-weight} models~\citep{2024llama,2025qwen3,guo2025deepseek}, users can freely alter or bypass the decoding pipeline, making inference-time watermarking inapplicable. Model-embedded schemes avoid these limitations by embedding the watermark directly into the model weights, thereby enabling standard generation without additional inference latency and making them especially well suited to open-weight settings.

One recent model-embedded watermarking scheme that operates directly in weight space is \gaussmark~\citep{2025gaussmark}, which adds keyed Gaussian noise to a subset of weight matrices, inducing a subtle but detectable shift in the output distribution. Unlike inference-time schemes based on hashed prefix keys, \gaussmark\ uses a single Gaussian weight-space key and full-sequence detection, avoiding the security--robustness trade-off~\citep{pang2024nofreelunch}. \citet{2025gaussmark} showed that, with sufficiently small noise variance and carefully chosen parameters, \gaussmark\ achieves nontrivial detectability with negligible text-quality degradation. These results establish \gaussmark\ as a promising approach, but leave open how close it is to the optimal quality--detectability frontier and whether this trade-off can be improved through fine-grained, adaptive modifications to model weights. \rev{Prior work has also studied \gaussmark\ under continued fine-tuning attacks~\citep{2025towards}, motivating us to examine whether on-policy signal amplification preserves robustness when the released weights are subsequently updated.}

\begin{figure*}[!t]
  \centering
  \includegraphics[width=\textwidth]{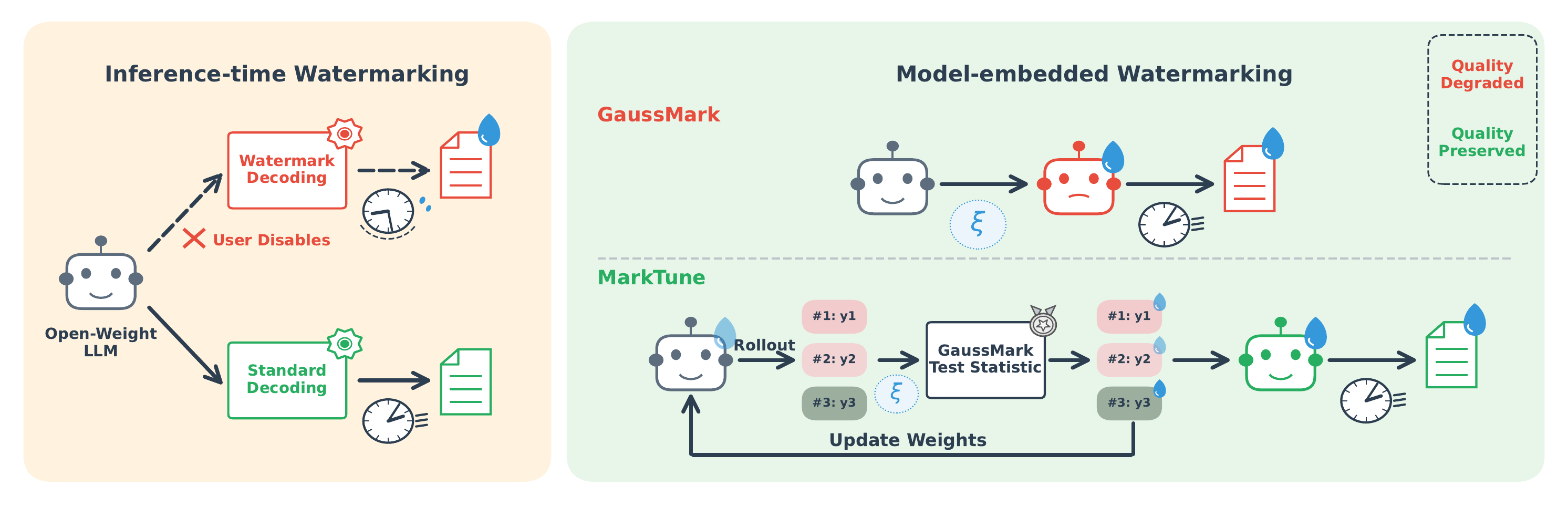}
  % \vspace{-1.8em}
  \caption{
  Overview of our framework compared to prior work~\citep{2023Aaronson,2023KGW,2024scalable,2025gaussmark}. \textbf{Left}: Inference-time watermarking schemes break down on open-weight LMs because users can disable the decoding algorithm, and these methods often introduce substantial generation latency (\fastclock\ indicates no latency; \slowclock\ indicates extra latency). \textbf{Right}: Our approach, \marktune, treats the \gaussmark\ test statistic as a reward and performs on-policy fine-tuning to embed a highly detectable yet quality-preserving watermark signal into the model's weights.
  }
   \label{fig:method_ill}
  % \vspace{-1.0em}
\end{figure*}

We thus ask---\emph{Can we design a watermarking scheme for open-weight LMs that preserves text quality and downstream task performance, achieves detection power comparable to inference-time watermarking, and remains robust to realistic text-editing and open-weight attacks?}

\paragraph{Our contribution.} We answer the question above by proposing \marktune, a theoretically principled on-policy fine-tuning framework that closes the quality-detectability gap between model-embedded and inference-time watermarking. Figure~\ref{fig:method_ill} contrasts our design with prior inference-time schemes. Our contributions are threefold:
\begin{itemize}[leftmargin=1.0em]
\item \emph{Methodology.} \marktune\ treats the detection statistic as an intrinsic watermark-signal reward and pairs it with a regularizer that anchors the trained model to an external high-quality distribution rather than to the base model, so that amplifying the watermark power and preserving generation quality are no longer antagonistic.
\item \emph{Theory.} We characterize a fundamental quality–detectability trade-off for \gaussmark: the perturbation strength and watermark subspace dimension jointly govern the quality cost and detection power. The \marktune\ detector maintains exact Type-I error control. In a stylized linear-softmax regime, we further establish a detection-power bound for \marktune\ and explain its improved trade-off through the lens of local optimization geometry.
\item \emph{Empirics.} On two open-weight LMs and two generation tasks, \marktune\ consistently improves the quality-detectability trade-off of \gaussmark\ and approaches the detectability of strong inference-time watermarks at comparable text quality. It preserves downstream-task performance, generalizes across datasets, and remains robust to realistic text-editing and open-weight attacks.
\end{itemize}

\subsection{\rev{Related Work}}\label{sec:related_work}
\rev{Watermarking provides a proactive approach to identifying AI-generated content by embedding detectable signals while preserving content quality. Recent work spans text~\citep{2023KGW,hu2023unbiased,2024ProvableRobust,christ2024undetectable,2024scalable,2024gu,2024provably,2025gaussmark}, images~\citep{fernandez2023stable,wen2023treerings,yang2024gaussian,ci2024ringid,feng2024aqualora,jiang2024certifiably,bui2025trustmark,sander2025watermark}, video~\citep{luo2023dvmark,zhang2023revmark,fernandez2024videoseal,hu2025videoshield,hu2025videomark}, audio~\citep{chen2023wavmark,sanroman2024proactive,liu2024timbre,singh2024silentcipher,sanroman2025latent}, and tabular data~\citep{he2024watermarking,zheng2024tabularmark,zhu2025tabwak,fang2026muse,zhao2026robust}. Across these modalities, watermarking schemes must balance detectability and robustness with perceptual fidelity or downstream utility. Signals can be embedded through changes to model components, generative sampling, or the generated artifacts themselves, depending on the modality and deployment setting.}

LM text watermarking methods fall into two broad categories: \emph{inference-time} and \emph{model-embedded} watermarking. Distortionary inference-time schemes modify the sampling process—for example, by biasing next-token selection toward a partitioned ``green list''~\citep{2023KGW,2024ProvableRobust}. While these methods offer statistical guarantees, they introduce noticeable text distortion and are vulnerable to realistic paraphrasing attacks~\citep{2024revisiting,2024watermarkstealing}. Nondistortionary inference-time schemes instead embed watermarks by controlling the pseudorandomness of sampling while preserving the original distribution. Examples include deterministic decoders based on the Gumbel-max trick~\citep{2023Aaronson} and inverse transform sampling~\citep{2023Robust}, as well as unbiased variants of KGW using maximal coupling~\citep{2024debiasing} or reweighting~\citep{2023dipmark}. However, these approaches face challenges in large-scale deployment due to increased generation latency~\citep{2025gaussmark} and failures to preserve the joint distribution under key collisions~\citep{wu2024distortion}. Tournament-based watermarking mitigates latency overhead~\citep{2024scalable}, but requires storing additional generated tokens, which limits its practicality.

\emph{Model-embedded} watermarking directly encodes signals into model weights and can be further divided into training-based~\citep{xu2024learning,2024gu} and weight-editing approaches~\citep{2025gaussmark,2024provably}. These methods are well suited to open-weight LMs and incur neither generation latency nor storage overhead. However, training-based schemes have limited generalization and lack rigorous statistical guarantees~\citep{2024gu,xu2024learning}, while existing weight-editing methods either require architectural changes~\citep{2024provably} or depend on costly hyperparameter searches with limited gains in the quality-detectability trade-off~\citep{2025gaussmark}.

% \vspace{-0.5em}
\section{Preliminaries} \label{sec:preliminary}
% \vspace{-0.5em}
% A language model is any conditional distribution mapping a \emph{prompt} $x \in \cX$ (the space of prompts) to a distribution over responses $y \in \cY$ (the space of responses), i.e. a function $p: \cX \to \Delta(\cY)$.  As is common in language modeling, we consider \emph{autoregressive generation} over a vocabulary $\cV$. We therefore model both inputs and outputs as \emph{finite token sequences} drawn from $\cV$, i.e., $\cX, \cY \subseteq \cV^\ast$. In this case, the model generates a response one token at a time by sampling $y_1 \sim p(\cdot | x)$, then $y_t \sim p(\cdot | x,y_1, \dots, y_{t-1})$ and concatenating the output tokens to form a response.  As we are chiefly concerned with transformer instantiations of language models, we generally parameterize the model by some set of weights $\Theta \subset \rr^d$ and write $p_\theta$ for the resulting model.  Typically, in the case of transformers, $\theta \in \Theta$ can be thought of as the concatenation of a large number of high-dimensional matrices, one for each layer of the transformer.

A language model is any conditional distribution mapping a \emph{prompt} $x \in \cX$ (the space of prompts) to a distribution over responses $y \in \cY$ (the space of responses), i.e. a function $p: \cX \to \Delta(\cY)$.
As is common in language modeling, we specialize to the autoregressive generation setting, where prompts and responses are both represented as finite sequences of tokens from some vocabulary $\cV$, so that $\cX, \cY \subseteq \cV^\ast$.
In this case, the model generates a response one token at a time by sampling $y_1 \sim p(\cdot \mid x)$, then $y_t \sim p(\cdot \mid x, y_1, \dots, y_{t-1})$ and concatenating the output tokens to form a response.
Since our primary focus is on transformer-based language models, we parameterize the model with weights $\theta \in \Theta \subset \mathbb{R}^d$ and denote the resulting model as $p_\theta$.

\paragraph{Hypothesis testing.}  As in \citet{2025gaussmark, 2023towards,2025statistical}, we formalize the notion of watermarking as a statistical hypothesis testing problem. Let $\Xi$ denote the key space, from which the secret watermarking key $\xi$ is sampled. Recall that a hypothesis testing problem consists of an observation space $\Xi \times \cY$ and two disjoint collections of distributions on the observation space, $\hnull$ and $\halt$. A test is a (possibly randomized) function $\phi: \Xi \times \cY \to \{0,1\}$, where $\phi(\xi,y) = 1$ indicates that the observation $(\xi, y)$ provides sufficient evidence to suggest that it was not sampled from any distribution in $\hnull$.  The test is said to have level $\alpha$ if the false positive rate, the probability that $\phi = 1$ even when $(\xi,y)$ is sampled from an element of the null hypothesis, is at most $\alpha$.  The power of the test, $1 - \beta$, is the probability that $\phi = 1$ when $(\xi,y)$ is truly sampled from an element of the alternative hypothesis. Clearly we wish to have a test with both $\alpha$ and $\beta$ as small as possible.

\paragraph{Weight-editing watermarking.}
There are three phases in a weight-editing watermarking scheme: embedding, generation and detection. Formally, let $\Xi$ be endowed with an associated distribution $\rho$. Given a language model $p_\theta$, we first sample a key $\xi \sim \rho$ and obtain a watermarked model $p_{\theta(\xi)}$ by applying a weight-editing algorithm $\mathcal{A} : \Theta \times \Xi \rightarrow \Theta$, so that $\theta(\xi) = \mathcal{A}(\theta, \xi)$. The watermarked text is then generated by sampling $y \sim p_{\theta(\xi)}(\cdot \mid x)$ using a prompt $x$. Detection is formulated as a hypothesis test, where $\hnull = \left\{ \rho \otimes q| q \in \Delta(\cY) \right\}$ denotes the set of joint distributions under which the key $\xi$ and the text $y$ are independent, and $\halt$ corresponds to the joint distribution induced by the watermarking generation process.\footnote{As stated, the detector has access to the prompt $x$ used to generate $y$.  In practice, this is of course not the case and \textbf{our empirical results do not rely on this access}.}
The fundamental statistical difficulty of distinguishing $\halt$ from $\hnull$ is governed by the total variation (TV) distance between the induced distributions under the two hypotheses~\citep{2023towards}, whose square is upper bounded by the well-known KL divergence up to a constant by Pinsker's inequality~\citep{2022short}.

\paragraph{GaussMark.} Due to the theoretical analysis below, in this work we focus on improving \gaussmark~\citep{2025gaussmark}, a recently proposed weight-editing watermarking scheme. Given a language model $p_\theta: \mathcal{X} \to \Delta(\mathcal{Y})$, \textsc{GaussMark} partitions the weights as $\theta = (\theta_{\mathrm{wm}}, \theta_0)$, where $\theta_{\mathrm{wm}}$ (with dimension $d_r$) is the subset of model weights modified to embed the watermark, and $\theta_0$ the remaining weights. The base model $p_\theta$ is stored as a fixed reference model, denoted as $\pref$ to avoid confusion, for later detection. To embed a watermark, \textsc{GaussMark} samples the key $\xi_\sigma \sim \mathcal{N}(0, \sigma^2 \mathbb{I}_{d_r})$ and obtains watermarked model $p_{\theta(\xi_\sigma)}$ with $\theta(\xi_\sigma) = (\theta_{\mathrm{wm}} + \xi_\sigma, \theta_0)$, i.e., it perturbs the selected weights with a small amount of Gaussian noise and leaves the others unchanged.  To detect the watermark, \gaussmark\ uses the following test statistic:
% \begin{equation}
% \label{eq:test_sta}
%     \psi(y,\xi_\sigma \mid x) = \frac{\langle \xi_\sigma, \nabla_{\theta_{\mathrm{wm}}}\log p_\theta(y \mid x)\rangle}{\sigma \|\nabla_{\theta_{\mathrm{wm}}}\log p_\theta(y \mid x)\|_2}.
% \end{equation}
\begin{equation}
\label{eq:test_sta}
    \psi(y,\xi_\sigma \mid x) = \frac{\langle \xi_\sigma, \nabla_{\theta_{\mathrm{wm}}}\log \pref(y \mid x)\rangle}{\sigma \|\nabla_{\theta_{\mathrm{wm}}}\log \pref(y \mid x)\|_2}.
\end{equation}
Intuitively, this statistic measures the alignment between the secret key $\xi_\sigma$ and the gradient of the reference model with respect to the watermarked weights. Under $\hnull$, $\xi_\sigma$ is independent of the text $y$, so $\psi(y,\xi_\sigma \mid x)$ follows a standard normal distribution and a test of level $\alpha$ can be constructed by thresholding the statistic at the inverse Gaussian CDF (denoted by $\Phi^{-1}$) at $1-\alpha$.

% \vspace{-0.5em}
\section{\marktune: On-Policy Fine-Tuning with Watermark Signal Rewards} \label{sec:method}
% \vspace{-0.5em}

In this section, we introduce our proposed watermarking scheme, \marktune, a theoretically principled approach that enhances the quality-detectability trade-off of \gaussmark\ through on-policy fine-tuning with watermark signal rewards.  To motivate \marktune, we first present the following result on the distortion introduced by \gaussmark.

\begin{proposition}\label{prop:tv_bound}
  Let $p_\theta$ be a language model with parameters $\theta = (\theta_{\mathrm{wm}}, \theta_0) \in \Theta$, where $\theta_{\mathrm{wm}} \in \mathbb{R}^{d_r}$ is the subset of parameters to which the \gaussmark\ is applied. Assume that the map $\theta' \mapsto \mathrm{D_{KL}}\!\left(p_{\theta'}(\cdot \mid x)\,\big\|\, p_{\theta}(\cdot \mid x)\right)$ is twice continuously differentiable with uniformly bounded Hessian.\footnote{This condition can be relaxed to local continuity and bounding of the Hessian at the cost of conditioning on the high probability event that $\theta(\xi)$ remains in this neighborhood.}  If $\xi_\sigma \sim \mathcal{N}(0, \sigma^2 \mathbb{I}_{d_r})$ and $\theta(\xi_\sigma) = (\theta_{\mathrm{wm}} + \xi_\sigma, \theta_0)$, then for any prompt $x$ there is a constant $C$ depending on the Hessian bound such that
    \begin{equation}
      \mathbb{E}_{\xi_\sigma}\!\left[
            \mathrm{D_{KL}}\!\left(
                p_{\theta(\xi_\sigma)}(\cdot \mid x)
                \,\Big\|\,
                p_{\theta}(\cdot \mid x)
            \right)
        \right]\le
        C \, \sigma^2 d_r.
    \end{equation}
\end{proposition}
% \vspace{-0.5em}
% Keep the plot and its caption together in the right half of the text block.
\begin{wrapfigure}{R}{0.45\textwidth}

    \centering
    \begin{tikzpicture}
    \begin{axis}[
        width=\linewidth,
        height=4.0cm,
        domain=-1.4:1.4,
        samples=200,
        axis x line=bottom,
        axis y line=left,
        axis line style={-stealth,line width=0.8pt},
        xlabel={watermark-sensitive direction},
        ylabel={generation quality loss},
        xmin=-1.45, xmax=1.45,
        ymin=-0.05, ymax=2.0,
        xtick=\empty,
        ytick=\empty,
        clip=false,
        every axis plot/.style={smooth},
        label style={font=\small},
        ticklabel style={font=\small},
    ]
    \addplot[very thick, blue!65!black]
        {0.06*x^10 + 0.02*x^2};
    \addplot[only marks, mark=*, mark size=2.3pt,
        draw=green!45!black, fill=green!65!black]
        coordinates {(-0.2,0.0008)};
    \addplot[only marks, mark=*, mark size=2.3pt,
        draw=purple!45!black, fill=purple!65!black]
        coordinates {(0.6,0.00756)};
    \addplot[only marks, mark=*, mark size=2.3pt,
        draw=red!60!black, fill=red!70!black]
        coordinates {(-0.8,0.01924245)};
    \addplot[only marks, mark=*, mark size=2.3pt,
        draw=teal!70!black, fill=teal!80!black]
        coordinates {(1.3,0.86095)};
    \node[above=2pt]      at (axis cs:0,0)       {$\theta^\star$};
    \node[above right=2pt]at (axis cs:0.6,0) {$\theta$};
    \node[above left=2pt] at (axis cs:-0.1,0.01924245) {$\theta_{\marktune}$};
    \node[above left=2pt] at (axis cs:1.3,0.86095) {$\theta_{\gaussmark}$};
    \end{axis}
    \end{tikzpicture}
    \caption{Stylized one-dimensional landscape along a watermark-sensitive direction. Both $\theta_{\gaussmark}$ and $\theta_{\marktune}$ lie at nontrivial distances from the base model $\theta$ along this direction, leading to significant watermark detectability, but $\theta_{\marktune}$ resides within the flat high-quality basin around $\theta^\star$ and therefore incurs substantially less quality degradation than $\theta_{\gaussmark}$.}
    \label{fig:stylized-illustration}
    % \vspace{-0.7em}
\end{wrapfigure}
Proposition~\ref{prop:tv_bound} characterizes the quality side of \gaussmark: the expected KL divergence between the perturbed and base models is bounded by $C\sigma^2 d_r$. Due to space constraints, we defer the detectability side to Appendix~\ref{app:detect_power}, where we show that the type-II error of the \gaussmark\ test decays exponentially in $d_r$ at a rate governed by $\sigma$. Together, these two results expose a fundamental quality-detectability trade-off for \gaussmark: increasing $\sigma$ or the watermark subspace dimension $d_r$ amplifies the detection signal but inevitably inflates the upper bound on the resulting quality distortion.

Nevertheless, we argue that such pessimism is overstated for two reasons. First, the base model $p_\theta$ should not be regarded as an oracle that perfectly characterizes the high-quality text but instead is a learned model with imperfections. Consequently, closeness \emph{to the base model} is not a necessary condition for achieving high-quality generation; rather, one should prefer closeness to some optimal distribution $p_{\theta^\star}$ that better expresses the true distribution of observed text.  Such intuition is expressed in Figure \ref{fig:stylized-illustration}, where the base model $p_\theta$ is suboptimal compared to the high-quality reference model $p_{\theta^\star}$; while \gaussmark\ attempts to maintain generation quality by remaining close to $\theta$, our proposed method \marktune\ (outlined below) instead seeks parameters that remain close to $\theta^\star$.

Second, KL divergence is a particularly stringent measure of distributional distortion: it upper-bounds worst-case deviations over all events, far exceeding what is necessary to capture human-perceived quality. Modern LMs are heavily over-parameterized, with loss landscapes that contain broad, flat basins~\citep{2025understanding}. Consequently, given a base model $p_\theta$, there typically exist alternative weight configurations $p_{\theta^\star}$ that exhibit a significant TV distance from $p_\theta$ but that yield essentially the same generation quality on general or domain-specific tasks.  This perspective is also depicted in Figure \ref{fig:stylized-illustration}, where both \gaussmark\ and \marktune\ achieve similar watermark strength (denoted by the horizontal distance from $\theta$) but our proposed \marktune\ maintains better generation quality by remaining in the basin of the high-quality reference model $\theta^\star$.

\paragraph{Design motivation.}
Inspired by the success of Reinforcement Learning from Human Feedback (RLHF)~\citep{ouyang2022training} and Reinforcement Learning with Verifiable Rewards (RLVR)~\citep{guo2025deepseek} in enhancing LLM capabilities, we propose \marktune, an on-policy fine-tuning framework that uses the \gaussmark\ test statistic as its reward. Unlike the rewards in RLHF and RLVR, which are typically produced by trained reward models, domain-specific verifiers, or gold-standard solutions, this test statistic is an intrinsic quantity that can be easily computed using a watermark key.
% Consequently, our method can be viewed as a form of self-improvement driven by intrinsic feedback, akin in spirit to~\citep{2025learning}, which employs self-certainty-based intrinsic rewards to enhance LLM reasoning.
Following prior work, we include a regularization term in the optimization
objective to preserve generation quality.   A central design choice is
\emph{what to regularize toward}. Conventional  RL objectives~\citep{guo2025deepseek,ouyang2022training,2017proximal}
anchor the trained model to the base model via KL divergence, but for watermarking
this anchor is counterproductive: increasing the watermark signal inherently
moves the trained weights away from  the base model, so a base-model anchor
pulls directly against the reward and this distribution-preservation
constraint can limit improvements over the existing \gaussmark\ trade-off
frontier. We instead anchor to an external high-quality reference
distribution, instantiated either as  the cross-entropy on  human text or
as the KL divergence to a stronger oracle from the same model family. Both
instantiations decouple quality preservation from base-model proximity: they
enlarge the  search space for watermark-aware weight updates and   bias  the
training dynamics   toward  watermark-sensitive directions along which
quality decays only slowly, yielding a better quality-detectability frontier.
We empirically validate this principle in Appendix~\ref{app:quality_roles} (Table~\ref{tab:ablation_anchor}), where replacing the anchor with the standard base-model KL fails to improve over \gaussmark\ at matched quality.

\paragraph{Optimization with dual objectives.}
Given a base model $p_\theta$ with $\theta = (\theta_{\mathrm{wm}}, \theta_0)$ and a watermark key $\xi_\sigma \sim \mathcal{N}(0, \sigma^2 \mathbb{I}_{d_r})$ for small $\sigma$, \marktune\ first applies a \gaussmark\ soft activation, $\theta(\xi_\sigma) = (\theta_{\mathrm{wm}} + \xi_\sigma, \theta_0)$, which induces a weak watermark signal while preserving generation quality. This procedure yields two distinct models: a fixed copy of the base model, denoted as $\pref$, which is used to compute the reward (i.e., the test statistic), and an updated data generating model $p_{\theta(\xi_\sigma)}$ (with parameter $\theta(\xi_\sigma)$) that is optimized to adapt the watermark embedding. The training objective for $p_{\theta(\xi_\sigma)} $ is to maximize

% \vspace{-1.0em}
\begin{equation}
\label{eq:overall_opt}
\begin{aligned}
\mathcal{J}(\theta(\xi_\sigma))&=
\mathbb{E}_{x \sim \mathcal{D},\; y \sim p_{\theta(\xi_\sigma)}(\cdot \mid x)}
\!\left[
\frac{\langle \xi_\sigma, \nabla_{\theta_{\mathrm{wm}}}\log \pref(y)\rangle}{\sigma \|\nabla_{\theta_{\mathrm{wm}}}\log \pref(y)\|_2}
\right] -\lambda \mathcal{L}_{\mathrm{reg}}\!\left(\theta(\xi_\sigma); x, y\right).
\end{aligned}
\end{equation}
% \vspace{-0.5em}

% \begin{equation}
% \label{eq:overall_opt}
% \begin{aligned}
% \max_{\theta(\xi_\sigma)}\;&
% \mathbb{E}_{x \sim \mathcal{D},\; y \sim p_{\theta(\xi_\sigma)}(\cdot \mid x)}
% \!\left[
% \mathtt{sg}\!\left(
% \frac{\langle \xi_\sigma, \nabla_{\theta_{\mathrm{wm}}}\log p_\theta(y)\rangle}
% {\sigma \|\nabla_{\theta_{\mathrm{wm}}}\log p_\theta(y)\|_2}
% \right)
% \right] \\
% &\;-\lambda\, \mathcal{L}_{\mathrm{reg}}\!\left(\theta(\xi_\sigma); x, y\right).
% \end{aligned}
% \end{equation}

where the first term is the watermark signal reward $\mathcal{R}_{\mathrm{wm}}(y;\xi_\sigma)$, formulated using the \gaussmark\ test statistic.\footnote{We ignore the prompt $x$ to match the real-world detection.}
% The operator $\mathtt{sg}(\cdot)$ denotes stop-gradient: the gradient $\nabla_{\theta_{\mathrm{wm}}}\log p_\theta(y)$ is computed on a fixed $p_\theta$ and treated as constant with respect to $\theta(\xi_\sigma)$, so no backpropagation passes through it.
The term $\mathcal{L}_{\mathrm{reg}}$ is a regularization term, which can be either the cross-entropy loss $\mathcal{L}_{\mathrm{ce}} = -\frac{1}{T}\sum_{t=1}^T \log p_{\theta(\xi_\sigma)}(x_t \mid x_{<t})$ or the KL divergence to an oracle model $\mathcal{L}_{\mathrm{kl}} = \mathrm{D_{KL}}\!\left(p_{\theta(\xi_\sigma)}(\cdot \mid x)\,\|\,p_{\text{oracle}}(\cdot \mid x)\right)$.\footnote{We present an ablation study on the choice of regularization term in Appendix~\ref{sec:reg_ablation}.} $\lambda > 0$ denotes a hyperparameter that balances the watermark reward against the regularization on text quality degradation. By rewarding generated samples $y \sim p_{\theta(\xi_\sigma)}(\cdot \mid x)$ that produce stronger watermark signals, while penalizing deviations from high-quality text through $\mathcal{L}_{\mathrm{reg}}$, this dual-objective, on-policy fine-tuning framework improves the quality-detectability trade-off of \gaussmark. The high-level procedure of \marktune\ is summarized in Algorithm~\ref{alg:marktune}. For practical implementation, we employ GRPO~\citep{shao2024deepseekmath} to maximize the overall objective~\eqref{eq:overall_opt}, with algorithm details provided in Appendix~\ref{app:alg_details}. We further examine the roles of Gaussian initialization and quality regularization using direct-GRPO and quality-only fine-tuning controls, respectively (Appendices~\ref{app:initialization} and~\ref{app:quality_roles}).

\begin{algorithm}[!t]
  \caption{\marktune\ with \gaussmark\ activation}
  \label{alg:marktune}
  \begin{algorithmic}[1]
    \STATE \textbf{Input:} LM $p_\theta$ with $\theta = (\theta_{\mathrm{wm}}, \theta_0)$, strength $\sigma>0$, dataset $\mathcal{D}$, regularization coefficient $\lambda > 0$.
    \STATE Sample a key $\displaystyle \xi_\sigma \sim \mathcal{N}(0, \sigma^2\mathbb{I}_{d_r})$ and apply \gaussmark\ to obtain $\theta(\xi_\sigma) = (\theta_{\mathrm{wm}} + \xi_\sigma,\theta_0)$.
    \STATE Store $p_\theta$ as the fixed reference model $\pref$.
    \STATE Conduct GRPO Algorithm~\ref{alg:GRPO-PO} for $\theta(\xi_\sigma)$ to obtain $\theta^\star(\xi_\sigma)$ via maximizing the objective~\eqref{eq:overall_opt}.
    \STATE \textbf{Output:} Watermarked weights $\theta^\star(\xi_\sigma)$.
  \end{algorithmic}
\end{algorithm}

% \vspace{-0.7em}
\begin{remark} \label{rmk:reward}
     \upshape  We emphasize that, in principle, \marktune\ can serve as a general on-policy fine-tuning framework for improving the quality-detectability trade-off of   a broad class of watermarking schemes,  by replacing the first term in~\eqref{eq:overall_opt} with the corresponding test statistic or   any other computable detection statistic  as the reward.  We choose \gaussmark\ in this work due to the theoretical motivation outlined below. An empirical demonstration of \marktune\ instantiated with the KGW $z$-score as reward is provided in Appendix~\ref{app:alter_reward} (Table~\ref{tab:ablation_kgw_reward}). Extensions to semantic-level statistics that involve detection-time resampling are nontrivial and discussed in Appendix~\ref{app:limitations}.
 \end{remark}
% \vspace{-0.3em}

\paragraph{Type I error control.} For watermark detection, we inherit the procedure in \gaussmark. Since the generated text $y$ remains independent of $\xi_\sigma$ under $\hnull$, Proposition~\ref{prop:FPR_control}, whose proof is in Appendix \ref{app:proof_2}, provides rigorous statistical guarantees on the false positive rate.
\begin{proposition} \label{prop:FPR_control}
    Let $\alpha \in (0,1)$, and $\tau_\alpha := \Phi^{-1}(1-\alpha)$ where $\Phi$ is the CDF of the standard normal distribution. Then, for any $y \in \mathcal{Y}$ the test $\mathbb{I}\!\left\{\frac{\langle \xi_\sigma, \nabla_{\theta_{\mathrm{wm}}}\log \pref(y)\rangle}{\sigma \|\nabla_{\theta_{\mathrm{wm}}}\log \pref(y)\|_2} \ge \tau_\alpha \right\}$ has level $\alpha$.
\end{proposition}
% \vspace{-0.5em}

% \vspace{-0.7em}

\paragraph{Informal analysis for \marktune's improvement.}
We can gain intuition for why \marktune\ improves upon \gaussmark\ by considering the linear-softmax regime, a common analytically tractable proxy for studying language models~\citep{2025gaussmark,foster2025good,huang2025self}. Suppose the ``language model'' takes the form $p_\theta(y \mid x) \propto \exp(\langle \theta, \Phi(x,y)\rangle)$, where $\Phi(x,y)$ is a fixed feature map. This setting can be viewed as an approximation in which all but the final layer of a transformer are frozen. Let $\theta^\star$ denote the parameter of the optimal language model and $\theta$ that of the base model. First, consider $\mathcal{L}_{\mathrm{kl}} = \mathrm{D}_{\mathrm{KL}}\!\left(p_{\theta(\xi_\sigma)} \,\|\, p_{\theta^\star}\right)$, under which the objective in~\eqref{eq:overall_opt} becomes a reverse KL-regularized policy optimization problem over $\theta(\xi_\sigma)$. Applying the standard reparameterization trick~\citep{rafailov2023direct}, if optimization is carried to completion (and ignoring the rescaling in~\eqref{eq:overall_opt} for simplicity), we obtain $p_{\theta_{\marktune}}(y \mid x) \propto p_{\theta^\star}(y \mid x)\exp(\langle \xi, \Phi(x,y)\rangle)$. Under the linear-softmax assumption, this yields $p_{\theta_{\marktune}}(y \mid x) \propto \exp(\langle \theta^\star+\xi,\Phi(x,y)\rangle)$, whereas \gaussmark\ satisfies $p_{\theta_{\gaussmark}}(y \mid x) \propto \exp(\langle \theta+\xi,\Phi(x,y)\rangle)$. Thus, at a high level, \marktune\ effectively applies \gaussmark\ to the \emph{optimal parameter} rather than the base-model parameter, providing intuition for its improved quality-detectability frontier. Appendix~\ref{app:population_power} provides a conditional detection power bound for this idealized population optimizer.

The cross-entropy regularizer $\mathcal{L}_{\mathrm{ce}}$ corresponds, up to an additive constant, to the forward KL $\mathrm{D}_{\mathrm{KL}}\!\left(p_{\theta^\star} \,\|\, p_{\theta(\xi_\sigma)}\right)$, for which the reparameterization trick does not directly apply. We therefore analyze this case locally in Appendix~\ref{app:marktune_analysis}. In a neighborhood of the high-quality distribution $p_{\theta^\star}$, there exists a deviation $\delta\theta$ along a watermark-sensitive direction such that the watermark reward increases at \emph{first order} in $\kappa=\tfrac{1}{\lambda}$, while the corresponding increase in cross-entropy (or KL) relative to $p_{\theta^\star}$ is only \emph{second order}. This provides complementary theoretical insight into how \marktune\ can adaptively embed a watermark while preserving model quality.

% \vspace{-0.5em}
\section{Experiments} \label{sec:exper}
% \vspace{-0.5em}

In this section, we evaluate \marktune\ on two language models and datasets to demonstrate its applicability across models and generalization across tasks. Our evaluation focuses on four aspects: 1) generation quality and watermark detectability compared to existing inference-time and model-embedded schemes; 2) downstream task performance; 3) robustness to text-editing attacks; and 4) robustness to fine-tuning attacks. \textbf{All experiments use \emph{prompt-agnostic} detection}, reflecting realistic deployment scenarios where prompts are unavailable to the detector. We instantiate the \marktune\ detector using the statistic in~\eqref{eq:test_sta} without prompt access:
$$
\psi(y,\xi_\sigma)
=
\frac{\langle \xi_\sigma,\, \nabla_{\theta_{\mathrm{wm}}}\log \pref(y)\rangle}{\sigma \|\nabla_{\theta_{\mathrm{wm}}}\log \pref(y)\|_2},
$$
where gradients are computed with respect to the base model log-likelihood of the full sequence $y$.

% \vspace{-0.5em}
\subsection{Experimental Setup}
% \vspace{-0.5em}

 \paragraph{Models and datasets.}
We evaluate \marktune\ on two open-weight LMs: Qwen3-4B-Base~\citep{2025qwen3} and Llama2-7B~\citep{2023llama}. Qwen3-4B-Base is a relatively new model, while Llama2-7B is widely used in prior watermarking studies~\citep{2023Robust,2023dipmark,xu2024learning,2024gu,2024zhaopermute}.   We adopt the sampling temperatures recommended by each model's technical report: $T=0.7$  for Qwen3-4B-Base  ~\citep{2025qwen3} and $T=1.0$  for Llama2-7B  ~\citep{2023llama}. We evaluate on two generation tasks: text completion and question answering. For text completion, prompts are drawn from the \textbf{RealNewsLike} split of \textbf{C4}~\citep{raffel2020exploring}; for long-form question answering, we use \textbf{ELI5}~\citep{fan2019eli5}. For each task, we sample 500 prompts and generate 200-token responses.
\marktune\ is fine-tuned on a subset of \textbf{OpenWebText}~\citep{Gokaslan2019OpenWeb}, using prompt lengths of 64-256 tokens for on-policy sampling. The training data are disjoint from and substantially different than the evaluation datasets, demonstrating strong cross-domain generalization.

\paragraph{Baselines and metrics.}
We benchmark \marktune\ against representative watermarking schemes, including inference-time methods KGW~\citep{2023KGW} with two hyperparameter settings (KGW-1/2), OpenAI's Gumbel-max~\citep{2023Aaronson}, Google's SynthID~\citep{2024scalable}, and model-embedded approaches such as KGW-D~\citep{2024gu} (logit distillation of KGW) and \gaussmark~\citep{2025gaussmark}. We exclude RL-based watermarking~\citep{xu2024learning}, as it failed to achieve nontrivial detection under multinomial sampling across datasets. For a fair comparison, we tune each method to attain nontrivial detectability with minimal quality degradation. Detectability is measured by the true positive rate at a fixed false positive rate of 1\% (TPR@FPR=1\%) and by the area under the ROC curve (AUC). Generation quality is assessed using perplexity (PPL), computed with Qwen3-8B and Llama2-13B as oracle language models for the corresponding evaluated models. We measure repetitiveness using Seq-rep-3, defined as the average rate of duplicate 3-grams within a sequence, and report MAUVE scores to quantify distributional similarity between model outputs and human text~\citep{2021mauve}. Implementation details are provided in Appendix~\ref{app:imp_details}. Ultra-low-FPR detection results and an additional model-embedded baseline are reported in Appendices~\ref{app:tail_detection} and~\ref{app:christ_baseline}; sensitivity to watermarked parameters, decoding strategy, and output length is analyzed in Appendices~\ref{app:parameter_sensitivity} and~\ref{app:decoder_length}.

% \vspace{-0.5em}
\subsection{Main Results} \label{sec:main_result}
\begingroup
\raggedbottom
\setlength{\intextsep}{6pt}
% \vspace{-0.5em}

 \begin{table*}[!t]
\centering
\small
\caption{Generation quality and watermark detectability across methods, best-performing model-embedded watermarks are shown in \textbf{bold}.}

\resizebox{\textwidth}{!}{
\begin{tabular}{l|ccccc|ccccc}
\toprule
\multicolumn{1}{l|}{\multirow{2}{*}{\textbf{Method}}} &
\multicolumn{5}{c|}{\textbf{C4-RealNewsLike}} &
\multicolumn{5}{c}{\textbf{ELI5}} \\
% \cmidrule(lr){3-7}\cmidrule(lr){8-12}
 & \textbf{AUC$\uparrow$} & \textbf{TPR$\uparrow$} & \textbf{PPL$\downarrow$} & \textbf{Seq-rep-3$\downarrow$} & \textbf{MAUVE$\uparrow$} & \textbf{AUC$\uparrow$} & \textbf{TPR$\uparrow$} & \textbf{PPL$\downarrow$} & \textbf{Seq-rep-3$\downarrow$} & \textbf{MAUVE$\uparrow$} \\
\midrule
\multicolumn{11}{c}{\cellcolor{blkC}\textbf{Qwen3-4B-Base, T=0.7}} \\
\midrule
No Watermark & 0.500 & 0.010 & \(4.92\err{1.34}\) &  0.04 & 0.98  & 0.500 & 0.010 & $2.67\err{0.94}$ & 0.03  & 0.94 \\
KGW-1    & 0.996 & 0.964 & \(5.51\err{1.56}\) &  0.04 & 0.98 & 0.971 & 0.710 & $3.04\err{1.12}$ &  0.04 & 0.92 \\
KGW-2    & 1.000 & 0.998 & $6.65\err{1.91}$ & 0.04 & 0.96 & 0.996 & 0.988 & $3.69\err{1.51}$ & 0.03 & 0.87 \\
Gumbel-max   & 0.995 & 0.980 & $4.89\err{2.07}$ & 0.18 & 0.74 & 0.991 & 0.938 & $2.71\err{1.01}$  & 0.13 & 0.77 \\
SynthID     & 0.996 & 0.986 & $5.05\err{1.54}$ & 0.04 & 0.96 & 0.993 & 0.946 & $2.76\err{0.95}$  & 0.06 & 0.91 \\
\midrule
KGW-D & 0.981 & 0.762 & $5.54\err{1.53}$ &  0.04 & 0.98 & 0.941 & 0.542 & $3.09\err{1.17}$  & 0.04 & 0.92 \\
GaussMark   & 0.983 & 0.788 & \(5.16\err{1.54}\)  & 0.04 & \textbf{0.98} & 0.979 & 0.734 & $2.83\err{1.07}$  & 0.03 & 0.93 \\
\cellcolor{blkA}\textbf{MarkTune}    &\cellcolor{blkA}\textbf{0.995} &\cellcolor{blkA}\textbf{0.964} &\cellcolor{blkA}$\mathbf{5.04\err{1.37}}$  &\cellcolor{blkA}\textbf{0.03} &\cellcolor{blkA}0.97 & \cellcolor{blkA}\textbf{0.996} & \cellcolor{blkA}\textbf{0.942} & \cellcolor{blkA}$\mathbf{2.78\err{1.03}}$  & \cellcolor{blkA}\textbf{0.03} & \cellcolor{blkA}\textbf{0.93} \\
\midrule
\multicolumn{11}{c}{\cellcolor{blkC}\textbf{Llama2-7B, T=1.0}} \\
\midrule
No Watermark & 0.500 & 0.010 & $6.55\err{1.80}$ & 0.01 & 0.98  & 0.500 & 0.010 & $7.86\err{2.23}$ & 0.02  & 0.97 \\
KGW-1    & 0.991 & 0.886 & $7.09\err{2.04}$ & 0.02 & 0.96 & 0.993 & 0.946 & $8.47\err{2.39}$ &  0.04 & 0.94 \\
KGW-2    & 0.997 & 0.996 & $8.14\err{2.36}$ & 0.01 & 0.95 & 1.000 & 1.000 & $9.61\err{2.68}$ & 0.03 & 0.91 \\
Gumbel-max   & 0.997 & 0.988 & $6.56\err{1.84}$ & 0.14 & 0.86 & 0.998 & 0.994 & $7.98\err{2.29}$  & 0.18 & 0.81 \\
SynthID     & 0.995 & 0.994 & $6.73\err{1.94}$ & 0.03 & 0.94 & 0.997 & 0.988 & $8.06\err{2.41}$  & 0.09 & 0.89 \\
\midrule
KGW-D & 0.945 & 0.582 & $7.04\err{1.98}$ & 0.02 & 0.98 & 0.969 & 0.754 & $8.53\err{2.41}$  & 0.05 & 0.96 \\
GaussMark   & 0.974 & 0.712 & $6.78\err{1.91}$  & 0.01 & \textbf{0.98} & 0.958 & 0.622 & $\mathbf{8.13\err{2.15}}$  & 0.02 & 0.94 \\
\cellcolor{blkA}\textbf{MarkTune}    &\cellcolor{blkA}\textbf{0.997} &\cellcolor{blkA}\textbf{0.982} &\cellcolor{blkA}$\mathbf{6.75\err{1.96}}$  &\cellcolor{blkA}\textbf{0.01} &\cellcolor{blkA}0.96 & \cellcolor{blkA}\textbf{0.993} & \cellcolor{blkA}\textbf{0.930} & \cellcolor{blkA}$8.19\err{2.27}$  & \cellcolor{blkA}\textbf{0.01} & \cellcolor{blkA}\textbf{0.96} \\
\bottomrule
\end{tabular}}
\label{tab:wm_distortion}
% \vspace{-1.0em}
\end{table*}

 % \paragraph{Text generation and detection.}

\Needspace{16\baselineskip}
\begin{wrapfigure}{r}{0.667\textwidth}
    \centering
    \begin{minipage}[t]{0.49\linewidth}
      \centering
      \includegraphics[width=\linewidth]{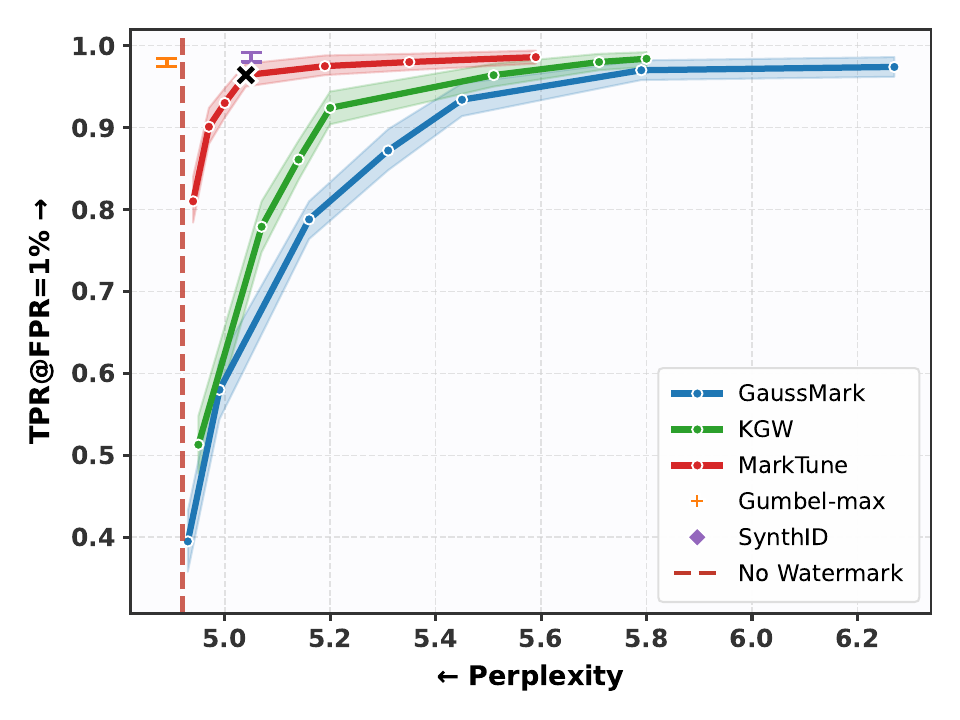}\\
      {\small Qwen3-4B-Base}
    \end{minipage}\hfill
    \begin{minipage}[t]{0.49\linewidth}
      \centering
      \includegraphics[width=\linewidth]{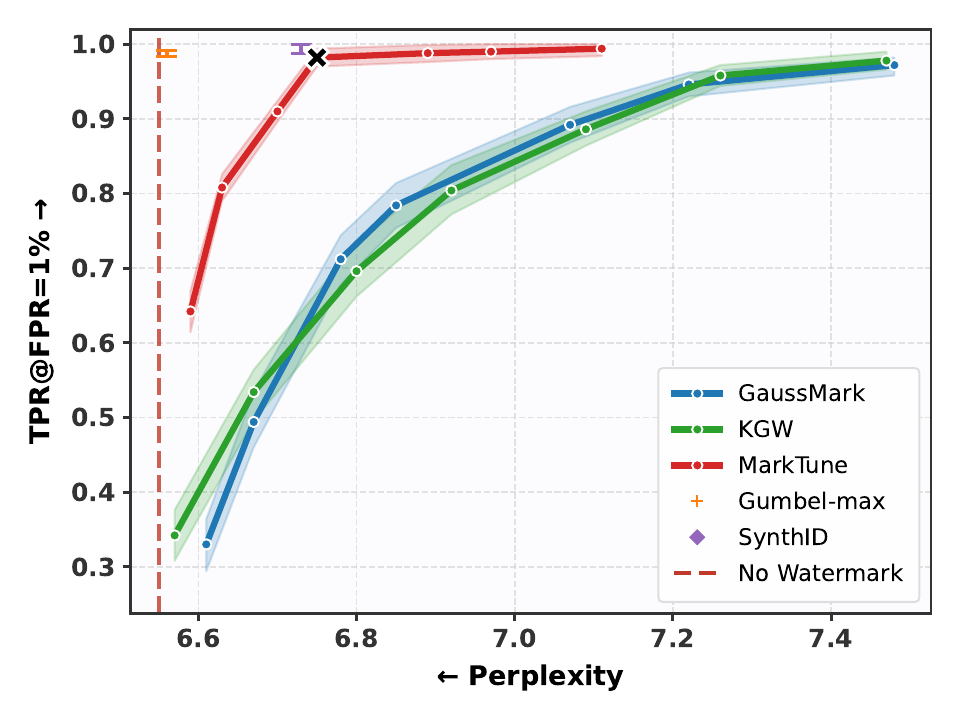}\\
      {\small Llama2-7B}
    \end{minipage}
    \caption{Quality-detectability trade-off at $200$ generated tokens. Black `×' marks the \marktune\ configuration used elsewhere.}
    \label{fig:trade-off}
\end{wrapfigure}
\begingroup
\setlength{\rightskip}{0pt plus 2em}
\setlength{\emergencystretch}{1em}
Table~\ref{tab:wm_distortion} reports the generation quality and detection performance of different watermarking methods across models and datasets, while Figures~\ref{fig:trade-off} and~\ref{fig:latency} respectively summarize the resulting quality-detectability trade-off and the wall-clock generation/detection times on \textbf{C4-RealNewsLike}.\footnote{We omit the trade-off curve of KGW-D since its performance is upper bounded by KGW.} For each method, we sweep a scheme-specific parameter to trace the trade-off curve, except for Gumbel-max and SynthID, which are non-distortionary and therefore yield a single TPR-PPL point. We observe that \textbf{\textsc{MarkTune} consistently outperforms other model-embedded watermarking approaches in detectability while introducing only minimal degradation in text quality} (all three quality metrics remain close to the unwatermarked baseline), pushing the trade-off frontier strictly beyond \gaussmark\ and KGW and matching inference-time watermarking methods. KGW attains high detectability but at the cost of substantial perplexity degradation. In contrast, Gumbel-max generally offers strong detectability and low perplexity, but suffers from notable repetitiveness due to its deterministic decoding procedure. SynthID generally performs well in both detectability and generation quality.
\par
\endgroup

\Needspace{16\baselineskip}
\begin{wrapfigure}{r}{0.667\textwidth}
  \centering
    \centering
    \begin{minipage}[t]{0.49\linewidth}
      \centering
      \includegraphics[width=\linewidth]{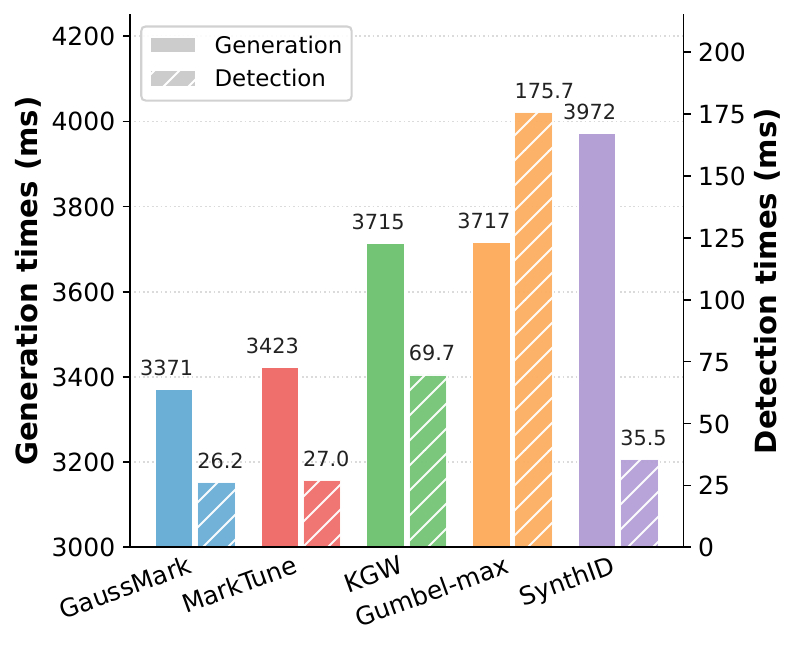}\\
      {\small Qwen3-4B-Base}
    \end{minipage}\hfill
    \begin{minipage}[t]{0.49\linewidth}
      \centering
      \includegraphics[width=\linewidth]{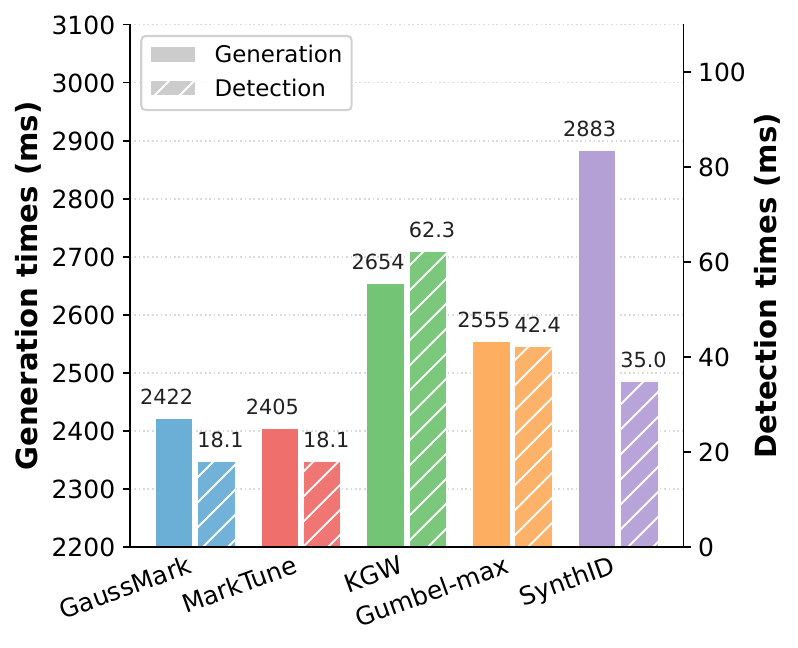}\\
      {\small Llama2-7B}
    \end{minipage}
    \caption{Mean wall-clock times on the first $100$ \textbf{C4-RealNewsLike} prompts ($200$-token prefix, $200$ generated tokens).}
    \label{fig:latency}
  \end{wrapfigure}
\begingroup
\setlength{\rightskip}{0pt plus 2em}
\setlength{\emergencystretch}{1em}
On the computation time side (Figure~\ref{fig:latency}), \marktune\ matches \gaussmark\ at both generation and detection time and is substantially faster than every inference-time baseline on both axes, so its quality-detectability gains incur no inference-time overhead. Computation time benchmark details are deferred to Appendix~\ref{app:lat_setup}. Detector access requirements and measured peak GPU memory are reported in Appendix~\ref{app:detection_resources}. \rev{This comparison concerns generation and detection after watermark embedding. \marktune\ additionally incurs a one-time offline GRPO training cost, reported separately in Appendix~\ref{app:cost}.}
\par
\endgroup

% \Needspace{13\baselineskip}

\paragraph{Impact on downstream task performance.} Because watermarking strategies alter the output distributions of LMs, it is essential to ensure that the utility of the underlying model is not  compromised. Although the results in Table~\ref{tab:wm_distortion} indicate minimal degradation in generation quality—measured by perplexity, Seq-rep-3, and MAUVE—recent work~\citep{2025gaussmark,2024downstream,2024scalable} has shown that such metrics do not reliably predict the performance degradation introduced by watermarking on downstream tasks. Therefore, we further evaluate the watermarked models on downstream tasks using the Language Model Evaluation Harness~\citep{2021framework} across three benchmarks: MMLU~\citep{2020mmlu} for general knowledge and reasoning, GSM8K~\citep{cobbe2021gsm8k} for mathematical problem solving, and MBPP~\citep{2021mbpp} for coding tasks.

Figure~\ref{fig:downstream} reports the accuracy of three model-embedded watermarking schemes on these benchmarks, relative to the unwatermarked model.

\Needspace{14\baselineskip}
\begin{wrapfigure}{r}{0.667\textwidth}
  \centering
  % \vspace{-1.0em}
  \begin{subfigure}{0.49\linewidth}
    \centering
    \includegraphics[width=\linewidth]{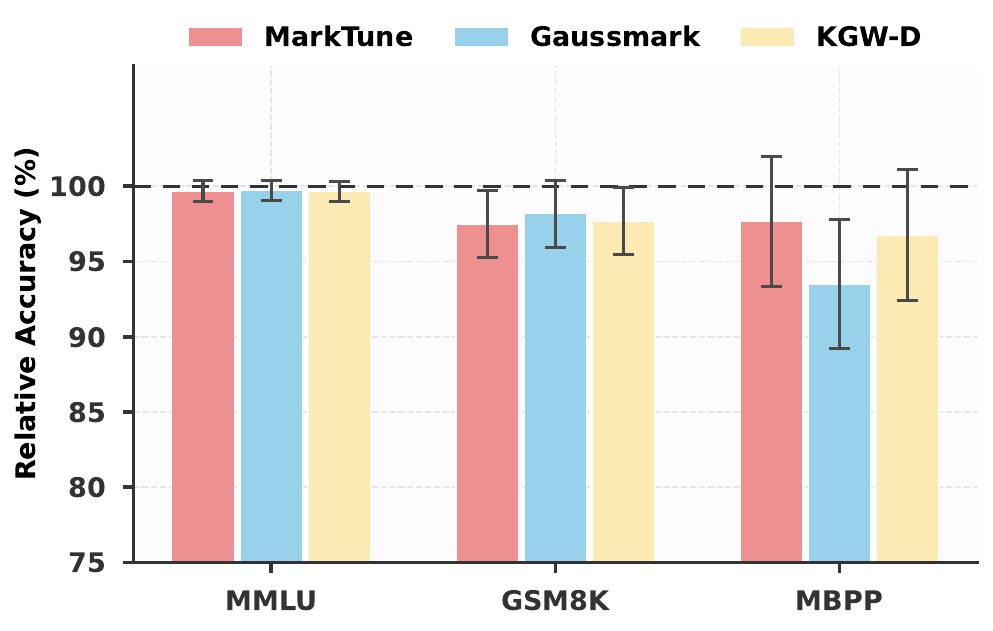}
    \caption{Qwen3-4B-Base}
    \label{fig:down_qwen3}
  \end{subfigure}
  \hfill
  \begin{subfigure}{0.49\linewidth}
    \centering
    \includegraphics[width=\linewidth]{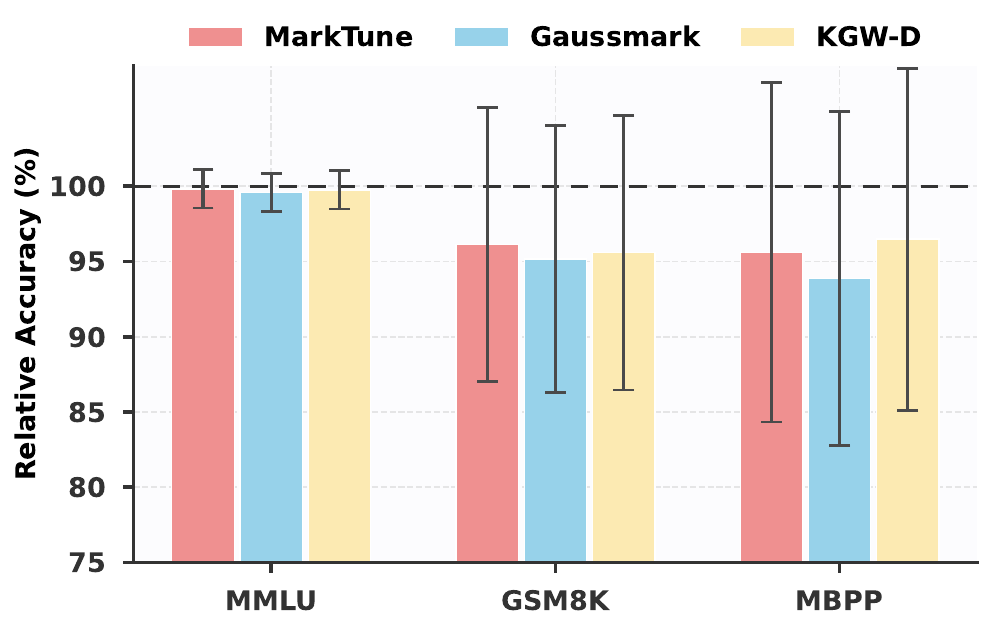}
    \caption{Llama2-7B}
    \label{fig:down_llama2}
  \end{subfigure}
  \caption{Relative downstream task accuracy compared to unwatermarked models across general, math, and coding benchmarks.}
  \label{fig:downstream}
  % \vspace{-1.0em}
\end{wrapfigure}
On the MMLU benchmark, all methods exhibit negligible accuracy degradation, indicating that they preserve the model's general capabilities. On the GSM8K and MBPP benchmarks, \marktune\ ranks first or second, substantially improving over \gaussmark, especially on MBPP. We attribute this to the fact that \marktune\ more effectively guides the training dynamics toward ``directions'' that simultaneously enhance the watermark signal and preserve utility, starting from only a slight weight perturbation. In contrast, \gaussmark\ requires a much larger perturbation to produce a nontrivial watermark signal, which can negatively affect downstream performance. We provide watermarked examples in Appendix~\ref{app:examples} for generation quality illustration beyond these quantitative metrics.

% \vspace{-0.5em}
\subsection{Robustness Evaluation} \label{sec:robust}
% \vspace{-0.5em}

\paragraph{Robustness to text-editing attacks.}
Text editing is common in practical LM use and can weaken watermark signals. We evaluate robustness against four editing attacks: word deletion, synonym substitution, paraphrasing, and roundtrip translation. As shown in Table~\ref{tab:robust}, for deletion and substitution, we randomly modify a fraction of watermarked tokens, sampling synonyms from WordNet~\citep{1995wordnet}. Paraphrasing is performed using two Dipper configurations with varying lexical and order diversity~\citep{krishna2023paraphrasing}, while a stronger paraphrasing attack uses French--English roundtrip translation with MarianMTModel~\citep{2024democratizing}. Under word deletion and synonym substitution, KGW-D often performs strongly among model-embedded methods, likely because it is distilled from the one-token-prefix KGW watermark. While \marktune\ consistently improves over \gaussmark, its watermark signal still decays at higher modification rates. However, such token-level attacks are relatively crude and can substantially degrade text quality. To quantify this effect, Table~\ref{tab:robust_extend} (deferred to the appendix for space) reports the perplexity of text after each editing attack. Token-level attacks substantially degrade text quality, whereas paraphrasing attacks better preserve quality while reducing the detectability of alternative watermarks. Notably, \marktune\ remains robust under these paraphrasing attacks. Thus, focusing on the more realistic paraphrasing setting, \marktune\ outperforms all watermarking baselines, including strong inference-time approaches, demonstrating its robustness and practical applicability under adversarial text editing.

\begin{table*}[!t]
\centering
\small
\caption{Watermark detectability under text-editing and paraphrasing attacks for different methods on \textbf{C4-RealNewsLike}. ``Del-$n$'' and ``Sub-$n$'' denote attacks that delete or substitute a proportion of $n$ tokens, respectively. ``Dipper-1'' refers to Dipper paraphrasing with lexical and order diversity set to 20, while ``Dipper-2'' uses a diversity level of 60. ``Translate'' denotes the roundtrip translation attack described in Section~\ref{sec:robust}. Best performances among model-embedded watermarks are shown in \textbf{bold}.}

\resizebox{\textwidth}{!}{
\begin{tabular}{l|cc|cc|cc|cc|cc|cc|cc}
\toprule
\multicolumn{1}{l|}{\multirow{2}{*}{\textbf{Method}}} &
\multicolumn{2}{c|}{\textbf{Del-0.2}} & \multicolumn{2}{c|}{\textbf{Del-0.5}} &
\multicolumn{2}{c|}{\textbf{Sub-0.2}} & \multicolumn{2}{c|}{\textbf{Sub-0.5}} &
\multicolumn{2}{c|}{\textbf{Dipper-1}} & \multicolumn{2}{c|}{\textbf{Dipper-2}} & \multicolumn{2}{c}{\textbf{Translate}} \\
% \cmidrule(lr){3-4} \cmidrule(lr){5-6} \cmidrule(lr){7-8} \cmidrule(lr){9-10}
% \cmidrule(lr){11-12} \cmidrule(lr){13-14} \cmidrule(lr){15-16}
 & \textbf{AUC$\uparrow$} & \textbf{TPR$\uparrow$} & \textbf{AUC$\uparrow$} & \textbf{TPR$\uparrow$} & \textbf{AUC$\uparrow$} & \textbf{TPR$\uparrow$} & \textbf{AUC$\uparrow$} & \textbf{TPR$\uparrow$} & \textbf{AUC$\uparrow$} & \textbf{TPR$\uparrow$} & \textbf{AUC$\uparrow$} & \textbf{TPR$\uparrow$} & \textbf{AUC$\uparrow$} & \textbf{TPR$\uparrow$} \\
\midrule
\multicolumn{15}{c}{\cellcolor{blkC}\textbf{Qwen3-4B-Base, T=0.7}} \\
\midrule
KGW-1    & 0.881 & 0.330 & 0.617 &  0.040 & 0.873 & 0.324 & 0.638 & 0.048 &  0.771 & 0.142 & 0.573 & 0.026  & 0.831 & 0.242 \\
KGW-2    & 0.975 & 0.746 & 0.697 & 0.082 & 0.986 & 0.826 & 0.747 & 0.106 & 0.884 & 0.372 & 0.616 & 0.038  & 0.928 & 0.552 \\
Gumbel-max   & 0.993 & 0.948 & 0.778 & 0.162 & 0.993 & 0.922 & 0.840 & 0.348  & 0.945 & 0.748 & 0.721 & 0.146  & 0.955 & 0.846 \\
SynthID     & 0.987 & 0.890 & 0.780 & 0.094 & 0.986 & 0.846 & 0.834 & 0.182 & 0.941 & 0.576 & 0.700 & 0.074  & 0.968 & 0.818 \\
\midrule
KGW-D & \textbf{0.947} & \textbf{0.582} & \textbf{0.851} & \textbf{0.196} & 0.939 & \textbf{0.578} & \textbf{0.837} & \textbf{0.190}  & 0.855 & 0.274 & 0.589 & 0.030  & 0.912 & 0.404 \\
GaussMark   & 0.908 & 0.366 & 0.622  & 0.048 & 0.919 & 0.384 & 0.775 & 0.100  & 0.915 & 0.438 & 0.795 & 0.138  & 0.948 & 0.574 \\
\cellcolor{blkA}\textbf{MarkTune}    &\cellcolor{blkA}0.921 &\cellcolor{blkA}0.484 &\cellcolor{blkA}0.744  &\cellcolor{blkA}0.068 &\cellcolor{blkA}\textbf{0.948} & \cellcolor{blkA}0.552 & \cellcolor{blkA}0.783 & \cellcolor{blkA}0.110  & \cellcolor{blkA}\textbf{0.977} & \cellcolor{blkA}\textbf{0.852} & \cellcolor{blkA}\textbf{0.851} & \cellcolor{blkA}\textbf{0.254}  & \cellcolor{blkA}\textbf{0.973} & \cellcolor{blkA}\textbf{0.858} \\
\midrule
\multicolumn{15}{c}{\cellcolor{blkC}\textbf{Llama2-7B, T=1.0}} \\
\midrule
KGW-1    & 0.853 & 0.232 & 0.581 & 0.024 & 0.878 & 0.300 & 0.638 & 0.030 &  0.757 & 0.116  & 0.551 & 0.014  & 0.777 & 0.146\\
KGW-2    & 0.975 & 0.712 & 0.711 & 0.060 & 0.980 & 0.756 & 0.772 & 0.126 & 0.862 & 0.322 & 0.614 & 0.032  & 0.922 & 0.448 \\
Gumbel-max   & 0.993 & 0.924 & 0.790 & 0.224 & 0.993 & 0.950 & 0.862 & 0.296 & 0.950 & 0.700 & 0.711 & 0.148  & 0.961 & 0.796 \\
SynthID     & 0.979 & 0.774 & 0.777 & 0.114 & 0.987 & 0.862 & 0.843 & 0.140  & 0.943 & 0.552 & 0.710 & 0.084  & 0.949 & 0.612 \\
\midrule
KGW-D & 0.914 & 0.410 & \textbf{0.803} & \textbf{0.148} & 0.887 & 0.406 & 0.759 & \textbf{0.144}  & 0.814 & 0.238 & 0.590 & 0.026  & 0.889 & 0.324 \\
GaussMark   & 0.835 & 0.242 & 0.617  & 0.036 & 0.853 & 0.228 & 0.759 & 0.098  & 0.857 & 0.240 & 0.687 & 0.060  & 0.925 & 0.436 \\
\cellcolor{blkA}\textbf{MarkTune}    &\cellcolor{blkA}\textbf{0.928} &\cellcolor{blkA}\textbf{0.506} &\cellcolor{blkA}0.783  &\cellcolor{blkA}0.080 &\cellcolor{blkA}\textbf{0.944} & \cellcolor{blkA}\textbf{0.578} & \cellcolor{blkA}\textbf{0.809} & \cellcolor{blkA}0.122  & \cellcolor{blkA}\textbf{0.961} & \cellcolor{blkA}\textbf{0.836} & \cellcolor{blkA}\textbf{0.859} & \cellcolor{blkA}\textbf{0.270}  & \cellcolor{blkA}\textbf{0.966} & \cellcolor{blkA}\textbf{0.834} \\
\bottomrule
\end{tabular}}
\label{tab:robust}
% \vspace{-1.0em}
\end{table*}

\paragraph{Robustness to open-weight attacks.} Given the open-weight deployment setting, we also evaluate the robustness of weight-embedded watermarks to fine-tuning attacks.
With the growing adoption of LoRA as a PEFT technique, we consider an adversary who attempts to remove the watermark by further fine-tuning the model with LoRA on a subset of \textbf{OpenWebText} that is disjoint from the dataset used during \marktune\ training, following the setup in prior work~\citep{2025towards}.
% \Needspace{13\baselineskip}
\Needspace{14\baselineskip}
\begin{wrapfigure}[14]{r}{0.667\textwidth}
  \centering
  % \vspace{-1.0em}
  \begin{subfigure}{0.49\linewidth}
    \centering
    \includegraphics[width=\linewidth]{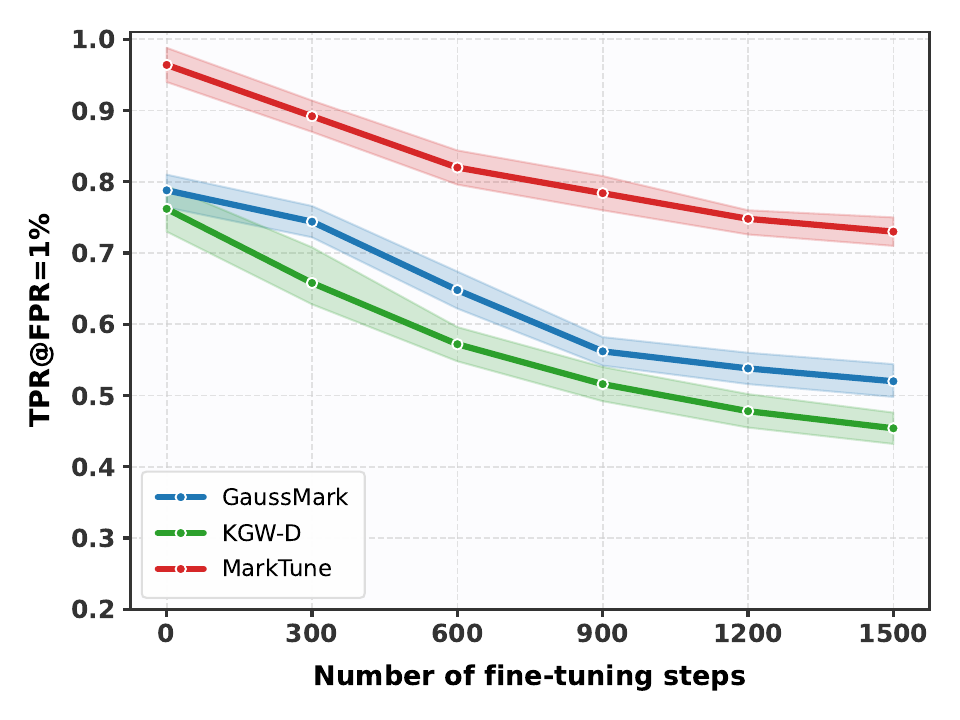}
    \caption{Qwen3-4B-Base}
    \label{fig:ft_qwen3}
  \end{subfigure}
  \hfill
  \begin{subfigure}{0.49\linewidth}
    \centering
    \includegraphics[width=\linewidth]{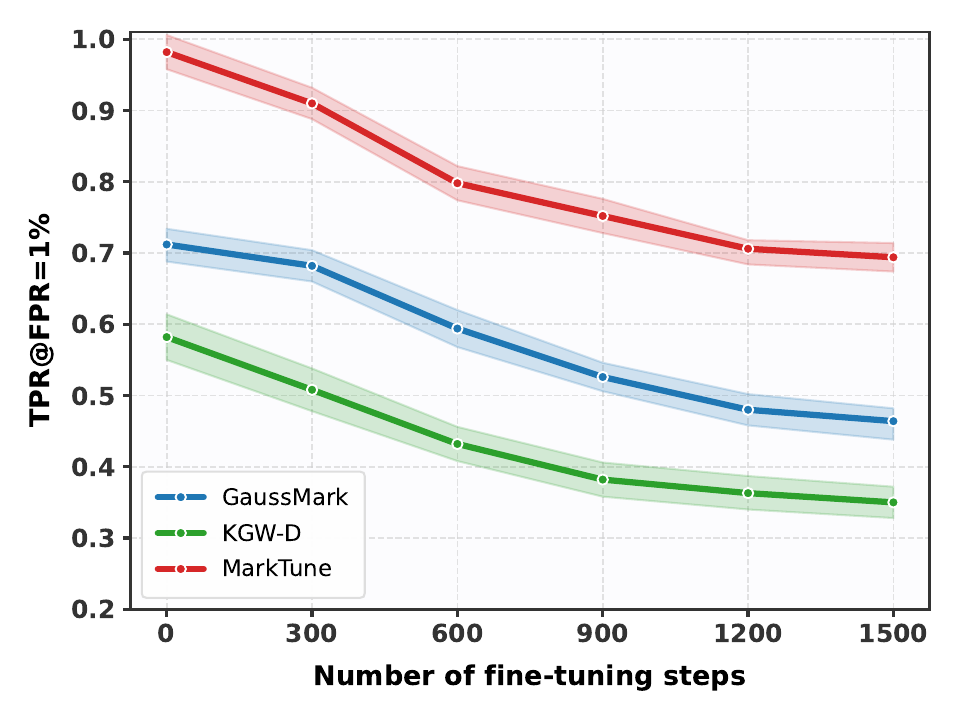}
    \caption{Llama2-7B}
    \label{fig:ft_llama2}
  \end{subfigure}
  \caption{Detectability (TPR@1\%FPR) decay under LoRA fine-tuning attack.}
  \label{fig:ft-attack}
  % \vspace{-1.0em}
\end{wrapfigure}
\begingroup
\setlength{\rightskip}{0pt plus 2em}
\setlength{\emergencystretch}{1em}
As shown in Figure~\ref{fig:ft-attack}, although \marktune\ enhances the watermark via fine-grained on-policy weight updates, it does not exhibit a faster TPR decay rate than the other two methods. Owing to its strong initial watermark signal, it retains substantial detection power even after 1500 fine-tuning steps, in contrast to \gaussmark\ and KGW-D.

Beyond LoRA fine-tuning, we evaluate whether common modifications to released weights weaken the watermark while preserving model utility. Under the evaluated INT8 quantization, WANDA pruning, and SLERP merging settings, \marktune\ retains nontrivial detectability with generation quality close to the unmodified watermarked model (Appendix~\ref{app:model_modifications}). We also evaluate logit-based distillation from a watermarked teacher to an initially unwatermarked student of the same architecture. The student retains most of the teacher's detection power at comparable perplexity (Appendix~\ref{app:distill}). For multi-key deployment, we examine whether text generated with one key is falsely attributed to another by testing independently sampled mismatched keys. The observed cross-key detection rates remain close to the nominal false-positive level, supporting key-specific attribution in the evaluated setting (Appendix~\ref{app:multi_key}). Finally, Appendix~\ref{sec:ablation} presents comprehensive ablation studies on how the regularization term, its coefficient, and the quality anchor affect the quality--detectability trade-off, examines sensitivity to GRPO group size and sampling temperature, and evaluates the effects of removing the watermark reward or replacing it with the KGW $z$-score.
\par
\endgroup

 % \begin{table}[!t]
% \centering
% \caption{Ablation study on components of \textsc{GaussMark}-\textsc{MarkTune}. Each color block includes different choices of components. Best results are highlighted in \textbf{bold}.}
% \begingroup
% \setlength{\tabcolsep}{6pt}
% \renewcommand{\arraystretch}{1.5}
% \resizebox{\textwidth}{!}{%
% \begin{tabular}{l c c c c c c c c c c c}
% \noalign{\hrule height 1.5pt}
% \textbf{\textbf{Method}} & \textbf{AUC$\uparrow$} & \textbf{TPR$\uparrow$} & \textbf{PPL$\downarrow$} & \textbf{Seq-rep-3$\downarrow$} & \textbf{MAUVE$\uparrow$} \\
% \hline
% % -------- Block 1 --------
% \cellcolor{blkA}\textsc{MarkTune}-CE-0.01   & Yes & Up Proj. & 8  & 0.1 & 0.5 & \textbf{0.895} & \textbf{0.990} & 13.59 & 3.118 & 0.306 & 37.97 \\
% \cellcolor{blkA}SFT   & Yes & Up Proj. & 8  & 0.1 & 0.5 & 0.627 & 0.966 & \textbf{12.03} & \textbf{3.088} & \textbf{0.310} & \textbf{37.21} \\
% \hline
% % -------- Block 2 --------
% \textsc{MarkTune} & \cellcolor{blkB}Yes & Up Proj. & 8 & 0.1 & 0.5 & \textbf{0.895} & \textbf{0.990} & \textbf{13.59} & \textbf{3.118} & \textbf{0.306} & \textbf{37.97} \\
% \textsc{MarkTune} & \cellcolor{blkB}No  & Up Proj. & 8 & 0.1 & 0.5 & 0.858 & 0.989 & 13.69 & 3.122 & 0.304 & 37.97\\
% \hline
% \end{tabular}}
% \label{tab:ablation}
% \endgroup
% \end{table}

% \vspace{-0.3em}
\endgroup
\raggedbottom
\section{Conclusion} \label{sec:conclusion}
% \vspace{-0.3em}

We introduced \marktune, a practical and theoretically grounded on-policy fine-tuning framework for watermarking open-weight language models. By optimizing a dual objective that combines a watermark-signal reward derived from the \gaussmark\ test statistic with   a quality regularizer that anchors the trained model to a high-quality reference distribution, instantiated either as the cross-entropy on human text or as the KL divergence to a stronger oracle model from the same family, \marktune\ enables watermark-aware yet quality-preserving exploration in the over-parameterized representation space. Extensive empirical evaluations across models and datasets show that \marktune\ consistently improves the quality-detectability trade-off of \gaussmark, reaching performance close to inference-time watermarking while preserving downstream task performance. Our approach is also very robust to semantic-preserving text-editing and open-weight attacks and generalizes well to unseen data. These results suggest that \marktune\ offers a flexible and effective framework for embedding reliable, high-quality watermarks into open-weight language models.

\begin{ack}
Z.S.W.\ was supported in part by an NSF CAREER Award and an NSF Award \#2232693.
\end{ack}


\phantomsection
\addcontentsline{toc}{section}{References}
\bibliography{neurips2026}

@article{1995wordnet,
  title={WordNet: a lexical database for {E}nglish},
  author={Miller, George A},
  journal={Communications of the ACM},
  volume={38},
  number={11},
  pages={39--41},
  year={1995},
  publisher={ACM New York, NY, USA}
}

@article{2017proximal,
  title={Proximal policy optimization algorithms},
  author={Schulman, John and Wolski, Filip and Dhariwal, Prafulla and Radford, Alec and Klimov, Oleg},
  journal={arXiv preprint arXiv:1707.06347},
  year={2017}
}

@inproceedings{2019defending,
  title={Defending against neural fake news},
  author={Zellers, Rowan and Holtzman, Ari and Rashkin, Hannah and Bisk, Yonatan and Farhadi, Ali and Roesner, Franziska and Choi, Yejin},
  booktitle={Advances in neural information processing systems},
  year={2019}
}

@inproceedings{fan2019eli5,
  title={{ELI}5: {L}ong {F}orm {Q}uestion {A}nswering},
  author={Fan, Angela and Jernite, Yacine and Perez, Ethan and Grangier, David and Weston, Jason and Auli, Michael},
  booktitle={Proceedings of the 57th Annual Meeting of the Association for Computational Linguistics},
  pages={3558--3567},
  year={2019}
}

@misc{Gokaslan2019OpenWeb,
    title={Open{W}eb{T}ext {C}orpus},
    author={Gokaslan, Aaron and Cohen, Vanya and Pavlick, Ellie and Tellex, Stefanie},
    howpublished={\url{http://Skylion007.github.io/OpenWebTextCorpus}},
    year={2019}
}

@article{raffel2020exploring,
  title={Exploring the limits of transfer learning with a unified text-to-text transformer},
  author={Raffel, Colin and Shazeer, Noam and Roberts, Adam and Lee, Katherine and Narang, Sharan and Matena, Michael and Zhou, Yanqi and Li, Wei and Liu, Peter J},
  journal={Journal of machine learning research},
  volume={21},
  number={140},
  pages={1--67},
  year={2020}
}

@article{2020deep,
  title={Deep learning for misinformation detection on online social networks: a survey and new perspectives},
  author={Islam, Md Rafiqul and Liu, Shaowu and Wang, Xianzhi and Xu, Guandong},
  journal={Social Network Analysis and Mining},
  volume={10},
  number={1},
  pages={82},
  year={2020},
  publisher={Springer}
}

@article{2021ethical,
  title={Ethical and social risks of harm from language models},
  author={Weidinger, Laura and Mellor, John and Rauh, Maribeth and Griffin, Conor and Uesato, Jonathan and Huang, Po-Sen and Cheng, Myra and Glaese, Mia and Balle, Borja and Kasirzadeh, Atoosa and others},
  journal={arXiv preprint arXiv:2112.04359},
  year={2021}
}

@article{2020mmlu,
  title={Measuring massive multitask language understanding},
  author={Hendrycks, Dan and Burns, Collin and Basart, Steven and Zou, Andy and Mazeika, Mantas and Song, Dawn and Steinhardt, Jacob},
  journal={arXiv preprint arXiv:2009.03300},
  year={2020}
}

@article{cobbe2021gsm8k,
  title={Training verifiers to solve math word problems},
  author={Cobbe, Karl and Kosaraju, Vineet and Bavarian, Mohammad and Chen, Mark and Jun, Heewoo and Kaiser, Lukasz and Plappert, Matthias and Tworek, Jerry and Hilton, Jacob and Nakano, Reiichiro and others},
  journal={arXiv preprint arXiv:2110.14168},
  year={2021}
}

@article{2021mbpp,
  title={Program synthesis with large language models},
  author={Austin, Jacob and Odena, Augustus and Nye, Maxwell and Bosma, Maarten and Michalewski, Henryk and Dohan, David and Jiang, Ellen and Cai, Carrie and Terry, Michael and Le, Quoc and others},
  journal={arXiv preprint arXiv:2108.07732},
  year={2021}
}

@article{2021framework,
  title={A framework for few-shot language model evaluation},
  author={Gao, Leo and Tow, Jonathan and Biderman, Stella and Black, Sid and DiPofi, Anthony and Foster, Charles and Golding, Laurence and Hsu, Jeffrey and McDonell, Kyle and Muennighoff, Niklas and others},
  journal={Zenodo},
  year={2021}
}

@inproceedings{2021mauve,
  title={Mauve: {M}easuring the gap between neural text and human text using divergence frontiers},
  author={Pillutla, Krishna and Swayamdipta, Swabha and Zellers, Rowan and Thickstun, John and Welleck, Sean and Choi, Yejin and Harchaoui, Zaid},
  booktitle={Advances in Neural Information Processing Systems},
  year={2021}
}

@article{2022ai,
  title={{AI} bot {C}hat{GPT} writes smart essays-should professors worry?},
  author={Stokel-Walker, Chris},
  journal={Nature News},
  year={2022}
}

@inproceedings{ouyang2022training,
  title={Training language models to follow instructions with human feedback},
  author={Ouyang, Long and Wu, Jeffrey and Jiang, Xu and Almeida, Diogo and Wainwright, Carroll and Mishkin, Pamela and Zhang, Chong and Agarwal, Sandhini and Slama, Katarina and Ray, Alex and others},
  booktitle={Advances in neural information processing systems},
  year={2022}
}

@article{2022short,
  title={A short note on an inequality between {KL} and {TV}},
  author={Canonne, Cl{\'e}ment L},
  journal={arXiv preprint arXiv:2202.07198},
  year={2022}
}

@article{2023chatgpt,
  title={ChatGPT for good? On opportunities and challenges of large language models for education},
  author={Kasneci, Enkelejda and Se{\ss}ler, Kathrin and K{\"u}chemann, Stefan and Bannert, Maria and Dementieva, Daryna and Fischer, Frank and Gasser, Urs and Groh, Georg and G{\"u}nnemann, Stephan and H{\"u}llermeier, Eyke and others},
  journal={Learning and individual differences},
  volume={103},
  pages={102274},
  year={2023},
  publisher={Elsevier}
}

@article{2023large,
  title={Large language models challenge the future of higher education},
  author={Milano, Silvia and McGrane, Joshua A and Leonelli, Sabina},
  journal={Nature Machine Intelligence},
  volume={5},
  number={4},
  pages={333--334},
  year={2023},
  publisher={Nature Publishing Group UK London}
}

@inproceedings{2023KGW,
  title={A watermark for large language models},
  author={Kirchenbauer, John and Geiping, Jonas and Wen, Yuxin and Katz, Jonathan and Miers, Ian and Goldstein, Tom},
  booktitle={International Conference on Machine Learning},
  year={2023}
}

@inproceedings{rafailov2023direct,
  title={Direct preference optimization: Your language model is secretly a reward model},
  author={Rafailov, Rafael and Sharma, Archit and Mitchell, Eric and Manning, Christopher D and Ermon, Stefano and Finn, Chelsea},
  booktitle={Advances in neural information processing systems},
  year={2023}
}

@article{2023llama,
  title={Llama 2: Open foundation and fine-tuned chat models},
  author={Touvron, Hugo and Martin, Louis and Stone, Kevin and Albert, Peter and Almahairi, Amjad and Babaei, Yasmine and Bashlykov, Nikolay and Batra, Soumya and Bhargava, Prajjwal and Bhosale, Shruti and others},
  journal={arXiv preprint arXiv:2307.09288},
  year={2023}
}

@article{2023Robust,
    title={Robust {D}istortion-free {W}atermarks for {L}anguage {M}odels},
    author={Rohith Kuditipudi and John Thickstun and Tatsunori Hashimoto and Percy Liang},
    journal={Transactions on Machine Learning Research},
    issn={2835-8856},
    year={2024}
}

@article{krishna2023paraphrasing,
  title={Paraphrasing evades detectors of ai-generated text, but retrieval is an effective defense},
  author={Krishna, Kalpesh and Song, Yixiao and Karpinska, Marzena and Wieting, John and Iyyer, Mohit},
  journal={Advances in Neural Information Processing Systems},
  volume={36},
  pages={27469--27500},
  year={2023}
}

@article{2023towards,
  title={Towards optimal statistical watermarking},
  author={Huang, Baihe and Zhu, Hanlin and Zhu, Banghua and Ramchandran, Kannan and Jordan, Michael I and Lee, Jason D and Jiao, Jiantao},
  journal={arXiv preprint arXiv:2312.07930},
  year={2023}
}

@misc{2023Aaronson,
  author       = {Aaronson, Scott and Kirchner, Hendrik},
  title        = {Watermarking of Large Language Models},
  howpublished = {\url{https://www.scottaaronson.com/talks/watermark.ppt}},
  year         = {2022}
}

@article{2023dipmark,
  author       = {Wu, Yihan and Hu, Zhengmian and Zhang, Hongyang and Huang, Heng},
  title        = {A Resilient and Accessible Distribution-Preserving Watermark for Large Language Models},
  year         = {2023},
  journal={arXiv preprint arXiv:2310.07710}
}

@inproceedings{hu2023unbiased,
  title={Unbiased watermark for large language models},
  author={Hu, Zhengmian and Chen, Lichang and Wu, Xidong and Wu, Yihan and Zhang, Hongyang and Huang, Heng},
  booktitle={International Conference on Learning Representations},
  year={2024}
}

@article{shao2024deepseekmath,
  title={Deepseekmath: Pushing the limits of mathematical reasoning in open language models},
  author={Shao, Zhihong and Wang, Peiyi and Zhu, Qihao and Xu, Runxin and Song, Junxiao and Bi, Xiao and Zhang, Haowei and Zhang, Mingchuan and Li, YK and Wu, Yang and others},
  journal={arXiv preprint arXiv:2402.03300},
  year={2024}
}

@inproceedings{christ2024undetectable,
  title={Undetectable watermarks for language models},
  author={Christ, Miranda and Gunn, Sam and Zamir, Or},
  booktitle={The Thirty Seventh Annual Conference on Learning Theory},
  pages={1125--1139},
  year={2024},
  organization={PMLR}
}

@article{2024debiasing,
  title={Debiasing watermarks for large language models via maximal coupling},
  author={Xie, Yangxinyu and Li, Xiang and Mallick, Tanwi and Su, Weijie and Zhang, Ruixun},
  journal={Journal of the American Statistical Association},
  volume={120},
  number={551},
  pages={1424--1436},
  year={2025},
  publisher={Taylor \& Francis}
}

@article{2024provably,
  title={Provably robust watermarks for open-source language models},
  author={Christ, Miranda and Gunn, Sam and Malkin, Tal and Raykova, Mariana},
  journal={arXiv preprint arXiv:2410.18861},
  year={2024}
}

@inproceedings{2024zhaopermute,
  title={Permute-and-{F}lip: An optimally stable and watermarkable decoder for {LLM}s},
  author={Zhao, Xuandong and Li, Lei and Wang, Yu-Xiang},
  booktitle={The Thirteenth International Conference on Learning Representations},
  year={2025}
}

@article{wu2024distortion,
  title={Distortion-free watermarks are not truly distortion-free under watermark key collisions},
  author={Wu, Yihan and Chen, Ruibo and Hu, Zhengmian and Chen, Yanshuo and Guo, Junfeng and Zhang, Hongyang and Huang, Heng},
  journal={arXiv preprint arXiv:2406.02603},
  year={2024}
}

@inproceedings{2024ProvableRobust,
  title={Provable robust watermarking for {AI}-generated text},
  author={Zhao, Xuandong and Ananth, Prabhanjan and Li, Lei and Wang, Yu-Xiang},
  booktitle={International Conference on Learning Representations},
  year={2024}
}

@article{2024revisiting,
  title={Revisiting the robustness of watermarking to paraphrasing attacks},
  author={Rastogi, Saksham and Pruthi, Danish},
  journal={arXiv preprint arXiv:2411.05277},
  year={2024}
}

@inproceedings{2024downstream,
  title={Downstream trade-offs of a family of text watermarks},
  author={Ajith, Anirudh and Singh, Sameer and Pruthi, Danish},
  booktitle={Findings of the Association for Computational Linguistics: EMNLP 2024},
  pages={14039--14053},
  year={2024}
}

@inproceedings{2024watermarkstealing,
    author = {Jovanović, Nikola and Staab, Robin and Vechev, Martin},
    title = {Watermark Stealing in Large Language Models},
    booktitle = {International Conference on Machine Learning},
    year = {2024}
}

@article{2024llama,
  title={The llama 3 herd of models},
  author={Grattafiori, Aaron and Dubey, Abhimanyu and Jauhri, Abhinav and Pandey, Abhinav and Kadian, Abhishek and Al-Dahle, Ahmad and Letman, Aiesha and Mathur, Akhil and Schelten, Alan and Vaughan, Alex and others},
  journal={arXiv preprint arXiv:2407.21783},
  year={2024}
}

@article{2024democratizing,
  title={Democratizing neural machine translation with {OPUS-MT}},
  author={Tiedemann, J{\"o}rg and Aulamo, Mikko and Bakshandaeva, Daria and Boggia, Michele and Gr{\"o}nroos, Stig-Arne and Nieminen, Tommi and Raganato, Alessandro and Scherrer, Yves and V{\'a}zquez, Ra{\'u}l and Virpioja, Sami},
  journal={Language Resources and Evaluation},
  volume={58},
  number={2},
  pages={713--755},
  year={2024},
  publisher={Springer}
}

@article{2024scalable,
  title={Scalable watermarking for identifying large language model outputs},
  author={Dathathri, Sumanth and See, Abigail and Ghaisas, Sumedh and Huang, Po-Sen and McAdam, Rob and Welbl, Johannes and Bachani, Vandana and Kaskasoli, Alex and Stanforth, Robert and Matejovicova, Tatiana and others},
  journal={Nature},
  volume={634},
  number={8035},
  pages={818--823},
  year={2024},
  publisher={Nature Publishing Group UK London}
}

@inproceedings{2024gu,
  title={On the Learnability of Watermarks for Language Models},
  author={Chenchen Gu and Xiang Lisa Li and Percy Liang and Tatsunori Hashimoto.},
  booktitle={International Conference on Learning Representations},
  year={2024}
}

@article{xu2024learning,
  title={Learning to watermark llm-generated text via reinforcement learning},
  author={Xu, Xiaojun and Yao, Yuanshun and Liu, Yang},
  journal={arXiv preprint arXiv:2403.10553},
  year={2024}
}

@article{2024generative,
  title={Generative {AI} in {EU} law: Liability, privacy, intellectual property, and cybersecurity},
  author={Novelli, Claudio and Casolari, Federico and Hacker, Philipp and Spedicato, Giorgio and Floridi, Luciano},
  journal={Computer Law \& Security Review},
  volume={55},
  pages={106066},
  year={2024},
  publisher={Elsevier}
}

@article{2024intellectual,
  title={Intellectual property protection for deep learning model and dataset intelligence},
  author={Jiang, Yongqi and Gao, Yansong and Zhou, Chunyi and Hu, Hongsheng and Fu, Anmin and Susilo, Willy},
  journal={arXiv preprint arXiv:2411.05051},
  year={2024}
}

@article{2025towards,
  title={Towards watermarking of open-source llms},
  author={Gloaguen, Thibaud and Jovanovi{\'c}, Nikola and Staab, Robin and Vechev, Martin},
  journal={arXiv preprint arXiv:2502.10525},
  year={2025}
}

@article{guo2025deepseek,
  title={Deepseek-r1: Incentivizing reasoning capability in llms via reinforcement learning},
  author={Guo, Daya and Yang, Dejian and Zhang, Haowei and Song, Junxiao and Zhang, Ruoyu and Xu, Runxin and Zhu, Qihao and Ma, Shirong and Wang, Peiyi and Bi, Xiao and others},
  journal={arXiv preprint arXiv:2501.12948},
  year={2025}
}

@inproceedings{2025gaussmark,
  title={GaussMark: A Practical Approach for Structural Watermarking of Language Models},
  author={Block, Adam and Rakhlin, Alexander and Sekhari, Ayush},
  booktitle={International Conference on Machine Learning},
  year={2025}
}

@article{2025statistical,
    title={A statistical framework of watermarks for large language models: {P}ivot, detection efficiency and optimal rules},
    author={Li, Xiang and Ruan, Feng and Wang, Huiyuan and Long, Qi and Su, Weijie J},
    journal={The Annals of Statistics},
    volume={53},
    number={1},
    pages={322--351},
    year={2025},
    publisher={Institute of Mathematical Statistics}
}

@article{xu2025copyright,
  title={Copyright Protection for Large Language Models: A Survey of Methods, Challenges, and Trends},
  author={Xu, Zhenhua and Yue, Xubin and Wang, Zhebo and Zhang, Haobo and Liu, Qichen and Zhao, Xixiang and Zhang, Jingxuan and Zeng, Wenjun and Xing, Wengpeng and Kong, Dezhang and Lin, Changting and Han, Meng},
  journal={arXiv preprint arXiv:2508.11548},
  year={2025}
}

@article{2025understanding,
  title={Unveiling the Basin-Like Loss Landscape in Large Language Models},
  author={Chen, Huanran and Dong, Yinpeng and Wei, Zeming and Huang, Yao and Zhang, Yichi and Su, Hang and Zhu, Jun},
  journal={arXiv preprint arXiv:2505.17646},
  year={2025}
}

@article{2025qwen3,
  title={Qwen3 technical report},
  author={Yang, An and Li, Anfeng and Yang, Baosong and Zhang, Beichen and Hui, Binyuan and Zheng, Bo and Yu, Bowen and Gao, Chang and Huang, Chengen and Lv, Chenxu and others},
  journal={arXiv preprint arXiv:2505.09388},
  year={2025}
}

@article{foster2025good,
  title={Is a Good Foundation Necessary for Efficient Reinforcement Learning? The Computational Role of the Base Model in Exploration},
  author={Foster, Dylan J and Mhammedi, Zakaria and Rohatgi, Dhruv},
  journal={arXiv preprint arXiv:2503.07453},
  year={2025}
}

@inproceedings{huang2025self,
  title={Self-Improvement in Language Models: The Sharpening Mechanism},
  author={Huang, Audrey and Block, Adam and Foster, Dylan J and Rohatgi, Dhruv and Zhang, Cyril and Simchowitz, Max and Ash, Jordan T and Krishnamurthy, Akshay},
  booktitle={The Thirteenth International Conference on Learning Representations},
  year={2025}
}

@inproceedings{foret2021sharpness,
  title={Sharpness-Aware Minimization for Efficiently Improving Generalization},
  author={Foret, Pierre and Kleiner, Ariel and Mobahi, Hossein and Neyshabur, Behnam},
  booktitle={International Conference on Learning Representations},
  year={2021}
}

@inproceedings{huo2026pmark,
  title={{PMark}: Towards Robust and Distortion-free Semantic-level Watermarking with Channel Constraints},
  author={Huo, Jiahao and others},
  booktitle={International Conference on Learning Representations},
  year={2026}
}

@article{laurent2000adaptive,
  title={Adaptive estimation of a quadratic functional by model selection},
  author={Laurent, Beatrice and Massart, Pascal},
  journal={Annals of Statistics},
  volume={28},
  number={5},
  pages={1302--1338},
  year={2000}
}

@book{nesterov2018lectures,
  title={Lectures on Convex Optimization},
  author={Nesterov, Yurii},
  edition={2},
  publisher={Springer},
  year={2018}
}

@article{pan2024markllm,
  title={{MarkLLM}: An open-source toolkit for {LLM} watermarking},
  author={Pan, Leyi and Liu, Aiwei and He, Zhiwei and Gao, Zitian and Zhao, Xuandong and Lu, Yijian and Zhou, Binglin and Liu, Shuliang and Hu, Xuming and Wen, Lijie and King, Irwin and Yu, Philip S.},
  journal={arXiv preprint arXiv:2405.10051},
  year={2024}
}

@article{feng2025soft,
  title={Soft Metropolis-Hastings Correction for Generative Model Sampling},
  author={Feng, Hanqi and Qiu, Peng and Zhang, Meng-Chun and Fan, You and Tao, Yiran and Poczos, Barnabas},
  journal={bioRxiv},
  year={2025},
  publisher={Cold Spring Harbor Laboratory}
}

@article{feng2025adaptive,
  title={Learning from b cell evolution: Adaptive multi-expert diffusion for antibody design via online optimization},
  author={Feng, Hanqi and Qiu, Peng and Zhang, Mengchun and Tao, Yiran and Fan, You and Xu, Jingtao and Poczos, Barnabas},
  journal={arXiv preprint arXiv:2508.02834},
  year={2025}
}

@inproceedings{pang2024nofreelunch,
  title={No free lunch in llm watermarking: Trade-offs in watermarking design choices},
  author={Pang, Qi and Hu, Shengyuan and Zheng, Wenting and Smith, Virginia},
  booktitle={Advances in Neural Information Processing Systems},
  volume={37},
  pages={138756--138788},
  year={2024}
}

@inproceedings{fernandez2023stable,
  title={The {S}table {S}ignature: Rooting watermarks in latent diffusion models},
  author={Fernandez, Pierre and Couairon, Guillaume and J{\'e}gou, Herv{\'e} and Douze, Matthijs and Furon, Teddy},
  booktitle={Proceedings of the IEEE/CVF International Conference on Computer Vision},
  pages={22466--22477},
  year={2023}
}

@inproceedings{wen2023treerings,
  title={{Tree-Rings} watermarks: Invisible fingerprints for diffusion images},
  author={Wen, Yuxin and Kirchenbauer, John and Geiping, Jonas and Goldstein, Tom},
  booktitle={Advances in Neural Information Processing Systems},
  volume={36},
  pages={58047--58063},
  year={2023}
}

@inproceedings{yang2024gaussian,
  title={{G}aussian {S}hading: Provable performance-lossless image watermarking for diffusion models},
  author={Yang, Zijin and Zeng, Kai and Chen, Kejiang and Fang, Han and Zhang, Weiming and Yu, Nenghai},
  booktitle={Proceedings of the IEEE/CVF Conference on Computer Vision and Pattern Recognition},
  pages={12162--12171},
  year={2024}
}

@article{fernandez2024videoseal,
  title={{Video Seal}: Open and efficient video watermarking},
  author={Fernandez, Pierre and Elsahar, Hady and Yalniz, I. Zeki and Mourachko, Alexandre},
  journal={arXiv preprint arXiv:2412.09492},
  year={2024}
}

@inproceedings{hu2025videoshield,
  title={{VideoShield}: Regulating diffusion-based video generation models via watermarking},
  author={Hu, Runyi and Zhang, Jie and Li, Yiming and Li, Jiwei and Guo, Qing and Qiu, Han and Zhang, Tianwei},
  booktitle={International Conference on Learning Representations},
  pages={52395--52425},
  year={2025}
}

@article{chen2023wavmark,
  title={{WavMark}: Watermarking for audio generation},
  author={Chen, Guangyu and Wu, Yu and Liu, Shujie and Liu, Tao and Du, Xiaoyong and Wei, Furu},
  journal={arXiv preprint arXiv:2308.12770},
  year={2023}
}

@inproceedings{sanroman2024proactive,
  title={Proactive detection of voice cloning with localized watermarking},
  author={San Roman, Robin and Fernandez, Pierre and Elsahar, Hady and D{\'e}fossez, Alexandre and Furon, Teddy and Tran, Tuan},
  booktitle={International Conference on Machine Learning},
  pages={43180--43196},
  year={2024}
}

@inproceedings{zhu2025tabwak,
  title={{TabWak}: A watermark for tabular diffusion models},
  author={Zhu, Chaoyi and Tang, Jiayi and Galjaard, Jeroen and Chen, Pin-Yu and Birke, Robert and Bos, Cornelis and Chen, Lydia},
  booktitle={International Conference on Learning Representations},
  pages={42174--42201},
  year={2025}
}

@inproceedings{fang2026muse,
  title={{MUSE}: Model-agnostic tabular watermarking via multi-sample selection},
  author={Fang, Liancheng and Liu, Aiwei and Zou, Henry Peng and Chen, Yankai and Zhang, Hengrui and Deng, Zhongfen and Yu, Philip},
  booktitle={International Conference on Learning Representations},
  pages={16541--16574},
  year={2026}
}

@inproceedings{ci2024ringid,
  title={{RingID}: Rethinking {Tree-Ring} watermarking for enhanced multi-key identification},
  author={Ci, Hai and Yang, Pei and Song, Yiren and Shou, Mike Zheng},
  booktitle={European Conference on Computer Vision},
  pages={338--354},
  year={2024}
}

@inproceedings{feng2024aqualora,
  title={{AquaLoRA}: Toward white-box protection for customized {S}table {D}iffusion models via watermark {LoRA}},
  author={Feng, Weitao and Zhou, Wenbo and He, Jiyan and Zhang, Jie and Wei, Tianyi and Li, Guanlin and Zhang, Tianwei and Zhang, Weiming and Yu, Nenghai},
  booktitle={International Conference on Machine Learning},
  pages={13423--13444},
  year={2024}
}

@inproceedings{jiang2024certifiably,
  title={Certifiably robust image watermark},
  author={Jiang, Zhengyuan and Guo, Moyang and Hu, Yuepeng and Jia, Jinyuan and Gong, Neil Zhenqiang},
  booktitle={European Conference on Computer Vision},
  pages={427--443},
  year={2024}
}

@inproceedings{bui2025trustmark,
  title={{TrustMark}: Robust watermarking and watermark removal for arbitrary resolution images},
  author={Bui, Tu and Agarwal, Shruti and Collomosse, John},
  booktitle={Proceedings of the IEEE/CVF International Conference on Computer Vision},
  pages={18629--18639},
  year={2025}
}

@inproceedings{sander2025watermark,
  title={Watermark anything with localized messages},
  author={Sander, Tom and Fernandez, Pierre and Oliviero Durmus, Alain and Furon, Teddy and Douze, Matthijs},
  booktitle={International Conference on Learning Representations},
  pages={79569--79599},
  year={2025}
}

@article{luo2023dvmark,
  title={{DVMark}: A deep multiscale network for video watermarking},
  author={Luo, Xiyang and Li, Yinxiao and Chang, Huiwen and Liu, Ce and Milanfar, Peyman and Yang, Feng},
  journal={IEEE Transactions on Image Processing},
  year={2023}
}

@inproceedings{zhang2023revmark,
  title={A novel deep video watermarking framework with enhanced robustness to {H.264/AVC} compression},
  author={Zhang, Yulin and Ni, Jiangqun and Su, Wenkang and Liao, Xin},
  booktitle={Proceedings of the 31st ACM International Conference on Multimedia},
  pages={8095--8104},
  year={2023}
}

@article{hu2025videomark,
  title={{VideoMark}: A distortion-free robust watermarking framework for video diffusion models},
  author={Hu, Xuming and Li, Hanqian and Li, Jungang and Huang, Yu and Liu, Shuliang and Zheng, Qi and Chen, Junhao and Liu, Aiwei},
  journal={arXiv preprint arXiv:2504.16359},
  year={2025}
}

@inproceedings{liu2024timbre,
  title={Detecting voice cloning attacks via timbre watermarking},
  author={Liu, Chang and Zhang, Jie and Zhang, Tianwei and Yang, Xi and Zhang, Weiming and Yu, Nenghai},
  booktitle={Network and Distributed System Security Symposium},
  year={2024}
}

@inproceedings{singh2024silentcipher,
  title={{SilentCipher}: Deep audio watermarking},
  author={Singh, Mayank Kumar and Takahashi, Naoya and Liao, Weihsiang and Mitsufuji, Yuki},
  booktitle={Interspeech 2024},
  pages={2235--2239},
  year={2024}
}

@inproceedings{sanroman2025latent,
  title={Latent watermarking of audio generative models},
  author={San Roman, Robin and Fernandez, Pierre and Deleforge, Antoine and Adi, Yossi and Serizel, Romain},
  booktitle={IEEE International Conference on Acoustics, Speech and Signal Processing},
  pages={1--5},
  year={2025}
}

@article{he2024watermarking,
  title={Watermarking generative tabular data},
  author={He, Hengzhi and Yu, Peiyu and Ren, Junpeng and Wu, Ying Nian and Cheng, Guang},
  journal={arXiv preprint arXiv:2405.14018v1},
  year={2024}
}

@inproceedings{zheng2024tabularmark,
  title={{TabularMark}: Watermarking tabular datasets for machine learning},
  author={Zheng, Yihao and Xia, Haocheng and Pang, Junyuan and Liu, Jinfei and Ren, Kui and Chu, Lingyang and Cao, Yang and Xiong, Li},
  booktitle={Proceedings of the 2024 on ACM SIGSAC Conference on Computer and Communications Security},
  pages={3570--3584},
  year={2024}
}

@article{zhao2026robust,
  title={Robust spectral watermark for synthetic tabular data},
  author={Zhao, Yizhou and Li, Xiang and Song, Peter XK and Long, Qi and Su, Weijie},
  journal={Statistical Learning and Data Science},
  volume={1},
  pages={100001},
  year={2026},
  publisher={Elsevier},
  doi={10.1016/j.slads.2026.100001}
}
\bibliographystyle{plainnat}

%%%%%%%%%%%%%%%%%%%%%%%%%%%%%%%%%%%%%%%%%%%%%%%%%%%%%%%%%%%%%%%%%%%%%%%%%%%%%%%
%%%%%%%%%%%%%%%%%%%%%%%%%%%%%%%%%%%%%%%%%%%%%%%%%%%%%%%%%%%%%%%%%%%%%%%%%%%%%%%
% APPENDIX
%%%%%%%%%%%%%%%%%%%%%%%%%%%%%%%%%%%%%%%%%%%%%%%%%%%%%%%%%%%%%%%%%%%%%%%%%%%%%%%
%%%%%%%%%%%%%%%%%%%%%%%%%%%%%%%%%%%%%%%%%%%%%%%%%%%%%%%%%%%%%%%%%%%%%%%%%%%%%%%
\clearpage
\setcounter{tocdepth}{3}
\tableofcontents
\clearpage
\appendix

\section{Practical Implementation of \marktune} \label{app:alg_details}

\begin{algorithm}[H]
  \caption{GRPO for \textsc{MarkTune}}
  \label{alg:GRPO-PO}
  \begin{algorithmic}[1]
    \STATE \textbf{Input}: initial model $p_{\theta(\xi_\sigma)}$; watermark signal reward $\mathcal{R}_{\mathrm{wm}}(\cdot;\xi_\sigma)$; text dataset $\mathcal{D}$; hyperparameters $\varepsilon, \lambda, T_1, T_2$.
    \STATE \textbf{Initialize}: Set $p_\theta \gets p_{\theta(\xi_\sigma)}$.

  \FOR{step $= 1, \ldots, T_1$}
    \STATE Sample a batch $\mathcal{D}_b \subset \mathcal{D}$.
    \STATE $p_{\theta_{\mathrm{old}}} \gets p_\theta$.
    \FOR{each prompt $x \in \mathcal{D}_b$}
      \STATE Sample $G$ outputs $\{y_j\}_{j=1}^{G} \sim p_{\theta_{\mathrm{old}}}(\cdot \mid x)$.
      \STATE Compute rewards $\{r_j\}_{j=1}^{G}$ using $\mathcal{R}_{\mathrm{wm}}$.
      \STATE Compute $\hat{A}_j$ for response $y_j$ via group-relative advantage Eq.(\ref{eq;compute_A}).
    \ENDFOR
    \FOR{iteration $= 1, \ldots, T_2$}
      \STATE Update $p_\theta$ by maximizing the GRPO objective Eq.(\ref{eq:grpo_obj}).
    \ENDFOR
  \ENDFOR
    \STATE \textbf{Output}: $p_\theta$.
  \end{algorithmic}
\end{algorithm}

% \begin{algorithm}[!t]
%   \caption{Application of \marktune\ to \gaussmark}
%   \label{alg:gaussmark-marktune}
%   \begin{algorithmic}[1]
%     \STATE \textbf{Input:} Language model $p_\theta$ with $\theta = (\theta_{\mathrm{wm}}, \theta_0)$, fixed base model $p_\theta = p_\theta$, strength $\sigma>0$, text dataset $\mathcal{D}$, regularization coefficient $\lambda > 0$.
%     \STATE Sample watermark key
%       $\displaystyle \xi_\sigma \sim \mathcal{N}(0, \sigma^2\mathbb{I}_{d_r})$.
%     \STATE Inject GaussMark perturbation
%       $\theta^\star_{\mathrm{wm}} \gets \theta_{\mathrm{wm}} + \xi_\sigma$ and obtain $\theta(\xi) = (\theta^\star_{\mathrm{wm}},\theta_0)$.
%     \STATE Conduct GRPO Algorithm~\ref{alg:GRPO-PO} for $\theta(\xi)$ to obtain $\theta^\star(\xi)$ via optimizing the objective
%     $$
%     \max_{\theta(\xi)}
%     \mathbb{E}_{x\sim\mathcal{D}, y \sim p_{\theta(\xi)}(\cdot \mid x)}\!\left[
%     \frac{\langle \xi_\sigma,\; \nabla_{\theta_{\mathrm{wm}}}\log p_\theta (y)\rangle}
%     {\sigma\|\nabla_{\theta_{\mathrm{wm}}}\log p_\theta (y)\|}
%     \right] - \lambda \mathcal{L}_{\mathrm{reg}}(\theta(\xi);\,x,y),
%     $$
%     \STATE \textbf{Output:} Watermarked parameters $\theta^\star(\xi)$.
%   \end{algorithmic}
% \end{algorithm}

To optimize the objective in Eq.(\ref{eq:overall_opt}), we use the GRPO~\citep{shao2024deepseekmath} algorithm. This choice is motivated by the fact that Eq.(\ref{eq:overall_opt}) cannot be directly optimized due to its dependence on sampled outputs $y \sim p_{\theta(\xi_\sigma)}$. GRPO is a critic-free, on-policy policy-gradient method that replaces the learned value baseline in PPO with a group-relative baseline computed from multiple samples per prompt, which makes it well-suited for optimizing our objective~\citep{feng2025adaptive}.

For each prompt from a batch $x \sim \mathcal{D}_b$, we sample a group of responses $\{y_j\}_{j=1}^G$ from the current policy and compute their watermark signal rewards $\mathbf{r} = \{r_j\}_{j=1}^G$. The advantage for sample $j$ is normalized within the group:
\begin{equation} \label{eq;compute_A}
    \widehat A_j = \big(r_j - \mathrm{mean}(\mathbf{r})\big)\big/\mathrm{std}(\mathbf{r})
\end{equation}
The policy update follows a clipped objective based on the importance ratio $\rho_j = \frac{p_{\theta}(y_j\mid x)}{p_{\theta_{\mathrm{old}}}(y_j\mid x)}$, maximizing $\mathbb{E}\!\left[\min\!\big(\rho_j\widehat A_j,\ \mathrm{clip}(\rho_j,1-\varepsilon,1+\varepsilon)\widehat A_j\big)\right]$ together with a specified language modeling regularization term. The overall GRPO objective is given by
\begin{equation}\label{eq:grpo_obj}
\begin{aligned}
    \mathcal{J}(\theta) &=
    \mathbb{E}_{x\sim \mathcal{D},\{y_j\}_{j=1}^G \sim p_{\theta_{\mathrm{old}}}(\cdot\mid x)} \\
    &\quad \frac{1}{G}\sum_{j=1}^{G}\min \left[\frac{p_{\theta}(y_j \mid x)}{p_{\theta_{\mathrm{old}}}(y_j \mid x)}\,\hat{A}_j,\,
    \mathrm{clip}\!\left(\frac{p_{\theta}(y_j \mid x)}{p_{\theta_{\mathrm{old}}}(y_j \mid x)},\,1-\varepsilon,\,1+\varepsilon \right)\hat{A}_j\right]
    \\
    &\quad -\lambda \mathcal{L}_{\mathrm{reg}}(\theta;x,\{y_j\}_{j=1}^G).
\end{aligned}
\end{equation}

See Algorithm~\ref{alg:GRPO-PO} for the detailed implementation of GRPO.

\section{Theorems and Proofs} \label{app:proofs}

\subsection{Proof of Proposition~\ref{prop:tv_bound}} \label{app:proof_prop1}
\begin{proof}

Fix a prompt $x$ and abbreviate $p_\theta := p_\theta(\cdot \mid x)$. Define
\[
F : \mathbb{R}^{d_r} \to \mathbb{R}, \qquad F(\Delta) \;:=\; \mathrm{D_{KL}}\!\left(p_{(\theta_{\mathrm{wm}}+\Delta,\,\theta_0)}(\cdot \mid x)\,\big\|\,p_\theta\right).
\]
The proof proceeds in three steps: (i) verify $F(0) = 0$ and $\nabla F(0) = 0$; (ii) use a uniform bound on the Hessian of $F$ to derive a pointwise quadratic upper bound on $F$; (iii) average over $\xi_\sigma \sim \mathcal{N}(0,\sigma^2 \mathbb{I}_{d_r})$. A short corollary translates the resulting KL bound into an expected-TV bound.

\smallskip
\noindent\textit{Step 1: $F(0) = 0$ and $\nabla F(0) = 0$.}
At $\Delta = 0$, the perturbed parameters coincide with the base parameters, so $p_{(\theta_{\mathrm{wm}}+\Delta, \theta_0)} = p_\theta$ and $F(0) = \mathrm{D_{KL}}(p_\theta \,\|\, p_\theta) = 0$. The KL divergence is nonnegative, so $\Delta = 0$ is a global minimizer of $F$. The differentiability assumption together with the first-order optimality condition then gives $\nabla F(0) = 0$.

\smallskip
\noindent\textit{Step 2: pointwise quadratic upper bound.}
By assumption, $F$ is twice continuously differentiable with $\sup_{\Delta \in \mathbb{R}^{d_r}} \|\nabla^2 F(\Delta)\|_{\mathrm{op}} \le M$ for some constant $M > 0$. Taylor's theorem with integral remainder, combined with $F(0) = 0$ and $\nabla F(0) = 0$, yields
\begin{equation}\label{eq:taylor-integral}
F(\Delta) \;=\; \int_0^1 (1-t)\, \Delta^\top \nabla^2 F(t\Delta)\, \Delta\, dt.
\end{equation}
For every $t \in [0,1]$, the operator-norm bound on the Hessian gives $\Delta^\top \nabla^2 F(t\Delta)\, \Delta \le \|\nabla^2 F(t\Delta)\|_{\mathrm{op}}\, \|\Delta\|^2 \le M\,\|\Delta\|^2$. Substituting this estimate into~\eqref{eq:taylor-integral} and using $\int_0^1 (1-t)\,dt = 1/2$,
\begin{equation}\label{eq:F-quadratic}
F(\Delta) \;\le\; \frac{M}{2}\,\|\Delta\|^2, \qquad \forall\,\Delta \in \mathbb{R}^{d_r}.
\end{equation}

\smallskip
\noindent\textit{Step 3: averaging over $\xi_\sigma$.}
Specializing $\Delta = \xi_\sigma$ in~\eqref{eq:F-quadratic} and taking expectation over $\xi_\sigma \sim \mathcal{N}(0, \sigma^2 \mathbb{I}_{d_r})$, we use $\mathbb{E}\,\|\xi_\sigma\|^2 = \sigma^2 d_r$ to obtain
\[
\mathbb{E}_{\xi_\sigma}\!\left[\mathrm{D_{KL}}\!\left(p_{\theta(\xi_\sigma)}(\cdot \mid x)\,\big\|\,p_\theta\right)\right] \;=\; \mathbb{E}_{\xi_\sigma}[F(\xi_\sigma)] \;\le\; \frac{M}{2}\,\mathbb{E}\,\|\xi_\sigma\|^2 \;=\; \frac{M}{2}\,\sigma^2 d_r,
\]
which establishes the bound with $C = M/2$.

\smallskip
\noindent\textit{Remark on the Hessian bound.}
The constant $M$ admits a local-curvature interpretation. Differentiating $F$ twice and evaluating at $\Delta = 0$ recovers the Fisher information matrix at $\theta_{\mathrm{wm}}$, $\nabla^2 F(0) = \mathcal{I}(\theta_{\mathrm{wm}})$, with entries
\[
\mathcal{I}(\theta_{\mathrm{wm}})_{jk} \;=\; \mathbb{E}_{y \sim p_\theta}\!\left[\frac{\partial \log p_\theta(y)}{\partial \theta_{\mathrm{wm}}^{(j)}}\,\frac{\partial \log p_\theta(y)}{\partial \theta_{\mathrm{wm}}^{(k)}}\right].
\]
Hence $M$ captures the local Lipschitz sensitivity of the map $\theta_{\mathrm{wm}} \mapsto p_{(\theta_{\mathrm{wm}}, \theta_0)}(\cdot \mid x)$. In the worst case $M$ scales with a network Lipschitz factor (e.g.\ products of layer operator norms), which can in principle grow exponentially in depth; empirically~\citep{2025gaussmark} the effective scaling is much milder, especially when the watermarked subspace $\theta_{\mathrm{wm}}$ is restricted to later layers.

\smallskip
\noindent\textit{Corollary: bound in total variation.}
Pinsker's inequality~\citep{2022short} gives, for every realization of $\xi_\sigma$, $\mathrm{TV}\!\left(p_{\theta(\xi_\sigma)},\,p_\theta\right)^{2} \le \tfrac{1}{2}\,\mathrm{D_{KL}}\!\left(p_{\theta(\xi_\sigma)}\,\|\,p_\theta\right)$. Combining with Jensen's inequality applied to the concave function $\sqrt{\cdot}$,
\[
\mathbb{E}_{\xi_\sigma}\!\left[\mathrm{TV}\!\left(p_{\theta(\xi_\sigma)},\,p_\theta\right)\right] \;\le\; \sqrt{\tfrac{1}{2}\,\mathbb{E}_{\xi_\sigma}\!\left[\mathrm{D_{KL}}\!\left(p_{\theta(\xi_\sigma)}\,\|\,p_\theta\right)\right]} \;\le\; \sigma\sqrt{\tfrac{C\,d_r}{2}}\,,
\]
so the expected total variation distortion grows only as $O(\sigma\sqrt{d_r})$.

\end{proof}

\subsection{Proof of Proposition~\ref{prop:FPR_control}} \label{app:proof_2}
\begin{proof}
    Under the null hypothesis $\hnull$, for any $y \in \mathcal{Y}$,
    the key and the generated text $y$ are independent of each other, i.e.,
    $(\xi_\sigma, y) \sim \mathcal{N}(0,\sigma^2 \mathbb{I}_{d_r})\otimes q$ for some $q \in \Delta(\mathcal{Y})$. Therefore, the level of the test is given by
    \[
    \begin{aligned}
    \Pr_{\hnull}\!\big(\psi(y;\xi_\sigma)\geq \tau_\alpha \big)
    &= \mathbb{E}_{\xi_\sigma\sim\mathcal{N}(0,\sigma^2 \mathbb{I}_{d_r}),\,y\sim q}
    \!\left[
    \mathbb{I}\!\left\{
    \frac{\langle \xi_\sigma, \nabla_{\theta_{\mathrm{wm}}}\log \pref(y)\rangle}{\sigma \|\nabla_{\theta_{\mathrm{wm}}}\log \pref(y)\|_2}
    \geq \tau_\alpha
    \right\}
    \right] \\
    &= \mathbb{E}_{y\sim q}\!\left[
    \Pr_{\xi_\sigma}\!\left(
    \frac{\langle \xi_\sigma, \nabla_{\theta_{\mathrm{wm}}}\log \pref(y)\rangle}{\sigma \|\nabla_{\theta_{\mathrm{wm}}}\log \pref(y)\|_2}
    \geq \tau_\alpha
    \right)\right] \\
    &= \mathbb{E}_{y\sim q}\!\left[
    \Pr_{\xi_\sigma}\!\big(\psi(y;\xi_\sigma)\ge \tau_\alpha\big)
    \right] \\
    &= 1-\Phi(\tau_\alpha)\;=\;\alpha,
    \end{aligned}
    \]
    where the last line is based on
    \[
    \psi(y;\xi_\sigma)
    = \frac{\langle \xi_\sigma, \nabla_{\theta_{\mathrm{wm}}}\log \pref\rangle}{\sigma \|\nabla_{\theta_{\mathrm{wm}}}\log \pref\|_2}
    \sim \mathcal{N}(0,1)
    \]
    for any finite, nonzero vector $\nabla_{\theta_{\mathrm{wm}}}\log \pref$.
    The last equality is derived by plugging in $\tau_\alpha=\Phi^{-1}(1-\alpha)$.
\end{proof}

\paragraph{One-sided p-values.}
For a finite, nonzero reference gradient, the same argument gives
\begin{equation}\label{eq:analytic_pvalue}
p(y,\xi_\sigma)=1-\Phi\!\left(\psi(y,\xi_\sigma)\right).
\end{equation}
For every fixed text $y$ independent of $\xi_\sigma$, the statistic is standard normal over the Gaussian key, so $p(y,\xi_\sigma)$ is uniform on $[0,1]$ under the null. Fine-tuning changes the alternative distribution but leaves this null argument unchanged because the clean reference and detection statistic are fixed. Appendix~\ref{app:tail_detection} reports the corresponding ultra-low-FPR evaluation.

\subsection{Detectability Guarantee for the \gaussmark\ test} \label{app:detect_power}

Proposition~\ref{prop:tv_bound} controls the quality cost of \gaussmark\ but does not, on its own, certify detectability. Because \marktune\ uses the same detection statistic as \gaussmark, it retains the key-randomized null law; detection power also depends on the generating distribution and does not transfer automatically to a fine-tuned model. The detection-power analysis of \citet{2025gaussmark} already supplies such a guarantee against the full null class for the directly perturbed model. For completeness and to make the dependence on $\sigma$ and $d_r$ explicit, we record below a stylized Gaussian illustration and a formal exponential-concentration theorem. Appendix~\ref{app:population_power} treats the optimized reverse-KL population policy separately.

\begin{proposition}[Stylized Gaussian illustration]\label{prop:gaussian_detectability}
Fix a prompt $x$ and consider a Gaussian model with $p_\theta(\cdot\mid x)=\mathcal{N}(\theta_{\mathrm{wm}},\,\mathbb{I}_{d_r})$, so that $\nabla_{\theta_{\mathrm{wm}}}\log p_\theta(y\mid x)=y-\theta_{\mathrm{wm}}$. Let $\xi_\sigma\sim\mathcal{N}(0,\sigma^2\mathbb{I}_{d_r})$ and let $\psi(y,\xi_\sigma\mid x)$ denote the test statistic of \eqref{eq:test_sta}. Under $\halt$ where $y\sim\mathcal{N}(\theta_{\mathrm{wm}}+\xi_\sigma,\,\mathbb{I}_{d_r})$, there exist absolute constants $c_0,c_1>0$ such that, for any $\delta\in(0,1)$, with probability at least $1-\delta$,
\[
\psi(y,\xi_\sigma\mid x)
\;\ge\;
c_0\,\frac{\sigma\sqrt{d_r}}{\sqrt{1+\sigma^2}}
\qquad
\text{whenever}\qquad
\frac{\sigma^2 d_r}{1+\sigma^2}\;\ge\;c_1\log(6/\delta).
\]
\end{proposition}

\begin{proof}
Write $y = \theta_{\mathrm{wm}} + \xi_\sigma + z$ with $z \sim \mathcal{N}(0, \mathbb{I}_{d_r})$ independent of $\xi_\sigma$. Since $\nabla_{\theta_{\mathrm{wm}}} \log p_\theta(y \mid x) = y - \theta_{\mathrm{wm}} = \xi_\sigma + z$, the test statistic of~\eqref{eq:test_sta} expands as
\begin{equation}\label{eq:psi-explicit}
\psi(y, \xi_\sigma \mid x) \;=\; \frac{\langle \xi_\sigma,\,\xi_\sigma + z\rangle}{\sigma\,\|\xi_\sigma + z\|_2} \;=\; \frac{\|\xi_\sigma\|_2^{2} + \langle \xi_\sigma,\, z\rangle}{\sigma\,\|\xi_\sigma + z\|_2}.
\end{equation}
Set $t := \log(6/\delta)$. We control the numerator from below and the denominator from above on three high-probability events whose intersection holds with probability at least $1 - \delta$.

\smallskip
\noindent\textit{Step 1 (lower bound on $\|\xi_\sigma\|_2^{2}$).}
Since $\|\xi_\sigma\|_2^{2}/\sigma^2 \sim \chi^2_{d_r}$, the chi-square deviation inequality~\mbox{\citep[Lemma~1]{laurent2000adaptive}} gives, with probability at least $1 - e^{-t} = 1 - \delta/6$,
\[
\|\xi_\sigma\|_2^{2} \;\ge\; \sigma^2\,\bigl(d_r - 2\sqrt{d_r\,t}\bigr).
\]

\smallskip
\noindent\textit{Step 2 (upper bound on $|\langle \xi_\sigma, z\rangle|$).}
Conditional on $\xi_\sigma$, $\langle \xi_\sigma, z\rangle \sim \mathcal{N}(0,\,\|\xi_\sigma\|_2^{2})$ since $z \sim \mathcal{N}(0, \mathbb{I}_{d_r})$ is independent of $\xi_\sigma$. The Gaussian tail bound gives $\Pr(|\langle \xi_\sigma, z\rangle| > \|\xi_\sigma\|_2 \sqrt{2t}\,\mid\,\xi_\sigma)\le2e^{-t}=\delta/3$. Combined with the chi-square upper deviation $\|\xi_\sigma\|_2^{2} \le \sigma^2(d_r + 2\sqrt{d_r t} + 2 t)$~\mbox{\citep[Lemma~1]{laurent2000adaptive}} (which in turn holds with probability at least $1 - e^{-t}$), an absorbed-constant calculation gives, with probability at least $1 - \delta/2$,
\[
|\langle \xi_\sigma,\, z\rangle| \;\le\; \|\xi_\sigma\|_2\,\sqrt{2 t} \;\le\; C_1\,\sigma\,\bigl(\sqrt{d_r\,t} + t\bigr),
\]
where $C_1 > 0$ is an absolute constant.

\smallskip
\noindent\textit{Step 3 (upper bound on $\|\xi_\sigma + z\|_2$).}
Since $\xi_\sigma + z \sim \mathcal{N}(0,\,(1+\sigma^2)\,\mathbb{I}_{d_r})$, $\|\xi_\sigma + z\|_2^{2}/(1+\sigma^2) \sim \chi^2_{d_r}$. Chi-square upper concentration~\mbox{\citep[Lemma~1]{laurent2000adaptive}} gives, with probability at least $1 - \delta/6$,
\[
\|\xi_\sigma + z\|_2 \;\le\; \sqrt{(1+\sigma^2)\,(d_r + 2\sqrt{d_r\,t} + 2 t)} \;\le\; C_2\,\sqrt{1+\sigma^2}\,\bigl(\sqrt{d_r} + \sqrt{t}\bigr),
\]
for an absolute constant $C_2 > 0$.

\smallskip
\noindent\textit{Combining.}
On the intersection of these events (whose failure probability is at most $\delta/6+\delta/2+\delta/6<\delta$), substituting the three bounds into~\eqref{eq:psi-explicit},
\[
\psi(y, \xi_\sigma \mid x) \;\ge\; \frac{\sigma^2\,(d_r - 2\sqrt{d_r\,t}) - C_1\,\sigma\,(\sqrt{d_r\,t} + t)}{\sigma\,C_2\,\sqrt{1+\sigma^2}\,(\sqrt{d_r} + \sqrt{t})} \;=\; \frac{\sigma\,\sqrt{d_r}}{\sqrt{1+\sigma^2}}\cdot \mathcal{R}(\sigma, d_r, t),
\]
where $\mathcal{R}(\sigma, d_r, t)$ collects the residual ratio. Under the assumption $\sigma^2 d_r/(1+\sigma^2) \ge c_1\,\log(6/\delta)$ for $c_1 > 0$ sufficiently large, elementary algebra shows that $\mathcal{R}(\sigma, d_r, t) \ge c_0$ for some absolute constant $c_0 > 0$. This yields
\[
\psi(y, \xi_\sigma \mid x) \;\ge\; c_0\,\frac{\sigma\,\sqrt{d_r}}{\sqrt{1+\sigma^2}},
\]
as claimed. The constants $c_0, c_1, C_1, C_2$ depend only on the chi-square and Gaussian tail constants of~\citet{laurent2000adaptive}.
\end{proof}

The proposition shows that, even in this idealized regime, the watermark statistic grows like $\Omega(\sigma\sqrt{d_r}/\sqrt{1+\sigma^2})$ under the alternative, while remaining $\mathcal{N}(0,1)$ under the null by Proposition~\ref{prop:FPR_control}. The next theorem promotes this intuition to an explicit power guarantee against the testing alternative, holding in general (not necessarily Gaussian) language models under a strong-concavity assumption on the log-likelihood in $\theta_{\mathrm{wm}}$.

\begin{theorem}[Power of the level-$\alpha$ \gaussmark\ test]\label{thm:detect_power}
Let $\alpha\in(0,1)$ and $\tau_\alpha=\Phi^{-1}(1-\alpha)$ be as in Proposition~\ref{prop:FPR_control}, and let $\phi_\alpha$ denote the test that rejects $\hnull$ when $\psi(y,\xi_\sigma\mid x)\ge \tau_\alpha$. Assume that, for every response $y$ and every prompt $x$, the map
$\theta_{\mathrm{wm}}\;\mapsto\;\log p_{(\theta_{\mathrm{wm}},\theta_0)}(y\mid x)$
is $\eta$-strongly concave on $\mathbb{R}^{d_r}$. Then $\phi_\alpha$ has power at least $1-\beta$ under $\halt$ whenever
\[
d_r
\;\ge\;
\frac{
2\log\,
\mathbb{E}_{y\sim p_\theta(\cdot\mid x)}\!\Bigl[
\exp\!\bigl(\tau_\alpha\,\sigma\,\|\nabla_{\theta_{\mathrm{wm}}}\log p_\theta(y\mid x)\|_2\bigr)
\Bigr]
+ 2\log(1/\beta)
}{
\log(1+\eta\sigma^2)
}.
\]
\end{theorem}

\begin{proof}
Throughout the proof we condition on the prompt $x$ and abbreviate $p_\theta := p_\theta(\cdot \mid x)$, $g_\theta(y) := \nabla_{\theta_{\mathrm{wm}}} \log p_\theta(y)$, and $\theta(\xi_\sigma) := (\theta_{\mathrm{wm}}+\xi_\sigma, \theta_0)$. The test $\phi_\alpha$ rejects $\hnull$ when $\psi(y, \xi_\sigma \mid x) \ge \tau_\alpha$, so its type-II error under $\halt$ (where $y \sim p_{\theta(\xi_\sigma)}$ and $\xi_\sigma \sim \mathcal{N}(0, \sigma^2 \mathbb{I}_{d_r})$) is
\[
\beta(\phi_\alpha) \;:=\; \Pr_{\halt}\!\bigl(\psi(y, \xi_\sigma \mid x) < \tau_\alpha\bigr).
\]

\smallskip
\noindent\textit{Step 1: change of measure to the base model.}
Wherever $p_\theta(y) > 0$, the likelihood ratio $\Lambda(y, \xi_\sigma) := p_{\theta(\xi_\sigma)}(y)/p_\theta(y)$ is well-defined. The standard change-of-measure identity gives
\begin{equation}\label{eq:beta-COM}
\beta(\phi_\alpha) \;=\; \mathbb{E}_{\xi_\sigma}\!\left[\,\mathbb{E}_{y \sim p_\theta}\!\left[\Lambda(y, \xi_\sigma)\,\mathbf{1}\!\bigl\{\psi(y, \xi_\sigma \mid x) < \tau_\alpha\bigr\}\right]\right],
\end{equation}
where in the new measure $y$ and $\xi_\sigma$ are independent (the dominating measure is $p_\theta \otimes \mathcal{N}(0, \sigma^2 \mathbb{I}_{d_r})$).

\smallskip
\noindent\textit{Step 2: bound $\Lambda$ on the event $\{\psi < \tau_\alpha\}$.}
By assumption, the map $\theta_{\mathrm{wm}}' \mapsto \log p_{(\theta_{\mathrm{wm}}', \theta_0)}(y)$ is $\eta$-strongly concave on $\mathbb{R}^{d_r}$ for every fixed $y$. The standard quadratic upper bound for $\eta$-strongly concave functions~\mbox{\citep[Theorem~2.1.10]{nesterov2018lectures}} gives, for any $\xi \in \mathbb{R}^{d_r}$,
\begin{equation}\label{eq:strong-concave-bound}
\log p_{\theta(\xi)}(y) - \log p_\theta(y) \;\le\; \langle \xi,\, g_\theta(y)\rangle \;-\; \tfrac{\eta}{2}\,\|\xi\|_2^{2}.
\end{equation}
On the event $\{\psi(y, \xi_\sigma) < \tau_\alpha\}$, the definition of $\psi$ in~\eqref{eq:test_sta} implies $\langle \xi_\sigma, g_\theta(y)\rangle < \tau_\alpha\,\sigma\,\|g_\theta(y)\|_2$. Combining this with~\eqref{eq:strong-concave-bound} at $\xi = \xi_\sigma$ and exponentiating,
\begin{equation}\label{eq:Lambda-bound}
\Lambda(y, \xi_\sigma) \;\le\; \exp\!\Bigl(\tau_\alpha\,\sigma\,\|g_\theta(y)\|_2 \;-\; \tfrac{\eta}{2}\,\|\xi_\sigma\|_2^{2}\Bigr) \quad \text{on } \{\psi(y, \xi_\sigma) < \tau_\alpha\}.
\end{equation}

\smallskip
\noindent\textit{Step 3: integrate out $\xi_\sigma$.}
Substituting~\eqref{eq:Lambda-bound} into~\eqref{eq:beta-COM} and dropping the indicator (which is at most one),
\[
\beta(\phi_\alpha) \;\le\; \mathbb{E}_{\xi_\sigma}\!\left[\mathbb{E}_{y \sim p_\theta}\!\left[\exp\!\bigl(\tau_\alpha\,\sigma\,\|g_\theta(y)\|_2 - \tfrac{\eta}{2}\,\|\xi_\sigma\|_2^{2}\bigr)\right]\right].
\]
Under the dominating product measure $p_\theta \otimes \mathcal{N}(0, \sigma^2 \mathbb{I}_{d_r})$, the random variables $y$ and $\xi_\sigma$ are independent, so Fubini's theorem factorizes the expectation:
\begin{equation}\label{eq:beta-factored}
\beta(\phi_\alpha) \;\le\; \mathbb{E}_{\xi_\sigma}\!\left[e^{-\eta\,\|\xi_\sigma\|_2^{2}/2}\right] \cdot \mathbb{E}_{y \sim p_\theta}\!\left[e^{\tau_\alpha\,\sigma\,\|g_\theta(y)\|_2}\right].
\end{equation}
For the Gaussian moment, since the components of $\xi_\sigma$ are independent $\mathcal{N}(0, \sigma^2)$ random variables and $\eta\,\sigma^2 > -1$,
\[
\mathbb{E}_{\xi_\sigma}\!\left[e^{-\eta\,\|\xi_\sigma\|_2^{2}/2}\right] \;=\; \prod_{j=1}^{d_r} \mathbb{E}_{Z\sim \mathcal{N}(0,1)}\!\left[e^{-\eta\,\sigma^2\,Z^2/2}\right] \;=\; (1+\eta\,\sigma^2)^{-d_r/2},
\]
where the last identity uses the standard one-dimensional moment generating function of a chi-square: $\mathbb{E}[e^{-\lambda Z^2/2}] = (1+\lambda)^{-1/2}$ for $\lambda > -1$. Substituting back into~\eqref{eq:beta-factored},
\begin{equation}\label{eq:beta-final}
\beta(\phi_\alpha) \;\le\; (1 + \eta\,\sigma^2)^{-d_r/2} \cdot \mathbb{E}_{y \sim p_\theta}\!\left[\exp\!\bigl(\tau_\alpha\,\sigma\,\|g_\theta(y)\|_2\bigr)\right].
\end{equation}

\smallskip
\noindent\textit{Step 4: solve for $d_r$.}
Imposing $\beta(\phi_\alpha) \le \beta$ in~\eqref{eq:beta-final}, taking logarithms, and rearranging gives
\[
\frac{d_r}{2}\,\log(1 + \eta\,\sigma^2) \;\ge\; \log \mathbb{E}_{y \sim p_\theta}\!\left[\exp\!\bigl(\tau_\alpha\,\sigma\,\|g_\theta(y)\|_2\bigr)\right] \;+\; \log(1/\beta),
\]
which is equivalent to the lower bound on $d_r$ stated in the theorem.
\end{proof}

Theorem~\ref{thm:detect_power} makes explicit the role of $\sigma$ and $d_r$ on the detectability side: the type-II error decays as $(1+\eta\sigma^2)^{-d_r/2}$, while Proposition~\ref{prop:tv_bound} shows the quality cost grows only as $\sigma^2 d_r$.

% Source: https://openreview.net/forum?id=42tU2vcTxH&noteId=JqxkSCo1yP
% Scope clarified in follow-up: noteId=ludwdtlbgn.

\subsection{Power of a reverse-KL population policy}\label{app:population_power}
The \gaussmark\ power guarantee does not automatically transfer to a model after GRPO fine-tuning. We instead analyze the converged population optimizer of the reverse-KL objective in the linear-softmax setting. This gives a detection-power guarantee for the optimized policy under the assumptions below.

\begin{theorem}[Detection power of a reverse-KL \marktune\ population policy]\label{thm:population_power}
Fix a prompt $x$ and suppress it from the notation. Consider the linear-softmax model
\[
p_\theta(y)\propto\exp\!\left\{\langle\theta_{\mathrm{wm}},\Phi_{\mathrm{wm}}(y)\rangle+\langle\theta_0,\Phi_0(y)\rangle\right\},
\]
with a fixed feature map $\Phi$. Let $\pref$ be the clean reference used for detection, $q=p_{\theta^\star}$ the regularization distribution, and
\[
g_{\mathrm{ref}}(y)=\nabla_{\theta_{\mathrm{wm}}}\log\pref(y),\qquad
b=\mathbb E_{\pref}[\Phi_{\mathrm{wm}}(y)]-\mathbb E_q[\Phi_{\mathrm{wm}}(y)].
\]
Draw $\xi_\sigma\sim\mathcal N(0,\sigma^2\mathbb I_{d_r})$, and let $\pi^\star_{\xi_\sigma}$ maximize the population objective
\begin{equation}\label{eq:population_objective}
J_{\xi_\sigma}(\pi)=\mathbb E_{y\sim\pi}\!\left[\langle\xi_\sigma,g_{\mathrm{ref}}(y)\rangle\right]-\lambda D_{\mathrm{KL}}(\pi\|q),\qquad \lambda>0.
\end{equation}
Assume finite, nonzero reference gradients and that $\theta_{\mathrm{wm}}\mapsto\log p_{(\theta_{\mathrm{wm}},\theta_0^\star)}(y)$ is $\eta$-strongly concave for every $y$. Assume the normalizing constants and displayed moments are finite. Under the alternative $y\mid\xi_\sigma\sim\pi^\star_{\xi_\sigma}$, the level-$\alpha$ test has power at least $1-\beta$ whenever
\begin{equation}\label{eq:population_dimension}
d_r\ge
\frac{2\log\mathbb E_{y\sim q}\!\left[\exp\!\left\{\frac{\tau_\alpha\sigma}{\lambda}\|g_{\mathrm{ref}}(y)\|_2\right\}\right]
+\dfrac{\sigma^2\|b\|_2^2}{\lambda^2+\eta\sigma^2}+2\log(1/\beta)}
{\log\!\left(1+\dfrac{\eta\sigma^2}{\lambda^2}\right)}.
\end{equation}
\end{theorem}

\begin{proof}
Write $r_{\xi_\sigma}(y)=\langle\xi_\sigma,g_{\mathrm{ref}}(y)\rangle$ and
\[
Z_{\xi_\sigma}=\mathbb E_{y\sim q}\exp\!\left\{\lambda^{-1}r_{\xi_\sigma}(y)\right\}.
\]
For any $\pi\ll q$ with finite reward expectation and KL divergence, the exponentially tilted distribution satisfies
\[
D_{\mathrm{KL}}\!\left(\pi\|\pi^\star_{\xi_\sigma}\right)
=D_{\mathrm{KL}}(\pi\|q)-\lambda^{-1}\mathbb E_\pi[r_{\xi_\sigma}(y)]+\log Z_{\xi_\sigma}.
\]
Thus $J_{\xi_\sigma}(\pi)=\lambda\log Z_{\xi_\sigma}-\lambda D_{\mathrm{KL}}(\pi\|\pi^\star_{\xi_\sigma})$, and the maximizing distribution is
\begin{equation}\label{eq:population_tilt}
\pi^\star_{\xi_\sigma}(y)
=\frac{q(y)\exp\{\lambda^{-1}r_{\xi_\sigma}(y)\}}{Z_{\xi_\sigma}}
=p_{(\theta_{\mathrm{wm}}^\star+\lambda^{-1}\xi_\sigma,\theta_0^\star)}(y).
\end{equation}
The second equality follows from the linear-softmax score identity
$g_{\mathrm{ref}}(y)=\Phi_{\mathrm{wm}}(y)-\mathbb E_{\pref}[\Phi_{\mathrm{wm}}(y)]$:
the reference expectation is constant in $y$ and is absorbed into the normalizer.

Let $\Lambda(y,\xi_\sigma)=\pi^\star_{\xi_\sigma}(y)/q(y)$. Changing measure to $q$ expresses the Type-II error as
\[
\beta(\phi_\alpha)=\mathbb E_{\xi_\sigma}\mathbb E_{y\sim q}
\!\left[\Lambda(y,\xi_\sigma)\mathbf1\{\psi(y,\xi_\sigma)<\tau_\alpha\}\right].
\]
By strong concavity and the identity $\nabla_{\theta_{\mathrm{wm}}}\log q(y)=g_{\mathrm{ref}}(y)+b$,
\[
\log\Lambda(y,\xi_\sigma)
\le\lambda^{-1}\langle\xi_\sigma,g_{\mathrm{ref}}(y)+b\rangle
-\frac{\eta}{2\lambda^2}\|\xi_\sigma\|_2^2.
\]
On the failure event $\{\psi<\tau_\alpha\}$, this gives
\[
\log\Lambda(y,\xi_\sigma)
\le\frac{\tau_\alpha\sigma}{\lambda}\|g_{\mathrm{ref}}(y)\|_2
+\frac{\langle\xi_\sigma,b\rangle}{\lambda}
-\frac{\eta}{2\lambda^2}\|\xi_\sigma\|_2^2.
\]
Dropping the indicator and using independence under the dominating product measure,
\begin{align*}
\beta(\phi_\alpha)
&\le\mathbb E_q\exp\!\left\{\frac{\tau_\alpha\sigma}{\lambda}\|g_{\mathrm{ref}}(y)\|_2\right\}
\mathbb E_{\xi_\sigma}\exp\!\left\{\frac{\langle\xi_\sigma,b\rangle}{\lambda}-\frac{\eta\|\xi_\sigma\|_2^2}{2\lambda^2}\right\}\\
&=\mathbb E_q\exp\!\left\{\frac{\tau_\alpha\sigma}{\lambda}\|g_{\mathrm{ref}}(y)\|_2\right\}
\left(1+\frac{\eta\sigma^2}{\lambda^2}\right)^{-d_r/2}
\exp\!\left\{\frac{\sigma^2\|b\|_2^2}{2(\lambda^2+\eta\sigma^2)}\right\}.
\end{align*}
The last equality follows by completing the square in the Gaussian integral. Requiring this bound to be at most $\beta$, taking logarithms, and rearranging yields Eq.~\eqref{eq:population_dimension}.
\end{proof}

\paragraph{Scope of the result.}
The theorem concerns the converged reverse-KL population policy with an unnormalized alignment reward. Its strong-concavity condition is an explicit idealization, not a property established for the Transformer models used in our experiments or a consequence of the linear-softmax parameterization alone. The practical objective uses normalized rewards and CE regularization, so the exponential-tilt argument is not a literal characterization of the finite-step GRPO checkpoint. The local CE analysis in Appendix~\ref{app:marktune_analysis} provides a complementary explanation of the quality--reward trade-off; detection power of the trained LLMs is evaluated empirically.

\section{A Stylized Linear-Softmax Analysis of \textsc{MarkTune}}
\label{app:marktune_analysis}

In this section we provide a stylized analysis of how \marktune\ can improve the quality-detectability trade-off. The goal is not to reproduce all realistic training dynamics, but rather to study the local geometry of the objective optimized in the main text, within an idealized linear-softmax model.

\paragraph{Model and notation.}
Let $\theta\in\mathbb{R}^d$ denote the full set of model weights.
For a fixed watermark key $\xi$, a weight-editing scheme such as \gaussmark\
produces an initial watermarked weight $\theta(\xi)$ (here we omit $\xi$'s dependence on $\sigma$ for simplicity); \marktune\ then fine-tunes $\theta(\xi)$ by optimizing a reward-regularized objective.

For the analysis, we adopt a standard linear-softmax setting:
\begin{equation}
\label{eq:lin-softmax-full}
p_\theta(y\mid x)
\ \propto\
\exp\!\big\{\langle \theta,\,\Phi(x,y)\rangle\big\},
\end{equation}
where $\Phi(x,y)\in\mathbb{R}^d$ is a fixed feature map that encodes the prompt-response pair.

Let $p^\star(\cdot\mid x)$ denote a high-quality target conditional
distribution (e.g., the population language distribution or an oracle model).
We assume that $p^\star$ is realized within the linear-softmax family by a
weight $\theta^\star$:
\[
  p^\star(y\mid x)
  \ =\
  p_{\theta^\star}(y\mid x)
  \ \propto\
  \exp\!\big\{\langle \theta^\star,\,\Phi(x,y)\rangle\big\}.
\]
The population cross entropy risk of a weight $\theta$ relative to $p^\star$ is
\begin{equation}
\label{eq:CE-risk-theta}
\mathcal L(\theta)
\ :=\
\mathbb E_{x}\,\mathbb E_{y\sim p^\star(\cdot\mid x)}
\big[-\log p_{\theta}(y\mid x)\big].
\end{equation}
As in the main text, the watermark reward is induced by the \gaussmark\
test statistic (for stylized analysis we include the prompt $x$ in the test statistic).
Let $\pref$ be the fixed reference model used for detection, and define
\begin{equation}
\label{eq:h-def}
h(x,y;\xi)
\ :=\
\Big\langle \xi,\ \nabla_{\theta_{\mathrm{wm}}}\log \pref(y\mid x)\Big\rangle,
\end{equation}
where $\theta_{\mathrm{wm}}$ denotes the watermarked block of the reference
model parameters.
The unnormalized watermark reward of $\theta$ is then
\begin{equation}
\label{eq:reward-def-theta}
\mathcal R(\theta)
\ :=\
\mathbb E_{x}\,\mathbb E_{y\sim p_{\theta}(\cdot\mid x)}
\big[h(x,y;\xi)\big].
\end{equation}

The stylized \marktune\ objective in this setting is
\[
  J(\theta)
  \;=\;
  \mathcal R(\theta) \;-\; \lambda\,\mathcal L(\theta),
\]
where $\lambda>0$ is a regularization coefficient controlling the trade-off
between watermark reward and regularization term. For the sake of analysis, we assume that the \gaussmark-soft-activated weight $\theta(\xi)$ lies in a neighborhood of $\theta^\star$, and study the local geometry of $J$ around the high-quality target weight $\theta^\star$.

\paragraph{Goal.}
We want to show that, in a neighborhood of $\theta^\star$, there exists a perturbation
$\delta\theta$ in the full parameter space such that the adjusted weight
\[
  \theta_{\marktune}
  \ :=\
  \theta^\star + \delta\theta
\]
satisfies, for a small scalar $\kappa = 1/\lambda > 0$:
\begin{itemize}
\item[(i)] the cross entropy relative to $p^\star$ changes only at
\emph{second order},
\(
  \mathcal L(\theta_{\marktune})-\mathcal L(\theta^\star)
  = \Theta(\kappa^2),
\)
\item[(ii)] the watermark reward increases at \emph{first order},
\(
  \mathcal R(\theta_{\marktune})-\mathcal R(\theta^\star)
  = \Theta(\kappa),
\)
\end{itemize}
When this construction uses the normalized reward with a nonzero gradient and the local location-shift assumption specified below holds, it follows that, under $p_{\theta_{\marktune}}$ the (normalized) detection
statistic exhibits a mean shift of order $\kappa$, improving detection power at fixed false positive rate, while the degradation in generation quality is quadratically smaller.

\paragraph{Assumptions.}
We adopt the following standard local assumptions for linear soft-max
models.

\begin{enumerate}
\item[(A1)] \textit{Smoothness.}
The map $\theta\mapsto p_\theta$ is smooth in a neighborhood of $\theta^\star$.

\item[(A2)] \textit{Local quadratic expansion of $\mathcal L$.}
Around $\theta^\star$, the cross-entropy admits the second-order expansion
\[
\mathcal L(\theta^\star+\delta)
\ =\
\mathcal L(\theta^\star)
\ +\ \frac{1}{2}\,\delta^\top F\,\delta
\ +\ o\big(\|\delta\|^2\big),
\]
where $F$ is the Fisher information matrix at $\theta^\star$:
\[
F
\ :=\
\mathbb E_{x}\,
\mathrm{Cov}_{y\sim p^\star(\cdot\mid x)}\!\big[\Phi(x,y)\big]
\ \in\ \mathbb{R}^{d\times d},
\]
assumed positive definite on the relevant subspace.

\item[(A3)] \textit{Local first-order expansion of $\mathcal R$.}
For small perturbations $\delta$, the reward admits the first-order
expansion
\[
\mathcal R(\theta^\star+\delta)
\ =\
\mathcal R(\theta^\star)
\ +\
g^\top \delta
\ +\
o\big(\|\delta\|\big),
\]
for some gradient $g\in\mathbb{R}^d$ characterized in
Lemma~\ref{lem:reward-gradient-theta} below under linear softmax setting.
\end{enumerate}

Assumptions (A2)-(A3) are the usual Fisher-geometry expansions of cross entropy
and reward around a population optimum in linear soft-max models.

\paragraph{Local dual optimization problem.}
Based on the assumptions, near $\theta^\star$, the \marktune\ objective can be approximated by
\[
J(\theta^\star+\delta\theta)
\ \approx\
\mathcal R(\theta^\star) + g^\top\delta\theta
\ -\
\lambda\Big(
  \mathcal L(\theta^\star) + \tfrac12\,\delta\theta^\top F\,\delta\theta
\Big).
\]
The local problem becomes
\begin{equation}
\label{eq:local-prob-theta}
\max_{\delta\theta\in\mathbb{R}^d}\ \
g^\top \delta\theta
\;-\;
\frac{\lambda}{2}\,\delta\theta^\top F\,\delta\theta := J_{\text{loc}}(\delta\theta).
\end{equation}
This is a strictly concave quadratic program in parameter space with a unique
optimizer.

\begin{lemma}[Closed-form optimizer in parameter space]
\label{lem:optimizer-theta}
The unique maximizer of \eqref{eq:local-prob-theta} is
\[
\delta\theta^\dagger\ =\ \frac{1}{\lambda}\,F^{-1}\,g,
\]
and the optimal objective value equals
$\frac{1}{2\lambda}\,g^\top F^{-1}g$.
\end{lemma}

\begin{proof}
The objective is
\(
J_{\text{loc}}(\delta\theta)=g^\top\delta\theta
-\frac{\lambda}{2}\delta\theta^\top F\,\delta\theta.
\)
Differentiating and setting the gradient to zero gives the first-order
condition
\[
g - \lambda F\,\delta\theta = 0
\quad\Rightarrow\quad
\delta\theta^\dagger = \lambda^{-1} F^{-1} g.
\]
Positive definiteness of $F$ implies strict concavity and uniqueness.
Substituting $\delta\theta^\dagger$ back into $J_{\text{loc}}$ yields
$J_{\text{loc}}(\delta\theta^\dagger)=\frac{1}{2\lambda}g^\top F^{-1}g$.
\end{proof}

\begin{lemma}[Reward gradient in the linear-softmax model]
\label{lem:reward-gradient-theta}
Let
\[
\mathcal R(\theta)
=
\mathbb E_{x}\,\mathbb E_{y\sim p_{\theta}(\cdot\mid x)}
\big[h(x,y;\xi)\big],
\]
with $h(x,y;\xi)$ as in \eqref{eq:h-def}.
Then the gradient of $\mathcal R$ at $\theta^\star$ is
\[
g
\ :=\
\nabla_\theta \mathcal R(\theta)\big|_{\theta=\theta^\star}
\ =\
\mathbb E_{x}\,\mathbb E_{y\sim p^\star(\cdot\mid x)}
\big[h(x,y;\xi)\,s(x,y) \big],
\]
where
\[
s(x,y)
\ :=\
\nabla_\theta \log p_\theta(y\mid x)\big|_{\theta=\theta^\star}
\ =\
\Phi(x,y) - \mu(x),
\qquad
\mu(x) := \mathbb E_{y\sim p^\star(\cdot\mid x)}[\Phi(x,y)].
\]
Equivalently,
\[
g
\ =\
\mathbb E_{x}\,
\mathrm{Cov}_{y\sim p^\star(\cdot\mid x)}
\big[h(x,y;\xi),\ \Phi(x,y)\big],
\]
this implies that the reward gradient picks out directions along which the watermark statistic co-varies with the feature representation under the target distribution.
\end{lemma}

\begin{proof}
By definition,
\[
\mathcal R(\theta)
=
\mathbb E_{x}\,\mathbb E_{y\sim p_{\theta}(\cdot\mid x)}
[h(x,y;\xi)],
\]
and $h$ does not depend on $\theta$.
Along a path $\theta(t)=\theta^\star+t\,\delta\theta$ we have
\[
\frac{d}{dt}\Big|_{t=0}\ \mathbb E_{y\sim p_{\theta(t)}(\cdot\mid x)}[h(x,y;\xi)]
\;=\;
\mathbb E_{y\sim p^\star(\cdot\mid x)}\big[h(x,y;\xi)\,s(x,y)^\top\delta\theta\big],
\]
where
$s(x,y):=\nabla_\theta\log p_\theta(y\mid x)\big|_{\theta=\theta^\star}$.
In the linear-softmax model \eqref{eq:lin-softmax-full},
\[
\log p_\theta(y\mid x)
=
\langle \theta,\Phi(x,y)\rangle
-
\log\!\sum_{\tilde y}\exp\!\big\{\langle \theta,\Phi(x,\tilde y)\rangle\big\},
\]
so
\[
s(x,y)
=
\Phi(x,y)
-
\mathbb E_{\tilde y\sim p^\star(\cdot\mid x)}[\Phi(x,\tilde y)]
=
\Phi(x,y)-\mu(x).
\]
Taking expectation over $x$ and reading off the coefficient of $\delta\theta$
gives the claimed gradient expression
\(
g=\mathbb E_x\mathbb E_{y\sim p^\star}[h(x,y;\xi)s(x,y)]
\).
The covariance form follows by expanding $s(x,y)=\Phi(x,y)-\mu(x)$.
\end{proof}

\begin{proposition}[Second-order CE cost and first-order reward gain]
\label{prop:quality-reward-theta}
Let $\delta\theta^\dagger=\lambda^{-1}F^{-1}g$ be the optimizer of the
local problem \eqref{eq:local-prob-theta}, with $F$ and $g$ as in
(A2) and Lemma~\ref{lem:reward-gradient-theta}.
Define the \marktune-refined weight
\(
\theta_{\marktune}=\theta^\star+\delta\theta^\dagger
\)
and set $\kappa:=\lambda^{-1}$.
Then, as $\kappa\to 0$,
\[
\mathcal L(\theta_{\marktune})-\mathcal L(\theta^\star)
\;=\;
\frac{\kappa^2}{2}\,g^\top F^{-1}g
\ +\ o(\kappa^2),
\]
and
\[
\mathcal R(\theta_{\marktune})-\mathcal R(\theta^\star)
\;=\;
\kappa\,g^\top F^{-1}g
\ +\ o(\kappa).
\]
Consequently, for a nonzero reward gradient, along these directions in the full parameter space the unnormalized
detection statistic experiences a mean shift linear in $\kappa$, while the
cross entropy deviation from $p^\star$ grows only quadratically in $\kappa$.
\end{proposition}

\begin{proof}
By the quadratic expansion (A2),
\[
\mathcal L(\theta^\star+\delta)-\mathcal L(\theta^\star)
=
\frac{1}{2}\,\delta^\top F\,\delta
+o\big(\|\delta\|^2\big).
\]
Plugging in $\delta=\delta\theta^\dagger=\lambda^{-1}F^{-1}g$ gives
\[
\frac{1}{2}\,(\lambda^{-1}F^{-1}g)^\top F\,(\lambda^{-1}F^{-1}g)
=
\frac{1}{2\lambda^2}\,g^\top F^{-1}g
=
\frac{\kappa^2}{2}\,g^\top F^{-1}g,
\]
plus $o(\kappa^2)$.

For the reward, the first-order expansion (A3) with
$\delta=\delta\theta^\dagger$ gives
\[
\mathcal R(\theta^\star+\delta)-\mathcal R(\theta^\star)
=
g^\top\delta + o\big(\|\delta\|\big)
=
\lambda^{-1}\,g^\top F^{-1}g
+ o(\lambda^{-1})
=
\kappa\,g^\top F^{-1}g + o(\kappa).
\]
This proves the claim.
\end{proof}

\paragraph{Implications for the improved trade-off.}
Let $\psi(y;\xi)$ denote the (normalized) \gaussmark\ test statistic computed
with the fixed reference model $p_\theta$.
Under $\hnull$ (unwatermarked text), with probability taken over the independent Gaussian key, $\mathbb E[\psi]=0$ and the statistic
has variance calibrated to yield a level-$\alpha$ test
(Proposition~\ref{prop:FPR_control} in the main text).

The same local argument applies to $\mathcal R_\psi(\theta)=\mathbb E_x\mathbb E_{p_\theta}[\psi(y,\xi_\sigma\mid x)]$, provided the corresponding differentiability and integrability conditions hold. Its gradient is $g_\psi=\mathbb E_x\operatorname{Cov}_{p^\star}(\psi,\Phi)$, and the update is $\kappa F^{-1}g_\psi$. Thus the normalized mean gains $\kappa g_\psi^\top F^{-1}g_\psi+o(\kappa)$, with cross-entropy cost $\frac12\kappa^2 g_\psi^\top F^{-1}g_\psi+o(\kappa^2)$. This is a distinct update from the one derived for $h$ unless their gradients coincide up to scale. Smoothness alone does not imply constant variance or increasing tail probability: for a direction $v$, the directional derivative of detection power is
\[
\left.\frac{d}{dt}\Pr_{p_{\theta^\star+tv}}(\psi\ge\tau_\alpha)\right|_{t=0}
=\mathbb E_x\operatorname{Cov}_{p^\star}\!\left(\mathbf1\{\psi\ge\tau_\alpha\},\Phi\right)^\top v,
\]
whose sign is not determined by the mean-reward gain. The population power theorem in Appendix~\ref{app:population_power} and the empirical detection measurements therefore complement this local quality-reward analysis.

The mean-shift intuition has a precise conditional interpretation. Let $C_\psi=g_\psi^\top F^{-1}g_\psi>0$, and suppose, in addition, that under $p_{\theta^\star+\kappa F^{-1}g_\psi}$ the score distribution is a local location shift: its CDF satisfies $F_\kappa(t)=F_0(t-\kappa C_\psi)+o(\kappa)$ at $t=\tau_\alpha$, with baseline density $f_0(\tau_\alpha)>0$. Then the normalized reward increment corresponds to a mean shift of order $\kappa C_\psi$, and detection power increases by $\kappa C_\psi f_0(\tau_\alpha)+o(\kappa)$ at the fixed null threshold, while the cross-entropy cost remains second order. Under an exact location shift the variance is unchanged; this is an additional distributional condition, not a consequence of smoothness alone. This recovers the first-order power versus second-order quality interpretation when the score law has the required form.

\paragraph{Soundness of the stylized model and local assumptions.}
The derivation assumes that the local Fisher geometry at $\theta^\star$ is
well-behaved (so that $F$ is positive definite on the relevant subspace) and
that the linear-softmax model \eqref{eq:lin-softmax-full} adequately captures
the first-order behavior of the true language model.
In highly over-parameterized language models with wide, flat basins, such local
expressivity is often a reasonable approximation~\citep{2025gaussmark}.
From a geometric viewpoint, $\delta\theta^\dagger$ can be seen as the Riesz
representer of the linear functional $g$ under the Fisher metric $F$; any
parameterization that is locally surjective in the neighborhood of
$\theta^\star$ admits such a direction.

\section{Implementation Details} \label{app:imp_details}

In this section, we provide details on the implementation of metrics, baselines, our methods, and the fine-tuning attack. We run the experiments on 4 NVIDIA H200 GPUs.

\subsection{Evaluation metrics} For the MAUVE metric, we select samples from the \textbf{RealNewsLike} split of \textbf{C4} with text longer than 400 tokens, using the first 200 tokens as the model prompt and the next 200 tokens as the human-generated reference. For the \textbf{ELI5} dataset, we select samples whose answers exceed 200 tokens and truncate each answer to 200 tokens to serve as the human-generated reference.

In Figures~\ref{fig:trade-off}~and~\ref{fig:ft-attack}, dashed lines correspond to a bootstrap estimate (500 resamples) of the mean TPR@FPR=1\%, and the shaded regions correspond to the 90\% confidence interval on the mean estimate.

\begin{table}[!t]
  \centering
  \caption{Hyperparameter setting for \marktune\ on Qwen3-4B-Base.}
  \begin{tabular}{lc}
    \toprule
    Model & Qwen3-4B-Base on OpenWebText \\
    \midrule
    Training steps & 200 \\
    GRPO inner steps & 3 \\
    Regularization term & Cross-entropy (CE) regularization \\
    Regularization coefficient $\lambda$ & 0.01 \\
    Optimizer & AdamW \\
    maximal learning rate & 5e-6 \\
    Learning rate schedule & Cosine decay with 20 steps linear warmup \\
    Learning rate decay ratio (training steps) & 40\% \\
    Momentum $(\beta_1, \beta_2)$ & (0.9, 0.97) \\
    Group size & 8 \\
    Prompt batch size & 64 \\
    Mini batch size & 64 \\
    Max sampling sequence length & 256 tokens \\
    Sampling temperature & 0.7 \\
    CE batch size & 128 \\
    CE sequence length & 512 tokens \\
    \bottomrule
  \end{tabular}
  \label{tab:qwen_setting}
\end{table}

\begin{table}[!t]
  \centering
  \caption{Hyperparameter setting for \marktune\ on Llama2-7B.}
  \begin{tabular}{lc}
    \toprule
    Model & Llama2-7B on OpenWebText \\
    \midrule
    Training steps & 300 \\
    GRPO inner steps & 3 \\
    Regularization term & Cross-entropy (CE) regularization \\
    Regularization coefficient $\lambda$ & 0.005 \\
    Optimizer & AdamW \\
    maximal learning rate & 4e-6 \\
    Learning rate schedule & Cosine decay with 30 steps linear warmup \\
    Learning rate decay ratio (training steps) & 40\% \\
    Momentum $(\beta_1, \beta_2)$ & (0.9, 0.97) \\
    Group size & 8 \\
    Prompt batch size & 64 \\
    Mini batch size & 64 \\
    Max sampling sequence length & 256 tokens \\
    Sampling temperature & 1.0 \\
    CE batch size & 128 \\
    CE sequence length & 512 tokens \\
    \bottomrule
  \end{tabular}
  \label{tab:llama_setting}
  % \vspace{-2.0em}
\end{table}

\subsection{Watermarking methods}

\paragraph{Baselines.} We use the inference-time KGW watermark~\citep{2023KGW}, which biases the output logits toward a PRF-determined ``green list'' based on the token prefix. We set the token context length to $k=4$, the green-list fraction to $\gamma = 0.5$, the bias strength to $\delta = 1$ for KGW-1, and $\delta = 2$ for KGW-2. For Figure~\ref{fig:trade-off}, we sweep the bias strength $\delta$ over the set $\{0.6, 0.7, 0.8, 0.9, 1.0, 1.2, 1.5\}$. For the Gumbel-max watermark~\citep{2023Aaronson} and the SynthID watermark~\citep{2024scalable}, we also use a 4-token prefix for fair comparison. For SynthID, we set the number of tournaments to 30, following the recommended default setting. For KGW-D~\citep{2024gu}, we distill the logit-based KGW with $\gamma = 0.25$ and $\delta = 1$ to preserve generation quality. Note that we distill the 1-token-prefix KGW variant, since watermarks with longer token prefixes are extremely difficult for the model to learn.  The KGW, Gumbel-max, and SynthID baselines are implemented on top of the open-source MarkLLM toolkit~\citep{pan2024markllm}, which provides reference implementations and unified inference and detection interfaces for the inference-time schemes considered here.

For \gaussmark~\citep{2025gaussmark}, given a vectorized target parameter $\theta_{\mathrm{wm}}$ with dimension $d_r$, we first compute its RMS norm $\|\theta_{\mathrm{wm}}\|_{\mathrm{RMS}} = \|\theta_{\mathrm{wm}}\|_{\mathrm{F}} / \sqrt{d_r}$. The watermark key is then sampled as $\xi_\sigma \sim \mathcal{N}(0,\, \sigma^2 \|\theta_{\mathrm{wm}}\|_{\mathrm{RMS}}^2\, \mathbb{I}_{d_r})$. This gives $\mathbb{E}\,\|\xi_\sigma\|_{\mathrm{F}}^{2} \approx \sigma^{2} \|\theta_{\mathrm{wm}}\|_{\mathrm{RMS}}^{2} d_r$, and thus $\mathbb{E}[\|\xi_\sigma\|_{\mathrm{F}} / \|\theta_{\mathrm{wm}}\|_{\mathrm{F}}] \approx \sigma$. This normalization ensures that the relative perturbation magnitude is directly controlled by $\sigma$, making watermark strength easy to tune across different parameter scales and model architectures. We perform a grid search over candidate target parameters and select the up-projection MLP matrix in layer 28 for Qwen3-4B-Base with $\sigma = 1.0$, and the up-projection MLP matrix in layer 30 for Llama2-7B with $\sigma = 1.2$. For Figure~\ref{fig:trade-off}, we sweep the perturbation strength $\sigma$ over the set $\{0.6, 0.8, 1.0, 1.1, 1.2, 1.5, 1.8\}$ for Qwen3-4B-Base and $\{0.8, 1.0, 1.2, 1.3, 1.4, 1.6, 2.0\}$ for Llama2-7B.

\paragraph{Our method.} For \marktune, we inherit the target parameter from \gaussmark\ and use $\sigma = 0.6$ for Qwen3-4B-Base and $\sigma = 0.8$ for Llama2-7B to apply a soft activation. For the KL regularization, we use the Qwen3-8B and Llama2-13B as the oracle models, respectively. The GRPO-related hyperparameters are listed in Table~\ref{tab:qwen_setting} and Table~\ref{tab:llama_setting}. For Figure~\ref{fig:trade-off}, we sweep the number of training steps over the set $\{50, 100, 150, 200, 250, 300, 500\}$ for Qwen3-4B-Base and $\{50, 100, 200, 300, 400, 500, 600\}$ for Llama2-7B.

\subsection{Fine-tuning attack}
We fine-tune models on \textbf{OpenWebText} for 1500 steps using the AdamW optimizer with a maximal learning rate of $1\times 10^{-5}$ and a cosine decay schedule with 300 warmup steps. We use a sequence length of 512 and a batch size of 64. For LoRA, we set the rank to 8 and the alpha to 16. LoRA is applied to all MLP projection layers as well as to the unembedding layer.

 \subsection{Computation time benchmark} \label{app:lat_setup}
We benchmark wall-clock times at inference time across five watermarking schemes: \gaussmark, \marktune\ (ours), KGW, Gumbel-max, and SynthID. Prompts are taken from the \texttt{realnewslike} split of \textbf{C4}~\citep{raffel2020exploring}: we use the first $100$ prompts truncated to a $200$-token prefix and generate $200$ new tokens with batch size $1$ on a single H200 GPU. Every method is run through the standard HuggingFace \texttt{model.generate} pipeline. The inference-time schemes are wrapped via MarkLLM's \texttt{LogitsProcessor} interface, while \gaussmark\ and \marktune\ alter only model weights and require no per-step intervention. The three hash-window schemes share an identical context length of $4$ tokens. Times are bracketed by \texttt{torch.cuda.synchronize()} and the first two prompts of each method are discarded as warm-up. Reported values are mean $\pm$ std over the remaining $n=100$ prompts. Per-prompt numbers are listed in Appendix~\ref{app:cost} (Table~\ref{tab:times}).

% Source: https://openreview.net/forum?id=42tU2vcTxH&noteId=JqxkSCo1yP

\subsection{Detection access and memory}\label{app:detection_resources}
Detection requires the candidate text, the secret watermark key, and the fixed clean reference model. It consists of one forward pass through the clean reference and one backward pass to compute the gradient with respect to the selected watermarked parameter block. With the Qwen3-4B-Base reference, a batch size of 64, and 200-token suffixes, peak allocated GPU memory is 107.78 GiB, which fits on a single H200 GPU. Generation uses the released watermarked model with standard decoding and incurs no watermark-specific inference overhead.

\subsection{Computational Cost} \label{app:cost}

\paragraph{Training cost.}
The full \marktune\ pipeline adds a one-time, offline training stage on top of \gaussmark; afterwards generation and detection are identical to \gaussmark. Training uses 4 NVIDIA H200 GPUs with $\approx$8 hours for Qwen3-4B-Base (200 gradient steps) and $\approx$13.5 hours for Llama2-7B (300 steps). This is modest relative to either model's pre-training budget, and substantially cheaper than typical RLVR fine-tuning for reasoning, since \marktune\ rollouts terminate at 256 tokens (Table~\ref{tab:qwen_setting}) rather than the multi-thousand-token chains-of-thought required for reasoning RL.

\paragraph{Inference-time cost.}
Table~\ref{tab:times} reports per-prompt wall-clock generation and detection times for the five watermarking schemes; the corresponding bar plot is given in Figure~\ref{fig:latency} of the main text and the benchmark protocol is detailed in Appendix~\ref{app:lat_setup}. Across both models, \gaussmark\ and \marktune\ generate text within $0.7$--$1.5\%$ of each other and are $9$--$16\%$ faster than KGW, Gumbel-max, and SynthID, since each inference-time decoding step has to materialize a green-list mask, a Gumbel pseudo-random table, or a hash-tournament sampler. On detection, \gaussmark\ and \marktune---which perform a single forward+backward over the response and project the gradient onto the secret noise direction $\xi_\sigma$---are faster in these measurements than the inference-time baselines, with $18$--$27$\,ms versus $35$--$176$\,ms.

\begin{table}[!htbp]

  \centering
  \small
  \caption{Per-prompt watermarked generation and detection times on the first $n=100$ \textbf{C4-RealNewsLike} prompts ($200$ prompt tokens, $200$ generated tokens, batch size $1$, single H200, bf16, single-prompt HuggingFace \texttt{model.generate}). Means $\pm$ standard deviations are computed after a $2$-prompt warm-up.}
  \label{tab:times}
  \begin{tabular}{lrr}
    \toprule
    \textbf{Method} & \textbf{Generation (ms)} & \textbf{Detection (ms)} \\
    \midrule
    \multicolumn{3}{l}{\textbf{Qwen3-4B-Base}} \\
    \gaussmark                                    & $3371.3 \pm 8.1$  & $26.2 \pm 0.2$  \\
    \textbf{\marktune\ (ours)}    & $3422.8 \pm 5.7$  & $27.0 \pm 0.2$  \\
    KGW                                           & $3714.6 \pm 11.0$ & $69.7 \pm 0.2$  \\
    Gumbel-max                                    & $3716.9 \pm 5.4$  & $175.7 \pm 0.9$ \\
    SynthID                                       & $3971.9 \pm 10.2$ & $35.5 \pm 0.2$  \\
    \midrule
    \multicolumn{3}{l}{\textbf{Llama2-7B}} \\
    \gaussmark                                    & $2421.8 \pm 6.8$ & $18.1 \pm 0.0$ \\
    \textbf{\marktune\ (ours)}    & $2404.9 \pm 5.3$ & $18.1 \pm 0.0$ \\
    KGW                                           & $2654.3 \pm 3.1$ & $62.3 \pm 0.1$ \\
    Gumbel-max                                    & $2555.0 \pm 4.0$ & $42.4 \pm 0.2$ \\
    SynthID                                       & $2882.9 \pm 7.2$ & $35.0 \pm 0.1$ \\
    \bottomrule
  \end{tabular}
\end{table}

\section{Ablation Study}\label{sec:ablation}
We isolate the roles of Gaussian initialization, the watermark reward, and quality regularization, and examine the sensitivity to GRPO hyperparameters and the selected parameter block. Unless otherwise specified, these experiments use Qwen3-4B-Base and C4-RealNewsLike, with TPR measured at 1\% FPR.

\subsection{Gaussian initialization}\label{app:initialization}
To isolate the role of Gaussian soft activation, we compare \marktune\ with direct GRPO from the clean base model under two CE coefficients. All trained models use the same 200-step protocol and differ only in initialization and the CE coefficient. Table~\ref{tab:direct_grpo} shows that direct GRPO with $\lambda=0.01$ achieves high detectability but produces repetitive, low-quality text, consistent with reward hacking. Increasing the coefficient to $0.1$ largely restores generation quality, at the cost of lower detection power.

% Source: https://openreview.net/forum?id=42tU2vcTxH&noteId=JqxkSCo1yP

\begin{table}[!htbp]

\centering\small
\caption{Gaussian initialization ablation on Qwen3-4B-Base/C4-RealNewsLike. TPR is measured at 1\% FPR. PPL subscripts denote standard deviations.}
\label{tab:direct_grpo}
\begin{adjustbox}{max width=\textwidth}
\begin{tabular}{lccccc}
\toprule
Method & AUC & TPR & PPL & Seq-rep-3 & MAUVE \\
\midrule
Clean base model & 0.500 & 0.010 & $4.92\err{1.34}$ & 0.04 & 0.98 \\
\marktune (\gaussmark init., CE $=0.01$) & 0.995 & 0.964 & $5.04\err{1.37}$ & 0.03 & 0.97 \\
Direct GRPO (CE $=0.01$) & 0.994 & 0.952 & $6.21\err{1.79}$ & 0.16 & 0.81 \\
Direct GRPO (CE $=0.1$) & 0.984 & 0.812 & $5.01\err{1.33}$ & 0.04 & 0.97 \\
\bottomrule
\end{tabular}
\end{adjustbox}
\end{table}

Table~\ref{tab:direct_grpo_lora} evaluates the same models under subsequent LoRA fine-tuning. Direct GRPO loses detection power more rapidly: after 1500 steps, TPR falls to $0.496$ and $0.384$ for the two CE coefficients, compared with $0.730$ for \marktune. Together, these controls support Gaussian initialization as a useful starting point for quality-preserving watermark optimization and robustness to continued model adaptation. A computable detection reward alone does not supply this initialization: KGW and Gumbel-max inject their signals at decoding time rather than through a soft perturbation of the model weights.

% Source: https://openreview.net/forum?id=42tU2vcTxH&noteId=JqxkSCo1yP

\begin{table}[!htbp]

\centering\small
\caption{Detection robustness of \marktune\ and direct GRPO under LoRA fine-tuning. Entries are TPR at 1\% FPR.}
\label{tab:direct_grpo_lora}
\begin{adjustbox}{max width=\textwidth}
\begin{tabular}{lccc}
\toprule
Fine-tuning steps & \marktune & Direct GRPO (CE $=0.01$) & Direct GRPO (CE $=0.1$) \\
\midrule
0 & 0.964 & 0.952 & 0.812 \\
300 & 0.892 & 0.814 & 0.698 \\
900 & 0.784 & 0.668 & 0.522 \\
1500 & 0.730 & 0.496 & 0.384 \\
\bottomrule
\end{tabular}
\end{adjustbox}
\end{table}

\subsection{Regularization term and coefficient}\label{sec:reg_ablation}
Table~\ref{tab:ablation_reg} highlights the effect of the regularization term. Without regularization, the model aggressively amplifies the watermark signal but at a cost to generation quality. Introducing cross-entropy or KL regularization restores perplexity and diversity while maintaining strong detectability, with cross-entropy with $\lambda=0.01$ offering the most balanced performance and thus serving as our default choice in Sections~\ref{sec:main_result}~and~\ref{sec:robust}. This shows that \marktune's gains over \gaussmark\ do not come from simply increasing watermark strength, but from guiding weight updates in a watermark-sensitive direction while staying close to the underlying high-quality text generation distribution.

\begin{table}[!htbp]

\centering
\caption{Ablation study on regularization hyperparameters. ``Base'' denotes the watermarked model trained using \marktune\ without a regularization term. Configuration highlighted in \textbf{bold} indicates Pareto-optimal performance.}
\label{tab:ablation_reg}
\renewcommand{\arraystretch}{0.9}

\begin{tabular}{lcccccc}
\toprule
\textbf{Method} & \textbf{AUC$\uparrow$} & \textbf{TPR$\uparrow$} & \textbf{PPL$\downarrow$} & \textbf{Seq-rep-3$\downarrow$} & \textbf{MAUVE$\uparrow$}\\
\midrule
Base       & 0.999 & 0.992 & $5.88\err{1.55}$ & 0.06 & 0.89\\
+ CE-0.005 & 0.997 & 0.976 & $5.35\err{1.43}$ & 0.05 & 0.94\\
\textbf{+ CE-0.01} & \textbf{0.995} & \textbf{0.964} & $\mathbf{5.04\err{1.37}}$ & \textbf{0.03} & \textbf{0.97}\\
+ CE-0.05 & 0.989 & 0.902 & $4.97\err{1.31}$ & 0.03 & 0.97\\
+ KL-0.001 & 0.997 & 0.986 & $5.69\err{1.56}$ & 0.06 & 0.93\\
+ KL-0.005 & 0.993 & 0.930 & $5.36\err{1.41}$ & 0.06 & 0.95\\
+ KL-0.01 & 0.986 & 0.854 & $5.13\err{1.32}$ & 0.04 & 0.97\\
\bottomrule
\end{tabular}
\end{table}

\subsection{Choice of quality anchor}\label{app:quality_roles}
A central design choice in \marktune\ is to regularize toward an external high-quality distribution (cross-entropy on human text or KL to a stronger oracle), rather than toward the base model as in standard RLHF/RLVR. The motivation is that, for model-embedded watermarking, increasing the watermark detectability inherently moves the trained model away from the base model; if the regularizer also penalizes deviation from the base model, the reward and the regularizer become directly antagonistic and this penalty can limit improvements over the existing \gaussmark\ frontier. Anchoring instead to an external high-quality reference decouples the two objectives, making room for a strictly better trade-off.
The fixed clean reference defines the statistic, whereas the external anchor defines the language-model regularization loss. Neither human-text cross-entropy nor base-model KL is a complete quality measure; an external anchor need not dominate the base model on every task.

To validate this, we replace $\mathcal{L}_{\mathrm{reg}}$ with the standard RL-style anchor $\mathrm{D_{KL}}\!\left(p_{\theta(\xi_\sigma)}(\cdot\mid x)\,\|\,p_{\theta_{\mathrm{ref}}}(\cdot\mid x)\right)$ (KL to the frozen base model), keeping all other hyperparameters identical to the Qwen3-4B-Base configuration in Table~\ref{tab:qwen_setting}. We track two checkpoints chosen to match the quality and the detectability of \marktune\ (CE-human), respectively. Results on \textbf{C4-RealNewsLike} are reported in Table~\ref{tab:ablation_anchor}. Anchoring to the base model yields essentially no improvement over \gaussmark\ at matched quality, and only modest detectability gains at the cost of perceptible quality degradation. In contrast, anchoring to an external high-quality distribution produces a strictly better point on the trade-off frontier.

\begin{table}[!htbp]

\centering
\caption{Ablation on the regularization anchor (Qwen3-4B-Base, \textbf{C4-RealNewsLike}). KL-to-base-model regularization fails to improve over \gaussmark\ at matched quality and only matches \marktune\ (CE-human) by sacrificing generation quality.}
\label{tab:ablation_anchor}
\renewcommand{\arraystretch}{1.0}
\resizebox{\textwidth}{!}{%
\begin{tabular}{lccccc}
\toprule
\textbf{Method} & \textbf{AUC$\uparrow$} & \textbf{TPR$\uparrow$} & \textbf{PPL$\downarrow$} & \textbf{Seq-rep-3$\downarrow$} & \textbf{MAUVE$\uparrow$}\\
\midrule
\gaussmark & 0.983 & 0.788 & $5.16\err{1.54}$ & 0.04 & 0.98\\
\marktune\ + KL-to-base (quality-matched) & 0.986 & 0.796 & $5.09\err{1.52}$ & 0.03 & 0.98\\
\marktune\ + KL-to-base (detectability-matched) & 0.993 & 0.940 & $5.46\err{1.60}$ & 0.03 & 0.94\\
\textbf{\marktune\ + CE-human (default)} & \textbf{0.995} & \textbf{0.964} & $\mathbf{5.04\err{1.37}}$ & \textbf{0.03} & \textbf{0.97}\\
\bottomrule
\end{tabular}}
\end{table}

\paragraph{Role of the regularization corpus.}
OpenWebText supplies both the RL prompts and the text for the CE regularizer. We use it as a regularization distribution; we neither search for nor curate a corpus that improves the pretrained base model. The anchor need not be globally better than the base model. Its role is to constrain quality loss while the watermark reward selects useful directions for increasing detectability. The local Fisher-geometry analysis in Appendix~\ref{app:marktune_analysis} explains how these objectives can be compatible even near a strong quality reference.

\paragraph{Quality-only fine-tuning.}
To measure the opportunity cost of watermarking under the same training procedure, we fine-tune the clean, unwatermarked base model on the regularization corpus using CE alone for 200 steps. Table~\ref{tab:quality_only} compares this control with the clean base model and \marktune. CE-only fine-tuning yields no clear overall quality improvement under PPL, Seq-rep-3, and MAUVE. Thus the observed watermarking gains cannot be explained simply by access to a corpus that improves the base model.

% Source: https://openreview.net/forum?id=42tU2vcTxH&noteId=IjRncTKrY1

\begin{table}[!htbp]

\centering\small
\caption{Generation quality of the clean base model, matched quality-only fine-tuning, and \marktune\ on Qwen3-4B-Base/C4-RealNewsLike. PPL subscripts denote standard deviations.}
\label{tab:quality_only}
\begin{adjustbox}{max width=\textwidth}
\begin{tabular}{lccc}
\toprule
Method & PPL & Seq-rep-3 & MAUVE \\
\midrule
No Watermark (base model) & $4.92\err{1.34}$ & 0.04 & 0.98 \\
No Watermark (CE-only FT) & $4.97\err{1.37}$ & 0.03 & 0.98 \\
\marktune & $5.04\err{1.37}$ & 0.03 & 0.97 \\
\bottomrule
\end{tabular}
\end{adjustbox}
\end{table}

These generation-quality measures are distinct from KL divergence to the base model in Proposition~\ref{prop:tv_bound}, which characterizes \gaussmark's distributional distortion. The regularizer guides optimization; oracle PPL, repetition, MAUVE, and downstream task performance evaluate its effect on utility.

\subsection{Removing the watermark reward}
We further isolate the role of the watermark reward by ablating it from the training objective entirely, i.e., dropping the first term of \eqref{eq:overall_opt} so that the objective reduces to minimizing $\mathcal{L}_{\mathrm{reg}}=\mathcal{L}_{\mathrm{ce}}$. Concretely, we initialize a strongly perturbed \gaussmark\ model with $\sigma=1.6$ on Qwen3-4B-Base so that its starting detectability roughly matches that of \marktune, and then perform full-parameter fine-tuning on \textbf{OpenWebText} using cross-entropy alone, reusing the optimizer and learning-rate schedule of Table~\ref{tab:qwen_setting}. This is the natural model-embedded analogue of the LoRA fine-tuning attack in Section~\ref{sec:robust} but applied without any LoRA constraint. As shown in Table~\ref{tab:ablation_no_rwm}, the quality-driven fine-tuning steadily moves the model back toward the base distribution and erodes the watermark. After 1000 steps, the resulting model is Pareto-dominated by \marktune\ on both axes (TPR $0.726$ vs. $0.964$, PPL $5.37$ vs. $5.04$). This confirms that $\mathcal{R}_{\mathrm{wm}}$ is necessary: without it, cross-entropy optimization treats the perturbation as quality damage to be undone, while \marktune\ uses $\mathcal{R}_{\mathrm{wm}}$ to steer updates along watermark-sensitive directions that preserve, rather than dissipate, the signal.

\begin{table}[!htbp]

\centering
\caption{Removing the watermark reward $\mathcal{R}_{\mathrm{wm}}$ during fine-tuning (Qwen3-4B-Base, \textbf{C4-RealNewsLike}). Starting from \gaussmark\ with $\sigma=1.6$, cross-entropy-only fine-tuning progressively erases the watermark; \marktune\ Pareto-dominates the result on both detectability and quality.}
\label{tab:ablation_no_rwm}
\renewcommand{\arraystretch}{1.0}
\begin{tabular}{lccccc}
\toprule
\textbf{Method} & \textbf{AUC$\uparrow$} & \textbf{TPR$\uparrow$} & \textbf{PPL$\downarrow$} & \textbf{Seq-rep-3$\downarrow$} & \textbf{MAUVE$\uparrow$}\\
\midrule
No Watermark & 0.500 & 0.010 & $4.92\err{1.34}$ & 0.04 & 0.98\\
\gaussmark\ ($\sigma=1.6$, init) & 0.995 & 0.970 & $5.79\err{1.68}$ & 0.05 & 0.93\\
w/o $\mathcal{R}_{\mathrm{wm}}$ (CE only, 200 steps) & 0.990 & 0.888 & $5.66\err{1.64}$ & 0.05 & 0.93\\
w/o $\mathcal{R}_{\mathrm{wm}}$ (CE only, 1000 steps) & 0.973 & 0.726 & $5.37\err{1.49}$ & 0.03 & 0.95\\
\textbf{\marktune\ (default)} & \textbf{0.995} & \textbf{0.964} & $\mathbf{5.04\err{1.37}}$ & \textbf{0.03} & \textbf{0.97}\\
\bottomrule
\end{tabular}
\end{table}

\subsection{GRPO sampling hyperparameters}
As for the impact of GRPO sampling hyperparameters on \marktune, Table~\ref{tab:ablation_grpo} indicates that the method is generally insensitive to the choice of group size. In practice, a group size of 8 provides the best balance between performance and computational cost. The sampling temperature should follow the model-specific settings commonly used for high-quality text generation, as deviations from these values either reduce output diversity or increase perplexity.

\begin{table}[!htbp]

\centering
\caption{Ablation study on GRPO sampling hyperparameters group size (``\textbf{G}'') and temperature (``\textbf{T}''). Configuration highlighted in \textbf{bold} indicates Pareto-optimal performance.}
\label{tab:ablation_grpo}
\renewcommand{\arraystretch}{1.0}

\begin{tabular}{cccccccc}
\toprule
\textbf{G} & \textbf{T} & \textbf{AUC$\uparrow$} & \textbf{TPR$\uparrow$} & \textbf{PPL$\downarrow$} & \textbf{Seq-rep-3$\downarrow$} & \textbf{MAUVE$\uparrow$}\\
\midrule
\textbf{8}  & \textbf{0.7} & \textbf{0.995} & \textbf{0.964} & $\mathbf{5.04\err{1.37}}$ & \textbf{0.03} & \textbf{0.97}\\
4  & 0.7 & 0.992 & 0.900 & $5.07\err{1.38}$ & 0.03 & 0.98\\
16 & 0.7 & 0.996 & 0.968 & $5.11\err{1.42}$ & 0.03 & 0.97\\
8  & 0.5 & 0.983 & 0.820 & $4.98\err{1.28}$ & 0.06 & 0.92\\
8  & 1.0 & 0.992 & 0.940 & $6.21\err{2.26}$ & 0.03 & 0.94\\
\bottomrule
\end{tabular}
\end{table}

\subsection{Alternative watermark rewards}\label{app:alter_reward}
Although our analysis and main experiments instantiate \marktune\ with the \gaussmark\ test statistic as the watermark reward $\mathcal{R}_{\mathrm{wm}}$, the framework itself does not depend on this specific choice. As stated in Remark~\ref{rmk:reward}, a broad class of computable detection statistics can in principle serve as the reward. To probe this generality, we replace the reward in \eqref{eq:overall_opt} with the KGW $z$-score~\citep{2023KGW} ($\gamma=0.5$, $\delta=0.6$, prefix length $k=1$) while reusing the Qwen3-4B-Base optimizer and regularization hyperparameters of Table~\ref{tab:qwen_setting}. Table~\ref{tab:ablation_kgw_reward} reports the resulting model on \textbf{C4-RealNewsLike} alongside two pure-KGW baselines that bracket its detectability. With the KGW reward, \marktune\ improves over the stronger KGW configuration ($\delta=0.8$) in both detectability (AUC, TPR) and perplexity, indicating that the framework's quality-preserving optimization geometry is not specific to the GaussMark statistic. Note that the GaussMark-rewarded \marktune\ still outperforms its KGW-rewarded counterpart, which we attribute to the cleaner statistical and theoretical inheritance available for \gaussmark.

\begin{table}[!htbp]

\centering
\caption{\marktune\ with the KGW $z$-score as reward (Qwen3-4B-Base, \textbf{C4-RealNewsLike}). \marktune\ (KGW) improves on KGW ($\delta=0.8$) in both detectability and quality.}
\label{tab:ablation_kgw_reward}
\renewcommand{\arraystretch}{1.0}
\resizebox{\textwidth}{!}{%
\begin{tabular}{lccccc}
\toprule
\textbf{Method} & \textbf{AUC$\uparrow$} & \textbf{TPR$\uparrow$} & \textbf{PPL$\downarrow$} & \textbf{Seq-rep-3$\downarrow$} & \textbf{MAUVE$\uparrow$}\\
\midrule
No Watermark & 0.500 & 0.010 & $4.92\err{1.34}$ & 0.04 & 0.98\\
KGW ($\gamma=0.5,\,\delta=0.6$) & 0.935 & 0.513 & $4.95\err{1.44}$ & 0.04 & 0.98\\
KGW ($\gamma=0.5,\,\delta=0.8$) & 0.987 & 0.861 & $5.14\err{1.50}$ & 0.05 & 0.98\\
\textbf{\marktune\ (KGW reward)} & \textbf{0.991} & \textbf{0.906} & $\mathbf{5.09\err{1.49}}$ & 0.06 & \textbf{0.98}\\
\marktune\ (\gaussmark\ reward, default) & 0.995 & 0.964 & $5.04\err{1.37}$ & 0.03 & 0.97\\
\bottomrule
\end{tabular}}
\end{table}

\subsection{Choice of watermarked parameters}\label{app:parameter_sensitivity}
\marktune\ inherits the layer and parameter matrix selected by the \gaussmark\ search. Because it optimizes the same detector statistic, this provides a natural candidate without exhaustively retraining \marktune\ for every parameter block. We examine nearby late-layer MLP projections on both model families in Table~\ref{tab:parameter_sensitivity}. The evaluated up-projection modules have comparable quality--detectability trade-offs, whereas the down-projection alternatives have weaker detection and no clear quality advantage. These comparisons support the late-layer up-projection choice within this candidate set; they do not establish global optimality.

% Sources: PKP4 Table 2 (NgVlQbtaF5), fPsW Table 4 (dzeq6nk5at).

\begin{table}[!htbp]

\centering\small
\caption{Watermarking performance for nearby parameter matrices. TPR is measured at 1\% FPR; PPL subscripts denote standard deviations.}
\label{tab:parameter_sensitivity}
\begin{tabular}{lccc}
\toprule
Target & TPR & PPL & Seq-rep-3 \\
\midrule
\multicolumn{4}{c}{\textbf{Qwen3-4B-Base}} \\
\midrule
L26 \texttt{mlp.up\_proj} & 0.822 & $5.28\err{1.81}$ & 0.04 \\
L28 \texttt{mlp.up\_proj} (default) & 0.788 & $5.16\err{1.54}$ & 0.04 \\
L30 \texttt{mlp.up\_proj} & 0.808 & $5.47\err{1.88}$ & 0.05 \\
L28 \texttt{mlp.down\_proj} & 0.610 & $5.66\err{1.86}$ & 0.05 \\
\midrule
\multicolumn{4}{c}{\textbf{Llama2-7B}} \\
\midrule
L28 \texttt{mlp.up\_proj} & 0.650 & $6.70\err{1.97}$ & 0.01 \\
L30 \texttt{mlp.up\_proj} (default) & 0.712 & $6.78\err{1.91}$ & 0.01 \\
L30 \texttt{mlp.down\_proj} & 0.586 & $7.12\err{2.01}$ & 0.01 \\
\bottomrule
\end{tabular}
\end{table}

The search evaluates \gaussmark\ on 500 prompts per candidate, without retraining \marktune\ for each candidate. The complete search takes 66 minutes for Qwen3-4B-Base and 41 minutes for Llama2-7B on a single H200 GPU.

\FloatBarrier
\section{Supplemental Results}\label{app:supp}
This section extends the main evaluation with low-FPR detection, additional baselines, decoder and length sensitivity, and text- and model-level attacks.

\subsection{Detection at ultra low false-positive rates}\label{app:tail_detection}
AUC summarizes ranking quality across thresholds, but neither AUC nor TPR at 1\% FPR characterizes the extreme left tail of the ROC curve. We therefore also report mean and median $\log_{10}p$ under the alternative and TPR at FPR@$10^{-4}$. For \gaussmark\ and \marktune, the one-sided p-value in Eq.~\eqref{eq:analytic_pvalue} is uniform under the key-independent null by Appendix~\ref{app:proof_2}.

% Source: https://openreview.net/forum?id=42tU2vcTxH&noteId=71FHXb1aHW

\begin{table}[!htbp]

\centering\small
\caption{Low-FPR detection and generation quality on Qwen3-4B-Base/C4-RealNewsLike. Brackets give 95\% bootstrap confidence intervals over 500 prompts. PPL subscripts denote standard deviations.}
\label{tab:tail_detection}
\begin{adjustbox}{max width=\textwidth}
\begin{tabular}{lcccccc}
\toprule
Method & Mean $\log_{10}p$ & Median $\log_{10}p$ & TPR at FPR $10^{-4}$ & PPL & Seq-rep-3 & MAUVE \\
\midrule
KGW-1 & -6.106 & -5.916 & 0.806 [0.770,0.840] & $5.51\err{1.56}$ & 0.04 & 0.98 \\
KGW-D & -3.320 & -3.062 & 0.302 [0.268,0.336] & $5.54\err{1.53}$ & 0.04 & 0.98 \\
\gaussmark & -4.265 & -3.923 & 0.484 [0.440,0.528] & $5.16\err{1.54}$ & 0.04 & 0.98 \\
\marktune & -6.014 & -5.871 & 0.800 [0.764,0.836] & $5.04\err{1.37}$ & 0.03 & 0.97 \\
\bottomrule
\end{tabular}
\end{adjustbox}
\end{table}

\marktune\ improves on \gaussmark\ and KGW-D in both log-p-value summaries and low-FPR detection power. At FPR $10^{-4}$, its TPR is $0.800$, compared with $0.484$ for \gaussmark\ and $0.302$ for KGW-D. \marktune\ and KGW-1 perform similarly in this comparison; the reported results do not establish a statistically significant quality advantage for either method. KGW-1 remains an inference-time reference that requires control of decoding, whereas \marktune\ supports open-weight deployment.

\subsection{Multi-key deployment} \label{app:multi_key}

In a realistic deployment, a provider may release multiple watermarked copies of the same base model, each with a different secret key. The relevant statistical question is whether keys collide: when the detector verifies one copy with the wrong key, does the test still trigger? In \gaussmark, each key $\xi_\sigma$ is sampled as an isotropic Gaussian vector in a high-dimensional parameter subspace. Two independent draws are therefore nearly orthogonal with overwhelming probability: their cosine similarity scales as $\mathcal{O}(1/\sqrt{d_r})$, which is negligible whenever the watermark subspace dimension $d_r$ is large (the configurations used in this paper have $d_r \ge 10^7$). \marktune\ inherits this property exactly because the key sampling distribution is unchanged.

To validate this empirically, we sample ten independent keys $\{\xi_i\}_{i=1}^{10}$ on Qwen3-4B-Base and run detection on text generated by the \marktune-trained model that was watermarked with a different key $\xi^\star$. Table~\ref{tab:multi_key} reports the resulting TPR@1\%FPR on \textbf{C4-RealNewsLike}. Every mismatched key produces a TPR within Monte Carlo fluctuation of the nominal level $\alpha=0.01$, confirming that cross-key false positives in a multi-key deployment are negligible at this scale; combined with the fact that each detector uses only its own key and the reference model, this implies that the matched-key detection power for any given key is unaffected by the simultaneous presence of additional independently sampled keys.

\begin{table}[!htbp]
\centering
\caption{Cross-key detection on Qwen3-4B-Base/\textbf{C4-RealNewsLike}: each cell reports TPR@1\%FPR when the detector uses a key $\xi_i$ different from the key $\xi^\star$ used to watermark the generator. All entries lie at or near the nominal $\alpha=0.01$ level.}
\label{tab:multi_key}
\renewcommand{\arraystretch}{1.0}
\resizebox{\textwidth}{!}{%
\begin{tabular}{lcccccccccc}
\toprule
\textbf{Mismatched key} & $\xi_1$ & $\xi_2$ & $\xi_3$ & $\xi_4$ & $\xi_5$ & $\xi_6$ & $\xi_7$ & $\xi_8$ & $\xi_9$ & $\xi_{10}$\\
\midrule
TPR@1\%FPR & 0.012 & 0.008 & 0.014 & 0.014 & 0.010 & 0.012 & 0.012 & 0.014 & 0.010 & 0.012\\
\bottomrule
\end{tabular}}
\end{table}

\subsection{Additional model-embedded baseline}\label{app:christ_baseline}
We compare against the training-free model-embedded watermark of \citet{2024provably}. Because Qwen3-4B-Base has no output-layer bias, we introduce a vocabulary-sized output bias parameter to apply the method, following its original design as closely as possible. Table~\ref{tab:christ_baseline} evaluates three perturbation strengths alongside the model-embedded baselines.

% Source: https://openreview.net/forum?id=42tU2vcTxH&noteId=NgVlQbtaF5

\begin{table}[!htbp]

\centering\small
\caption{Model-embedded watermarking methods on Qwen3-4B-Base/C4-RealNewsLike. TPR is measured at 1\% FPR; PPL subscripts denote standard deviations.}
\label{tab:christ_baseline}
\begin{adjustbox}{max width=\textwidth}
\begin{tabular}{lcccc}
\toprule
Method & AUC & TPR & PPL & Seq-rep-3 \\
\midrule
Christ ($\epsilon=0.1$) & 0.672 & 0.026 & $5.17\err{1.60}$ & 0.05 \\
Christ ($\epsilon=0.3$) & 0.930 & 0.308 & $5.25\err{1.57}$ & 0.05 \\
Christ ($\epsilon=0.6$) & 0.993 & 0.884 & $5.60\err{1.74}$ & 0.07 \\
KGW-D & 0.981 & 0.762 & $5.54\err{1.53}$ & 0.04 \\
\gaussmark & 0.983 & 0.788 & $5.16\err{1.54}$ & 0.04 \\
\marktune & 0.995 & 0.964 & $5.04\err{1.37}$ & 0.03 \\
\bottomrule
\end{tabular}
\end{adjustbox}
\end{table}

Increasing the output-bias perturbation improves detectability but raises perplexity and repetition. Among the evaluated model-embedded methods, \marktune\ achieves the highest AUC and TPR with the lowest PPL and Seq-rep-3.

\subsection{Decoding strategy and output length}\label{app:decoder_length}
\marktune\ does not require a watermark-specific decoding rule, but changing the decoder changes the output distribution and can affect both detectability and quality. Table~\ref{tab:decoder_sweep} compares multinomial sampling at three temperatures, nucleus sampling, and top-$k$ sampling under the same evaluation protocol. Detection remains strong across all five settings. Higher temperature increases diversity at the cost of perplexity.

% Source: https://openreview.net/forum?id=42tU2vcTxH&noteId=NgVlQbtaF5

\begin{table}[!htbp]

\centering\small
\caption{\marktune\ under standard decoding strategies on Qwen3-4B-Base/C4-RealNewsLike. TPR is measured at 1\% FPR; PPL subscripts denote standard deviations.}
\label{tab:decoder_sweep}
\begin{adjustbox}{max width=\textwidth}
\begin{tabular}{lccccc}
\toprule
Decoding strategy & AUC & TPR & PPL & Seq-rep-3 & MAUVE \\
\midrule
Multinomial ($T=0.7$) (default) & 0.995 & 0.964 & $5.04\err{1.37}$ & 0.03 & 0.97 \\
Multinomial ($T=0.5$) & 0.994 & 0.942 & $3.91\err{1.00}$ & 0.07 & 0.94 \\
Multinomial ($T=1.0$) & 0.996 & 0.982 & $12.36\err{4.66}$ & 0.01 & 0.91 \\
Nucleus ($T=0.7,\ p=0.9$) & 0.993 & 0.952 & $4.20\err{1.12}$ & 0.04 & 0.96 \\
Top-$k$ ($T=0.7,\ k=50$) & 0.993 & 0.944 & $4.98\err{1.37}$ & 0.03 & 0.97 \\
\bottomrule
\end{tabular}
\end{adjustbox}
\end{table}

We next vary the generation length while holding the model, prompts, key, and decoding strategy fixed. Table~\ref{tab:length_sweep} shows that detection improves with more generated tokens and then saturates: TPR rises from $0.414$ at 50 tokens to $0.964$ at 200 tokens and remains near $0.99$ at 400--500 tokens. Longer outputs provide more evidence for detection, while also exhibiting greater repetition.

% Source: https://openreview.net/forum?id=42tU2vcTxH&noteId=NgVlQbtaF5

\begin{table}[!htbp]

\centering\small
\caption{Detection and generation quality across output lengths. TPR is measured at 1\% FPR; PPL subscripts denote standard deviations.}
\label{tab:length_sweep}
\begin{adjustbox}{max width=\textwidth}
\begin{tabular}{lcccc}
\toprule
Generated tokens & AUC & TPR & PPL & Seq-rep-3 \\
\midrule
50 & 0.907 & 0.414 & $6.50\err{2.75}$ & 0.01 \\
100 & 0.971 & 0.736 & $5.97\err{1.79}$ & 0.02 \\
200 (default) & 0.995 & 0.964 & $5.04\err{1.37}$ & 0.03 \\
300 & 0.997 & 0.986 & $4.60\err{1.26}$ & 0.05 \\
400 & 0.999 & 0.992 & $4.07\err{1.16}$ & 0.06 \\
500 & 1.000 & 0.990 & $3.75\err{1.03}$ & 0.08 \\
\bottomrule
\end{tabular}
\end{adjustbox}
\end{table}
\subsection{Robustness to text editing}\label{app:text_edits}
\paragraph{Token-edit strength.}
We evaluate random insertion, word deletion, and synonym substitution at seven edit ratios from $0.02$ to $0.5$. Table~\ref{tab:edit_sweep} reports TPR at 1\% FPR on Qwen3-4B-Base/C4-RealNewsLike. Detection decreases with the edit ratio for both methods, and \marktune\ retains higher TPR than \gaussmark\ throughout the evaluated grid. At large deletion or substitution ratios, both signals become weak.

% Source: https://openreview.net/forum?id=42tU2vcTxH&noteId=eaVDxZHYxX

\begin{table}[!htbp]

\centering\small
\caption{TPR at 1\% FPR under seven token-edit ratios on Qwen3-4B-Base/C4-RealNewsLike.}
\label{tab:edit_sweep}
\begin{adjustbox}{max width=\textwidth}
\begin{tabular}{llccccccc}
\toprule
Method & Attack & 0.02 & 0.05 & 0.1 & 0.2 & 0.3 & 0.4 & 0.5 \\
\midrule
\gaussmark & Random insertion & 0.842 & 0.802 & 0.738 & 0.518 & 0.358 & 0.220 & 0.126 \\
\gaussmark & Word deletion & 0.800 & 0.722 & 0.626 & 0.366 & 0.138 & 0.060 & 0.048 \\
\gaussmark & Synonym substitution & 0.800 & 0.748 & 0.612 & 0.384 & 0.264 & 0.138 & 0.100 \\
\marktune & Random insertion & 0.914 & 0.900 & 0.858 & 0.734 & 0.604 & 0.434 & 0.310 \\
\marktune & Word deletion & 0.896 & 0.850 & 0.766 & 0.484 & 0.226 & 0.100 & 0.068 \\
\marktune & Synonym substitution & 0.912 & 0.870 & 0.766 & 0.552 & 0.292 & 0.168 & 0.110 \\
\bottomrule
\end{tabular}
\end{adjustbox}
\end{table}

\paragraph{Semantic-preserving attacks.}
Random edits control the amount of surface modification, but can substantially degrade fluency and meaning. We therefore also evaluate Dipper paraphrasing and roundtrip translation, which better preserve the utility of the text. Table~\ref{tab:semantic_edits} reports their mean realized token edit ratios, measured by normalized token-level Levenshtein distance.

% Source: https://openreview.net/forum?id=42tU2vcTxH&noteId=eaVDxZHYxX

\begin{table}[!htbp]

\centering\small
\caption{Mean realized token edit ratio under Dipper paraphrasing and roundtrip translation.}
\label{tab:semantic_edits}
\begin{adjustbox}{max width=\textwidth}
\begin{tabular}{lccc}
\toprule
Method & Dipper-1 & Dipper-2 & Translate \\
\midrule
KGW-1 & 0.470 & 0.856 & 0.488 \\
KGW-2 & 0.501 & 0.871 & 0.521 \\
Gumbel-max & 0.462 & 0.868 & 0.443 \\
SynthID & 0.456 & 0.761 & 0.390 \\
KGW-D & 0.482 & 0.879 & 0.507 \\
\gaussmark & 0.450 & 0.761 & 0.409 \\
\marktune & 0.437 & 0.788 & 0.427 \\
\bottomrule
\end{tabular}
\end{adjustbox}
\end{table}

The realized edit ratios are not identical across methods. In particular, Dipper-2 changes \marktune\ outputs less than KGW-2 and KGW-D outputs ($0.788$ versus $0.871$ and $0.879$). Relative to \gaussmark, \marktune\ has a lower edit ratio under Dipper-1 but higher ratios under Dipper-2 and translation. The results therefore characterize robustness to these rewriting pipelines; they do not isolate robustness at matched realized edit counts or establish a general tendency for \marktune\ text to resist modification.

\paragraph{Global key and attack quality.}
\gaussmark\ and \marktune\ use a single Gaussian key in weight space and a full-sequence gradient-alignment statistic. They have no prefix hash window, so a prefix-window sweep and its associated security--robustness trade-off do not apply. Table~\ref{tab:robust_extend} supplements the main robustness results with post-attack perplexity. Deletion and substitution substantially degrade quality, whereas paraphrasing preserves lower perplexity. These quality measures complement the edit-strength sweep when assessing attacks intended to preserve useful text.

\subsection{Quantization, pruning, and model merging}\label{app:model_modifications}
We complement LoRA fine-tuning and logit-based distillation with three open-weight model modifications, following the evaluation protocol of \citet{2025towards}: INT8 quantization, WANDA pruning at $\rho=0.2$, and SLERP merging at $t=0.5$. Table~\ref{tab:model_modifications} reports the detection and generation quality of \marktune\ before and after each attack.

% Source: https://openreview.net/forum?id=42tU2vcTxH&noteId=JqxkSCo1yP

\begin{table}[!htbp]

\centering\small
\caption{\marktune\ under open-weight attacks. TPR is measured at 1\% FPR; PPL subscripts denote standard deviations.}
\label{tab:model_modifications}
\begin{adjustbox}{max width=\textwidth}
\begin{tabular}{lccc}
\toprule
Attack & AUC & TPR & PPL \\
\midrule
No attack & 0.995 & 0.964 & $5.04\err{1.37}$ \\
Quantization (INT8) & 0.981 & 0.754 & $5.27\err{1.46}$ \\
Pruning (WANDA, $\rho=0.2$) & 0.988 & 0.838 & $5.49\err{1.50}$ \\
Model merging (SLERP, $t=0.5$) & 0.949 & 0.584 & $5.18\err{1.42}$ \\
\bottomrule
\end{tabular}
\end{adjustbox}
\end{table}

All three attacks reduce detection power while preserving generation quality close to the unmodified model. \marktune\ retains nontrivial detectability, with TPR $0.754$ after quantization, $0.838$ after pruning, and $0.584$ after merging. These results support robustness within the evaluated settings. Full-parameter fine-tuning of the final \marktune\ model and cross-architecture distillation remain outside this evaluation.

\subsection{Robustness to logit-based distillation} \label{app:distill}

We evaluate whether the watermark embedded by \marktune\ survives logit-based knowledge distillation from a watermarked teacher~\citep{2024gu}. Following the protocol of \citet{2024gu}, we instantiate teacher and student as the same architecture, Llama2-7B, with the teacher initialized from a \marktune-trained checkpoint and the student initialized from a fresh copy of the unwatermarked Llama2-7B; the student is then trained on \textbf{OpenWebText} for 5000 steps with AdamW, batch size 64, sequence length 512, peak learning rate $1\times 10^{-5}$, and a cosine schedule with linear warmup. Detection is run on the student outputs using the original watermark key. Table~\ref{tab:distill} shows that the distilled same-size student retains the bulk of the watermark signal (TPR $0.982 \to 0.916$ at matched perplexity), suggesting that the \marktune\ signal is encoded in directions that are partially preserved under output-distribution matching, rather than in surface artifacts that distillation would erase. We caution that this experiment isolates the same-architecture setting and does not test smaller-student distillation, where capacity mismatches between teacher and student may either dampen or amplify the signal. An adversary, lacking the secret key, cannot directly verify whether distillation has removed the watermark; fully training-resistant watermarking, including the smaller-student case, remains an open problem that we discuss further in Appendix~\ref{app:limitations}.

\begin{table}[!t]
\centering
\caption{Same-architecture logit-based distillation: a Llama2-7B student distilled from a \marktune-trained Llama2-7B teacher (\textbf{C4-RealNewsLike}). The student inherits substantial detectability without quality loss; the smaller-student case is left to future work.}
\label{tab:distill}
\renewcommand{\arraystretch}{1.0}
\begin{tabular}{lccccc}
\toprule
\textbf{Method} & \textbf{AUC$\uparrow$} & \textbf{TPR$\uparrow$} & \textbf{PPL$\downarrow$} & \textbf{Seq-rep-3$\downarrow$} & \textbf{MAUVE$\uparrow$}\\
\midrule
No Watermark & 0.500 & 0.010 & $6.55\err{1.80}$ & 0.01 & 0.98\\
\marktune\ (teacher) & 0.997 & 0.982 & $6.75\err{1.96}$ & 0.01 & 0.96\\
\textbf{Distilled student} & \textbf{0.991} & \textbf{0.916} & $6.81\err{1.92}$ & \textbf{0.01} & \textbf{0.96}\\
\bottomrule
\end{tabular}
\end{table}

\subsection{Watermark reward during training}
The rollout-averaged watermark reward during \marktune\ training is shown below for both model families.

\begin{figure}[!htbp]

  \centering
  \begin{subfigure}{0.48\linewidth}
    \centering
    \includegraphics[width=\linewidth]{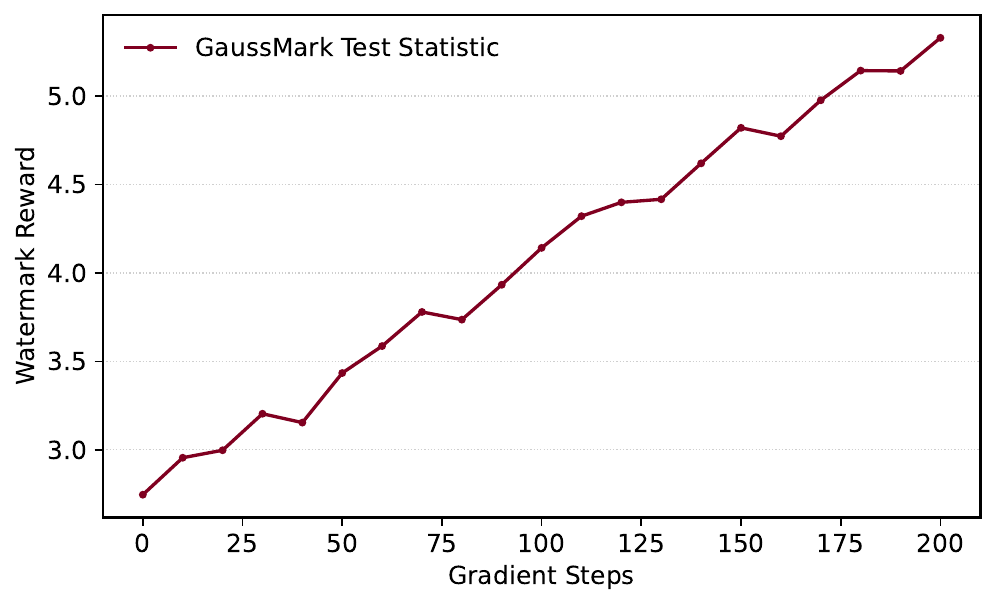}
    \caption{Qwen3-4B-Base}
  \end{subfigure}
  \hfill
  \begin{subfigure}{0.48\linewidth}
    \centering
    \includegraphics[width=\linewidth]{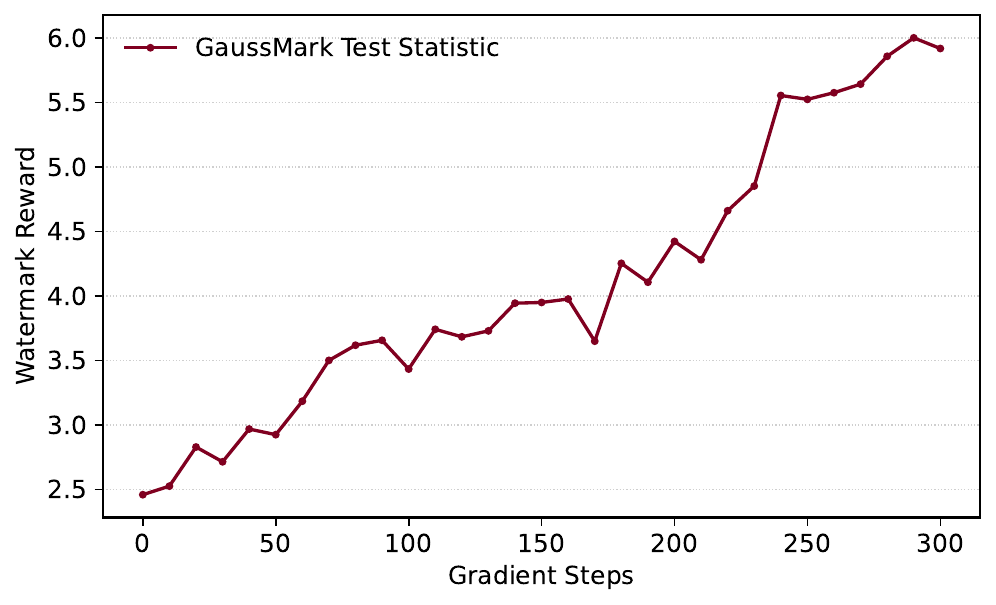}
    \caption{Llama2-7B}
  \end{subfigure}
  \caption{Watermark reward (\gaussmark\ test statistic) averaged over rollouts during \marktune\ training.}
\end{figure}

\FloatBarrier
\section{Limitations and Future Directions}\label{app:limitations}
\paragraph{Theoretical scope.}
The reverse-KL result in Appendix~\ref{app:population_power} analyzes an idealized converged population policy under a linear-softmax model and an explicit strong-concavity assumption. It is not a convergence or detection-power guarantee for the particular finite-step CE-regularized GRPO model. The local analysis in Appendix~\ref{app:marktune_analysis} complements this result by explaining the quality--reward geometry of CE regularization.

\marktune\ opens up several promising directions for future work. The first concerns adversaries who go beyond text editing and continued fine-tuning. Although the distillation experiment in Appendix~\ref{app:distill} already shows that the \marktune\ signal partially survives output-distribution matching, no model-embedded watermark is currently known to remain robust against an adversary who can retrain a model from scratch on the watermarked outputs~\citep{2024gu,2025towards}. A natural next step is to couple \marktune\ with sharpness-aware optimization~\citep{foret2021sharpness}, steering the watermarked weights toward broad, flat basins of the loss landscape in which parameter-level perturbations (continued fine-tuning, distillation, and partial retraining alike) leave the signal essentially undisturbed. The second direction concerns the choice of detection statistic. The present work instantiates the framework with the \gaussmark\ test statistic and supplies a KGW-reward variant in Appendix~\ref{sec:ablation}. The same on-policy reward template should extend to semantic-level watermarks such as PMark~\citep{huo2026pmark}, where the open challenge is to recalibrate the detection-time resampling procedure so that its statistical guarantees continue to hold under the fine-tuned generator. Sample-level correction techniques developed for generative-model sampling, such as soft Metropolis-Hastings reweighting~\citep{feng2025soft}, offer a complementary refinement to the on-policy reward estimate. Together, these directions push \marktune\ toward a watermarking framework that is robust under stronger attack models while applying to a broader class of detection statistics.

\section{Examples of Generated Texts} \label{app:examples}
In this section, we present sample text completions generated by the Qwen3-4B-Base and Llama2-7B models. Prompts are drawn from the \textbf{realnewslike} split of the \textbf{C4} dataset and truncated to 200 tokens. For Qwen3-4B-Base, we use a sampling temperature of 0.7 to generate 200-token outputs, and for Llama2-7B, we use a temperature of 1.0. Note that the resulting prompts differ slightly between the two models due to differences in their tokenizers. In the examples below, the text under the heading ``Base Model'' corresponds to a random completion produced by the original, unwatermarked model, while the text under ``Watermarked Model'' shows a random completion generated by the model trained with \marktune\ using the hyperparameters detailed in Appendix~\ref{app:imp_details}.

\subsection{Example 1}

%----------------- Input Prompt box -----------------
\begin{tcolorbox}[
  enhanced,
  breakable,
  colback=gray!10,
  colframe=black,
  boxrule=0.8pt,
  left=3mm,right=3mm,top=2mm,bottom=2mm,
  before skip=5pt,
  after skip=5pt
]
\textbf{Input Prompt:}
Dubai: The Dubai Financial Market (DFM) witnessed another round of selling on Monday as traders remained cautious ahead of a long holidays. The DFM general index closed 0.44 per cent lower at 2,727.41, as it managed to stay above the support level of 2.706 that has proved critical after the gauged bounced back from these levels twice. Emaar Properties closed 1.2 per cent lower at Dh4.74. Emaar shares are struggling to find a bottom as the stock has been trading near to its previous low. A close over Dh4.85 shall be considered bullish, and ignite recovery to Dh4.95/5.20 in the short term, Shiv Prakash, senior analyst with First Abu Dhabi Bank Securities said in a note. Emaar's movements appear to be concerning at this time of year where investors are expected to be taking positions instead
\end{tcolorbox}

%----------------- Qwen3-4B-Base Generations box -----------------
\begin{tcolorbox}[
  enhanced,
  breakable,
  colback=purple!5,
  colframe=black,
  boxrule=0.8pt,
  colbacktitle=purple!20,
  coltitle=black,
  fonttitle=\bfseries,
  title={Qwen3-4B-Base Generations},
  left=3mm,right=3mm,top=2mm,bottom=2mm,
  before skip=5pt,
  after skip=5pt
]

\textbf{Base Model:} of holding shares ahead of the upcoming holidays, he added. Shares of Emaar had touched Dh4.45 earlier, before recovering to Dh4.50. DFM-listed companies in the real-estate and construction sector declined on the day, with Emaar, Dubai Land Department, Dubai World, Dubai Holding and Dubai Islamic Bank all trading lower. The DFM's general index has been trading below 2.706 since August 12. The stock index had breached this support level in July, as it rallied from its 13-month low of 2.616, which was recorded on July 12. The index has bounced twice after testing the support level. On the other hand, shares of Al Maktoum International Airport had surged 3.8 per cent to Dh1.084 after the airport operator said on Monday that it has entered into a Memorandum of Understanding (MoU) with the International Air
\begin{center}
\fbox{\textbf{z-score: -0.6454; p-value: 0.7406; perplexity: 5.0251}}
\end{center}

% \vspace{0.4em}
\noindent\hdashrule{\linewidth}{0.5pt}{2pt 2pt}
% \vspace{0.1em}

\textbf{Watermarked Model:} of selling off. We believe the stock has priced in the worst of the recent troubles and given the firm's strong fundamentals, we now expect a bounce towards Dh5.20, he said. We continue to expect Emaar's shares to rally towards Dh5.20 in the next couple of months as the firm has proved resilient in the past. The stock looks cheap given the discount to the firm's intrinsic value, and we now expect the shares to rally towards Dh5.20 in the next couple of months, he added. The real estate sector was the biggest loser of the day, with the sector index down by 1.2 per cent at 3,124.13. Almarai (ALM) declined 2.2 per cent to Dh5.85, while Alrosa (ARO) shed 1.9 per cent to Dh1.53 and Alshaya (ALY) dropped 1.
\begin{center}
\fbox{\textbf{z-score: 4.5114; p-value: 3.2187e-06; perplexity: 4.625}}
\end{center}
\end{tcolorbox}

\begin{tcolorbox}[
  enhanced,
  breakable,
  colback=gray!10,
  colframe=black,
  boxrule=0.8pt,
  left=3mm,right=3mm,top=2mm,bottom=2mm,
  before skip=5pt,
  after skip=5pt
]
\textbf{Input Prompt:}
Dubai: The Dubai Financial Market (DFM) witnessed another round of selling on Monday as traders remained cautious ahead of a long holidays. The DFM general index closed 0.44 per cent lower at 2,727.41, as it managed to stay above the support level of 2.706 that has proved critical after the gauged bounced back from these levels twice. Emaar Properties closed 1.2 per cent lower at Dh4.74. Emaar shares are struggling to find a bottom as the stock has been trading near to its previous low. A close over Dh4.85 shall be considered bullish, and ignite recovery to Dh4.95/5.20 in the short term, Shiv Prakash, senior analyst with First Abu Dhabi Bank Securities
\end{tcolorbox}

%----------------- Llama2-7B Generations box -----------------
\begin{tcolorbox}[
  enhanced,
  breakable,
  colback=orange!5,
  colframe=black,
  boxrule=0.8pt,
  colbacktitle=orange!20,
  coltitle=black,
  fonttitle=\bfseries,
  title={Llama2-7B Generations},
  left=3mm,right=3mm,top=2mm,bottom=2mm,
  before skip=5pt,
  after skip=5pt
]

\textbf{Base Model:} (FASB Securities), said.The stock has traded in a narrow range for almost two years, during which time the stock did gain some momentum but failed to penetrate the key resistance point of Dh4.98/5. Since then, there has been no meaningful progress and the stock price has been in a free fall mode, he said.Emirates Telecommunication Group was the most active stock traded on the index with nearly 9.619 million shares, closely followed by Arabtec with over 8.219 million shares.On value terms, Emaar Properties contributed roughly Dh153.32 million to the market capitalisation, followed by Arabtec with Dh43.25 million and Emirates NBD with Dh20.69 million.In sectoral terms, the real estate sector led the top gainers with a surge
\begin{center}
\fbox{\textbf{z-score: 0.5016; p-value: 0.3079; perplexity: 5.2253}}
\end{center}

% \vspace{0.4em}
\noindent\hdashrule{\linewidth}{0.5pt}{2pt 2pt}
% \vspace{0.1em}

\textbf{Watermarked Model:} (FASB Securities), told the Gulf News.Dubai Islamic Bank, meanwhile, was another leading index heavyweight to gain ground on Monday. Shares of the bank picked up 0.77 per cent to close at Dh6.42.On the DFM and other markets in the region, we're looking at the 19th [August] as a down day for traders and investors, said Ahmed Saqr, assistant director, financial control department at Abu Dhabi Islamic Bank.After the Eid week, many traders and investors will have just a few days left in August to trade, he added.Investors will take a wait-and-watch approach, he said.Their sentiment towards the markets may turn negative at the beginning of September, because of the US September rates hike talks. This could negatively
\begin{center}
\fbox{\textbf{z-score: 4.3386; p-value: 7.1526e-06; perplexity: 5.4523}}
\end{center}
\end{tcolorbox}

\subsection{Example 2}

%----------------- Input Prompt box -----------------
\begin{tcolorbox}[
  enhanced,
  breakable,
  colback=gray!10,
  colframe=black,
  boxrule=0.8pt,
  left=3mm,right=3mm,top=2mm,bottom=2mm,
  before skip=5pt,
  after skip=5pt
]
\textbf{Input Prompt:}
MADRID: As Cristiano Ronaldo enjoyed his time off at the French Riviera, his Portugal team-mates were busy handing Italy another disappointing result in the UEFA Nations League. Portugal did not need Ronaldo, who skipped international duty to recharge after a busy summer, to beat Italy 1-0 in Europe's newest soccer competition. Elsewhere, Turkey mounted a remarkable comeback from two goals down against Sweden, while Kosovo made history with their first competitive win since being admitted to UEFA in 2016.The European champions easily outplayed Italy, which had opened its Nations League campaign with a disappointing 1-1 home draw against Poland in its first competitive match under coach Roberto Mancini. Andre Silva scored on a counterattack early in the second half to give Portugal the Group 3 win in the top-tier League A. The hosts had several chances to extend their lead at the Stadium of Light in Lisbon, while Italy created little and relied mostly on set pieces. The Nations League gives
\end{tcolorbox}

%----------------- Qwen3-4B-Base Generations box -----------------
\begin{tcolorbox}[
  enhanced,
  breakable,
  colback=purple!5,
  colframe=black,
  boxrule=0.8pt,
  colbacktitle=purple!20,
  coltitle=black,
  fonttitle=\bfseries,
  title={Qwen3-4B-Base Generations},
  left=3mm,right=3mm,top=2mm,bottom=2mm,
  before skip=5pt,
  after skip=5pt
]

\textbf{Base Model:} a fresh impetus to the tournament that replaced the European Championship qualifiers. The top-tier league is divided into four pots, with each group consisting of four teams. The top two teams in each group will qualify for the 2020 European Championship in France. Turkey, which is not currently ranked in the top 20, came back from two goals down to beat Sweden 3-2 at the Group 3 venue in Istanbul. The Swedes took the lead in the 25th minute through a penalty, but the hosts had already equalised a minute later on a goal from a corner. Turkey took the lead through a goal from Ali Kocak, but the Swedes equalised through a Hakan Yakin penalty, before Kocak scored in the 70th minute to make the match up until then a 2-2 stalemate. The hosts took the lead through a goal from Cemre Onay in the 89th minute to win
\begin{center}
\fbox{\textbf{z-score: 0.9452; p-value: 0.1722; perplexity: 4.9626}}
\end{center}

% \vspace{0.4em}
\noindent\hdashrule{\linewidth}{0.5pt}{2pt 2pt}
% \vspace{0.1em}

\textbf{Watermarked Model:} the teams a chance to warm up for next summer's Euros, where they will face each other. But with Italy's defence vulnerable, the likes of Paulo Dybala and the injured Mario Balotelli were hardly needed. We should have won the game, but we didn't get the best of it, Mancini said after the defeat. This was a good way to start the Nations League, but we now have to find a way to get better because it's not enough to beat Portugal. Italy's next opponents in the group stage are the likes of Poland and Wales, but the 2020 European champions are already planning for the big games after the Nations League ends. We'll start preparing for the next phase of the year, because we know we're going to face some tough opponents in the next few weeks, Mancini said. We'll focus on those games and try to make the most of the opportunities we have. The Italians will also be looking to bounce back
\begin{center}
\fbox{\textbf{z-score: 4.3033; p-value: 8.4043e-06; perplexity: 5.4245}}
\end{center}

\end{tcolorbox}

%----------------- Input Prompt box -----------------
\begin{tcolorbox}[
  enhanced,
  breakable,
  colback=gray!10,
  colframe=black,
  boxrule=0.8pt,
  left=3mm,right=3mm,top=2mm,bottom=2mm,
  before skip=5pt,
  after skip=5pt
]
\textbf{Input Prompt:}
MADRID: As Cristiano Ronaldo enjoyed his time off at the French Riviera, his Portugal team-mates were busy handing Italy another disappointing result in the UEFA Nations League. Portugal did not need Ronaldo, who skipped international duty to recharge after a busy summer, to beat Italy 1-0 in Europe's newest soccer competition. Elsewhere, Turkey mounted a remarkable comeback from two goals down against Sweden, while Kosovo made history with their first competitive win since being admitted to UEFA in 2016. The European champions easily outplayed Italy, which had opened its Nations League campaign with a disappointing 1-1 home draw against Poland in its first competitive match under coach Roberto Mancini. Andre Silva scored on a counterattack early in the second half to give Portugal the Group 3 win in the top-tier League A. The hosts
\end{tcolorbox}

%----------------- Llama2-7B Generations box -----------------
\begin{tcolorbox}[
  enhanced,
  breakable,
  colback=orange!5,
  colframe=black,
  boxrule=0.8pt,
  colbacktitle=orange!20,
  coltitle=black,
  fonttitle=\bfseries,
  title={Llama2-7B Generations},
  left=3mm,right=3mm,top=2mm,bottom=2mm,
  before skip=5pt,
  after skip=5pt
]

\textbf{Base Model:} dominated for most of the game by playing Portugal more openly and allowing more space in midfield, but they struggled to take advantage. Portugal's previous victory over Italy was 1-0 in the Euro 2016 semifinals when it lifted the tournament's trophy. Italy played most of that game with 10 men after defender Giorgio Chiellini received a red card. Portugal was the only one of the 21 European members of the competition which had already played two games in Group 3 prior to the opener against Italy. The host nation, Spain and Poland had one game apiece and will continue the campaign this week. Italy lost its first two fixtures in the Nations League and missed out on qualification for the 2018 World Cup in Russia, while Portugal is one of Europe's major powers. It finished third in its group at the World Cup, losing
\begin{center}
\fbox{\textbf{z-score: -1.7278; p-value: 0.9579; perplexity: 5.8205}}
\end{center}

% \vspace{0.4em}
\noindent\hdashrule{\linewidth}{0.5pt}{2pt 2pt}
% \vspace{0.1em}

\textbf{Watermarked Model:} dominated for most of the game by patiently counterattacking when Italy tried to play with the ball. The Portuguese had their most threatening chances when Matteo Politano was sent off after 30 minutes for a rash challenge on Joao Mario outside the penalty area. Ricardo Horta shot wide from a good position. Italy dominated ball play in the midfield and defense after Portugal scored in the 53rd, but lacked the final attacking touch. Portugal, which lost its top spot in the FIFA rankings last week, remained two points above Sweden in Group 2 after both teams began their campaigns. The Portuguese host Sweden in September, while Italy travels to Poland. Sweden opened the match with two goals from former Premier League player Emil Forsberg, who scored in the first and 55th minutes. Sweden looked to have the title in its sights as they pressed for a
\begin{center}
\fbox{\textbf{z-score: 3.5713; p-value: 0.0002; perplexity: 6.2521}}
\end{center}

\end{tcolorbox}

\subsection{Example 3}

%----------------- Input Prompt box -----------------
\begin{tcolorbox}[
  enhanced,
  breakable,
  colback=gray!10,
  colframe=black,
  boxrule=0.8pt,
  left=3mm,right=3mm,top=2mm,bottom=2mm,
  before skip=5pt,
  after skip=5pt
]
\textbf{Input Prompt:}
The New York Times on Thursday printed a special section featuring portraits of nearly every member of the record number of female lawmakers in Congress. A viral Twitter thread from the Times' design editor, Josh Crutchmer, includes a video of the covers the night before publication that has been viewed over 320,000 times as of Thursday morning. The photos, which were taken on Capitol Hill over five days by Times photographers Elizabeth D. Herman and Celeste Sloman, are meant as a testament to what power looks like in 2019, according to the paper. Like the work of Kehinde Wiley, who painted Barack Obama Barack Hussein ObamaChina, Russia, Iran rise in Latin America as US retreats Castro wants to follow Obama's lead on balancing presidency with fatherhood Trump's regulatory rollback boosts odds of a financial crisis MORE's official presidential portrait, these photographs evoke the imagery we are used to seeing in the halls of power, but place people not previously seen
\end{tcolorbox}

%----------------- Qwen3-4B-Base Generations box -----------------
\begin{tcolorbox}[
  enhanced,
  breakable,
  colback=purple!5,
  colframe=black,
  boxrule=0.8pt,
  colbacktitle=purple!20,
  coltitle=black,
  fonttitle=\bfseries,
  title={Qwen3-4B-Base Generations},
  left=3mm,right=3mm,top=2mm,bottom=2mm,
  before skip=5pt,
  after skip=5pt
]

\textbf{Base Model:} occupying them.\" The photographs are intended to show the power of women in our government, according to a Times source. The Times' print section, which ran the portraits on Thursday, also printed an article titled The 2020 Women's History Month Celebration. The article contains a list of female lawmakers and the causes they represent. The Times also published a companion piece featuring portraits of 12 women in Congress who have taken on more than just their usual roles. The article also included a video of the portraits, which had been viewed over 350,000 times as of Thursday morning, and a list of the women whose portraits appeared in the print section. The women profiled in the article include Rep. Jackie Speier, a California Democrat who has been battling cancer since 2015 and who, as of 2019, is one of only 13 women in Congress who have survived cancer. The article also features Rep. Barbara
\begin{center}
\fbox{\textbf{z-score: 0.6038; p-value: 0.2729; perplexity: 5.9870}}
\end{center}

% \vspace{0.4em}
\noindent\hdashrule{\linewidth}{0.5pt}{2pt 2pt}
% \vspace{0.1em}

\textbf{Watermarked Model:} in that position. It's the kind of thing the paper has been doing for decades. In 2016, it published a series of its own power portraits featuring the leaders of every major political party during the presidential primaries. This year, the paper has been highlighting the stories of the women running for president and the women in the other major races. But this year's section is particularly special because it's the first time the paper has managed to get the coverage it deserves for its work. The section is the culmination of a massive effort by the paper's editorial and design teams to make the most of the historic moment. There are so many great stories in this election cycle, but we wanted to make sure we told the ones that matter most, said the paper's executive editor, Dean Baquet. The section is just the beginning of what the paper hopes will be a long-term commitment to telling the stories of women in politics. The paper has already announced plans to launch a new
\begin{center}
\fbox{\textbf{z-score: 3.5408; p-value: 0.0002; perplexity: 5.0692}}
\end{center}

\end{tcolorbox}

%----------------- Input Prompt box -----------------
\begin{tcolorbox}[
  enhanced,
  breakable,
  colback=gray!10,
  colframe=black,
  boxrule=0.8pt,
  left=3mm,right=3mm,top=2mm,bottom=2mm,
  before skip=5pt,
  after skip=5pt
]
\textbf{Input Prompt:}
The New York Times on Thursday printed a special section featuring portraits of nearly every member of the record number of female lawmakers in Congress. A viral Twitter thread from the Times' design editor, Josh Crutchmer, includes a video of the covers the night before publication that has been viewed over 320,000 times as of Thursday morning. The photos, which were taken on Capitol Hill over five days by Times photographers Elizabeth D. Herman and Celeste Sloman, are meant as a testament to what power looks like in 2019, according to the paper. Like the work of Kehinde Wiley, who painted Barack Obama Barack Hussein ObamaChina, Russia, Iran rise in Latin America as US retreats Castro wants to follow Obama's lead on balancing presidency with fatherhood Trump's regulatory rollback boosts odds
\end{tcolorbox}

%----------------- Llama2-7B Generations box -----------------
\begin{tcolorbox}[
  enhanced,
  breakable,
  colback=orange!5,
  colframe=black,
  boxrule=0.8pt,
  colbacktitle=orange!20,
  coltitle=black,
  fonttitle=\bfseries,
  title={Llama2-7B Generations},
  left=3mm,right=3mm,top=2mm,bottom=2mm,
  before skip=5pt,
  after skip=5pt
]

\textbf{Base Model:} of successful Obamacare repeal MORE's official presidential portrait, our intent is to show just how diverse the new Congress is from all perspectives, Crutchmer tweeted. According to The Washington Post, the photos also include Rep. Alexandria Ocasio-Cortez, as well as the first Muslim woman elected to Congress. Rep. Ayanna Pressley, who won her election in Massachusetts and is the seat Ocasio-Cortez will take, is featured as the main photo on the cover, which reads Unity in Action. Rep. Ilhan Omar on Monday became the first Somali-American and one of the first Muslim women to serve in Congress. Lawmakers said upon arriving in Washington that they would be a force multiplier, as Omar said. In a statement to The Hill when photos first came out, Omar said the group is going to be
\begin{center}
\fbox{\textbf{z-score: 0.3792; p-value: 0.3522; perplexity: 6.2308}}
\end{center}
% \vspace{0.4em}
\noindent\hdashrule{\linewidth}{0.5pt}{2pt 2pt}
% \vspace{0.1em}

\textbf{Watermarked Model:} of nuclear incident MORE for the first time, this project honors women on the national stage, portraying them as they are, Times executive editor Dean Baquet said in a statement. One of the lawmakers, Rep. Deb Haaland Deb HaalandNational Archives to release Trump's passport application, other records on Friday: report Judy Chu is the next chairwoman of the Congressional Asian Pacific American Caucus Nadler's top priority is fighting 'epidemic of corruption in our political system' MORE (D-N.M.) who is the first Native American woman elected to Congress, said the photoshoot reflected the way women have been portrayed for centuries, both in media and in the public realm\" Baquet was in charge of the paper's decision to use the photo of Rep. Alexandria Ocasio-Cortez (D-N.Y.) in a sk
\begin{center}
\fbox{\textbf{z-score: 3.6125; p-value: 0.0001; perplexity: 6.2042}}
\end{center}
\end{tcolorbox}

\clearpage
\begin{table}[!htbp]

\centering
\small

\caption{Extension of Table~\ref{tab:robust} with perplexity (PPL) to validate attack plausibility.}
\renewcommand{\arraystretch}{1.10}

\resizebox{\textwidth}{!}{
\begin{tabular}{l|l|ccccccc}
\toprule
\textbf{Attack} & \textbf{Metric} & \textbf{KGW-1} & \textbf{KGW-2} & \textbf{Gumbel-max} & \textbf{SynthID} & \textbf{KGW-D} & \textbf{GaussMark} & \textbf{MarkTune} \\
\midrule
\multicolumn{9}{c}{\cellcolor{blkC}\textbf{Qwen3-4B-Base, T=0.7}} \\
\midrule
\multirow{3}{*}{\textbf{Del-0.2}}
  & \textbf{AUC$\uparrow$} & 0.881 & 0.975 & 0.993 & 0.987 & 0.947 & 0.908 & 0.921 \\
  & \textbf{TPR$\uparrow$} & 0.330 & 0.746 & 0.948 & 0.890 & 0.582 & 0.366 & 0.484 \\
  & \textbf{PPL$\downarrow$} & 20.50 & 25.19 & 18.82 & 19.15 & 21.23 & 19.49 & 19.39 \\
\midrule
\multirow{3}{*}{\textbf{Del-0.5}}
  & \textbf{AUC$\uparrow$} & 0.617 & 0.697 & 0.778 & 0.780 & 0.851 & 0.622 & 0.744 \\
  & \textbf{TPR$\uparrow$} & 0.040 & 0.082 & 0.162 & 0.094 & 0.196 & 0.048 & 0.068 \\
  & \textbf{PPL$\downarrow$} & 98.53 & 117.25 & 94.82 & 94.19 & 99.95 & 96.82 & 94.79 \\
\midrule
\multirow{3}{*}{\textbf{Sub-0.2}}
  & \textbf{AUC$\uparrow$} & 0.873 & 0.986 & 0.993 & 0.986 & 0.939 & 0.919 & 0.948 \\
  & \textbf{TPR$\uparrow$} & 0.324 & 0.826 & 0.922 & 0.846 & 0.578 & 0.384 & 0.552 \\
  & \textbf{PPL$\downarrow$} & 24.71 & 28.64 & 21.46 & 23.06 & 25.44 & 23.42 & 23.17 \\
\midrule
\multirow{3}{*}{\textbf{Sub-0.5}}
  & \textbf{AUC$\uparrow$} & 0.638 & 0.747 & 0.840 & 0.834 & 0.837 & 0.775 & 0.783 \\
  & \textbf{TPR$\uparrow$} & 0.048 & 0.106 & 0.348 & 0.182 & 0.190 & 0.100 & 0.110 \\
  & \textbf{PPL$\downarrow$} & 71.05 & 80.28 & 60.34 & 67.19 & 72.48 & 68.96 & 67.52 \\
\midrule
\multirow{3}{*}{\textbf{Dipper-1}}
  & \textbf{AUC$\uparrow$} & 0.771 & 0.884 & 0.945 & 0.941 & 0.855 & 0.915 & 0.977 \\
  & \textbf{TPR$\uparrow$} & 0.142 & 0.372 & 0.748 & 0.576 & 0.274 & 0.438 & 0.852 \\
  & \textbf{PPL$\downarrow$} & 6.51 & 7.30 & 6.07 & 6.18 & 6.56 &  6.24  & 6.16 \\
\midrule
\multirow{3}{*}{\textbf{Dipper-2}}
  & \textbf{AUC$\uparrow$} & 0.573 & 0.616 & 0.721 & 0.700 & 0.589 & 0.795 & 0.851 \\
  & \textbf{TPR$\uparrow$} & 0.026 & 0.038 & 0.146 & 0.074 & 0.030 & 0.138 & 0.254 \\
  & \textbf{PPL$\downarrow$} & 7.72 & 8.34 & 7.81 & 7.59 & 7.70 & 7.80 & 7.62 \\
\midrule
\multirow{3}{*}{\textbf{Translate}}
  & \textbf{AUC$\uparrow$} & 0.831 & 0.928 & 0.955 & 0.968 & 0.912 & 0.948 & 0.973 \\
  & \textbf{TPR$\uparrow$} & 0.242 & 0.552 & 0.846 & 0.818 & 0.404 & 0.574 & 0.858 \\
  & \textbf{PPL$\downarrow$} & 9.16 & 11.37 & 8.33 & 8.73 & 9.29 & 8.95 & 8.81 \\
\midrule
\multicolumn{9}{c}{\cellcolor{blkC}\textbf{Llama2-7B, T=1.0}} \\
\midrule
\multirow{3}{*}{\textbf{Del-0.2}}
  & \textbf{AUC$\uparrow$} & 0.853 & 0.975 & 0.993 & 0.979 & 0.914 & 0.835 & 0.928 \\
  & \textbf{TPR$\uparrow$} & 0.232 & 0.712 & 0.924 & 0.774 & 0.410 & 0.242 & 0.506 \\
  & \textbf{PPL$\downarrow$} & 23.71 & 27.13 & 21.66 & 22.55 & 23.54 & 22.86 & 22.49 \\
\midrule
\multirow{3}{*}{\textbf{Del-0.5}}
  & \textbf{AUC$\uparrow$} & 0.581 & 0.711 & 0.790 & 0.777 & 0.803 & 0.617 & 0.783 \\
  & \textbf{TPR$\uparrow$} & 0.024 & 0.060 & 0.224 & 0.114 & 0.148 & 0.036 & 0.080 \\
  & \textbf{PPL$\downarrow$} & 92.81 & 105.16 & 87.05 & 88.56 & 95.12 & 89.98 & 88.95 \\
\midrule
\multirow{3}{*}{\textbf{Sub-0.2}}
  & \textbf{AUC$\uparrow$} & 0.878 & 0.980 & 0.993 & 0.987 & 0.887 & 0.853 & 0.944 \\
  & \textbf{TPR$\uparrow$} & 0.300 & 0.756 & 0.950 & 0.862 & 0.406 & 0.228 & 0.578 \\
  & \textbf{PPL$\downarrow$} & 22.97 & 25.89 & 21.42 & 22.50 & 22.84 & 22.68 & 22.53 \\
\midrule
\multirow{3}{*}{\textbf{Sub-0.5}}
  & \textbf{AUC$\uparrow$} & 0.638 & 0.772 & 0.862 & 0.843 & 0.759 & 0.759 & 0.809 \\
  & \textbf{TPR$\uparrow$} & 0.030 & 0.126 & 0.296 & 0.140 & 0.144 & 0.098 & 0.122 \\
  & \textbf{PPL$\downarrow$} & 48.50 & 52.78 & 44.80 & 46.69 & 47.12 & 46.78 & 46.51 \\
\midrule
\multirow{3}{*}{\textbf{Dipper-1}}
  & \textbf{AUC$\uparrow$} & 0.757 & 0.862 & 0.950 & 0.943 & 0.814 & 0.857 & 0.961 \\
  & \textbf{TPR$\uparrow$} & 0.116 & 0.322 & 0.700 & 0.552 & 0.238 & 0.240 & 0.836 \\
  & \textbf{PPL$\downarrow$} & 8.33 & 8.91 & 7.90 & 8.11 & 8.19 & 8.13 & 8.09 \\
\midrule
\multirow{3}{*}{\textbf{Dipper-2}}
  & \textbf{AUC$\uparrow$} & 0.551 & 0.614 & 0.711 & 0.710 & 0.590 & 0.687 & 0.859 \\
  & \textbf{TPR$\uparrow$} & 0.014 & 0.032 & 0.148 & 0.084 & 0.026 & 0.060 & 0.270 \\
  & \textbf{PPL$\downarrow$} & 9.34 & 9.58 & 8.64 & 8.97 & 9.25 & 9.07 & 9.02 \\
\midrule
\multirow{3}{*}{\textbf{Translate}}
  & \textbf{AUC$\uparrow$} & 0.777 & 0.922 & 0.961 & 0.949 & 0.889 & 0.925 & 0.966 \\
  & \textbf{TPR$\uparrow$} & 0.146 & 0.448 & 0.796 & 0.612 & 0.324 & 0.436 & 0.834 \\
  & \textbf{PPL$\downarrow$} & 13.27 & 14.99 & 11.69 & 12.68 & 12.86 & 12.56 & 12.47 \\
\bottomrule
\end{tabular}}
\label{tab:robust_extend}
\end{table}

\clearpage
\section{AI Use Statement}\label{app:ai_use}

We used LLM-based tools for manuscript editing and literature-review assistance, and for implementation, refactoring, and documentation of the released code. All AI-assisted manuscript edits were reviewed by the authors for accuracy, appropriate attribution, and consistency with the underlying research. We take full responsibility for the final manuscript and released implementation, including the mathematical claims, proofs, experimental results, and analysis.

\section{Ethics Statement} \label{app:impact}

The research presented in this work adheres to the NeurIPS Code of Ethics. This paper does not involve crowdsourcing or research with human subjects, and introduces no new datasets containing personal data. \marktune\ aims to support provenance attribution, accountability for misuse, and intellectual-property protection for open-weight language models, where users can modify the decoding pipeline. Watermark detection nevertheless carries risks of false attribution and missed detection, and robustness is limited to the conditions evaluated in this work. Detection outcomes should therefore be interpreted with their statistical uncertainty and should not serve as the sole basis for punitive decisions or claims of ownership.

\clearpage
\section*{NeurIPS Paper Checklist}

\begin{enumerate}

\Needspace{5\baselineskip}
\item {\bf Claims}
    \item[] Question: Do the main claims made in the abstract and introduction accurately reflect the paper's contributions and scope?
    \item[] Answer: \answerYes{}
    \item[] Justification: The method and theoretical claims are developed in Section~\ref{sec:method} and Appendices~\ref{app:proofs} and~\ref{app:marktune_analysis}. Section~\ref{sec:exper} supports the reported quality--detectability trade-off, downstream utility, cross-dataset generalization, and robustness under the evaluated settings.

\Needspace{5\baselineskip}
\item {\bf Limitations}
    \item[] Question: Does the paper discuss the limitations of the work performed by the authors?
    \item[] Answer: \answerYes{}
    \item[] Justification: Appendix~\ref{app:limitations} discusses the scope of the population and local analyses, stronger attack models, and extensions to other detection statistics. Section~\ref{sec:exper} and Appendix~\ref{app:supp} specify the evaluated models, datasets, and attack settings.
    \item[] Guidelines:
    \begin{itemize}
        \item The answer NA means that the paper has no limitation while the answer No means that the paper has limitations, but those are not discussed in the paper.
        \item The authors are encouraged to create a separate ``Limitations'' section in their paper.
        \item The paper should point out any strong assumptions and how robust the results are to violations of these assumptions (e.g., independence assumptions, noiseless settings, model well-specification, asymptotic approximations only holding locally). The authors should reflect on how these assumptions might be violated in practice and what the implications would be.
        \item The authors should reflect on the scope of the claims made, e.g., if the approach was only tested on a few datasets or with a few runs. In general, empirical results often depend on implicit assumptions, which should be articulated.
        \item The authors should reflect on the factors that influence the performance of the approach. For example, a facial recognition algorithm may perform poorly when image resolution is low or images are taken in low lighting. Or a speech-to-text system might not be used reliably to provide closed captions for online lectures because it fails to handle technical jargon.
        \item The authors should discuss the computational efficiency of the proposed algorithms and how they scale with dataset size.
        \item If applicable, the authors should discuss possible limitations of their approach to address problems of privacy and fairness.
        \item While the authors might fear that complete honesty about limitations might be used by reviewers as grounds for rejection, a worse outcome might be that reviewers discover limitations that aren't acknowledged in the paper. The authors should use their best judgment and recognize that individual actions in favor of transparency play an important role in developing norms that preserve the integrity of the community. Reviewers will be specifically instructed to not penalize honesty concerning limitations.
    \end{itemize}

\Needspace{5\baselineskip}
\item {\bf Theory assumptions and proofs}
    \item[] Question: For each theoretical result, does the paper provide the full set of assumptions and a complete (and correct) proof?
    \item[] Answer: \answerYes{}
    \item[] Justification: Propositions~\ref{prop:tv_bound} and~\ref{prop:FPR_control}, Proposition~\ref{prop:gaussian_detectability}, and Theorems~\ref{thm:detect_power} and~\ref{thm:population_power} are stated with their assumptions and proved in Appendix~\ref{app:proofs}. The reverse-KL power result concerns an idealized population policy, while Appendix~\ref{app:marktune_analysis} gives the assumptions and derivations for the local CE analysis.

\Needspace{5\baselineskip}
\item {\bf Experimental result reproducibility}
    \item[] Question: Does the paper fully disclose all the information needed to reproduce the main experimental results of the paper to the extent that it affects the main claims and/or conclusions of the paper (regardless of whether the code and data are provided or not)?
    \item[] Answer: \answerYes{}
    \item[] Justification: Section~\ref{sec:exper} and Appendices~\ref{app:alg_details} and~\ref{app:imp_details} specify the training and evaluation protocols. The public implementation supplies model-specific configurations and data-preparation, training, and evaluation commands for the main experiments.

\Needspace{5\baselineskip}
\item {\bf Open access to data and code}
    \item[] Question: Does the paper provide open access to the data and code, with sufficient instructions to faithfully reproduce the main experimental results, as described in supplemental material?
    \item[] Answer: \answerYes{}
    \item[] Justification: Code, model-specific configurations, and installation and execution instructions are available at \url{https://github.com/zhyzmath/MarkTune-LLM-Watermarking}.
    \item[] Guidelines:
    \begin{itemize}
        \item The answer NA means that paper does not include experiments requiring code.
        \item Please see the NeurIPS code and data submission guidelines (\url{https://neurips.cc/public/guides/CodeSubmissionPolicy}) for more details.
        \item While we encourage the release of code and data, we understand that this might not be possible, so No is an acceptable answer. Papers cannot be rejected simply for not including code, unless this is central to the contribution (e.g., for a new open-source benchmark).
        \item The instructions should contain the exact command and environment needed to run to reproduce the results. See the NeurIPS code and data submission guidelines (\url{https://neurips.cc/public/guides/CodeSubmissionPolicy}) for more details.
        \item The authors should provide instructions on data access and preparation, including how to access the raw data, preprocessed data, intermediate data, and generated data, etc.
        \item The authors should provide scripts to reproduce all experimental results for the new proposed method and baselines. If only a subset of experiments are reproducible, they should state which ones are omitted from the script and why.
        \item At submission time, to preserve anonymity, the authors should release anonymized versions (if applicable).
        \item Providing as much information as possible in supplemental material (appended to the paper) is recommended, but including URLs to data and code is permitted.
    \end{itemize}

\Needspace{5\baselineskip}
\item {\bf Experimental setting/details}
    \item[] Question: Does the paper specify all the training and test details (e.g., data splits, hyperparameters, how they were chosen, type of optimizer) necessary to understand the results?
    \item[] Answer: \answerYes{}
    \item[] Justification: Section~\ref{sec:exper} and Appendix~\ref{app:imp_details} describe the training and evaluation settings, including the per-model hyperparameters in Tables~\ref{tab:qwen_setting} and~\ref{tab:llama_setting}.

\Needspace{5\baselineskip}
\item {\bf Experiment statistical significance}
    \item[] Question: Does the paper report error bars suitably and correctly defined or other appropriate information about the statistical significance of the experiments?
    \item[] Answer: \answerYes{}
    \item[] Justification: Table~\ref{tab:wm_distortion} reports perplexity means and standard deviations across $500$ evaluation prompts. Figures~\ref{fig:trade-off} and~\ref{fig:ft-attack} report bootstrap-estimated $90\%$ confidence intervals ($500$ resamples) on TPR@1\%FPR, with the procedure described in Appendix~\ref{app:imp_details}. These summaries quantify variation across evaluation samples.
    \item[] Guidelines:
    \begin{itemize}
        \item The answer NA means that the paper does not include experiments.
        \item The authors should answer Yes if the results are accompanied by error bars, confidence intervals, or statistical significance tests, at least for the experiments that support the main claims of the paper.
        \item The factors of variability that the error bars are capturing should be clearly stated (for example, train/test split, initialization, random drawing of some parameter, or overall run with given experimental conditions).
        \item The method for calculating the error bars should be explained (closed form formula, call to a library function, bootstrap, etc.)
        \item The assumptions made should be given (e.g., Normally distributed errors).
        \item It should be clear whether the error bar is the standard deviation or the standard error of the mean.
        \item It is OK to report 1-sigma error bars, but one should state it. The authors should preferably report a 2-sigma error bar than state that they have a 96\% CI, if the hypothesis of Normality of errors is not verified.
        \item For asymmetric distributions, the authors should be careful not to show in tables or figures symmetric error bars that would yield results that are out of range (e.g., negative error rates).
        \item If error bars are reported in tables or plots, the authors should explain in the text how they were calculated and reference the corresponding figures or tables in the text.
    \end{itemize}

\Needspace{5\baselineskip}
\item {\bf Experiments compute resources}
    \item[] Question: For each experiment, does the paper provide sufficient information on the computer resources (type of compute workers, memory, time of execution) needed to reproduce the experiments?
    \item[] Answer: \answerYes{}
    \item[] Justification: Appendix~\ref{app:cost} reports training on 4 NVIDIA H200 GPUs for approximately 8\,h (Qwen3-4B-Base) and 13.5\,h (Llama2-7B). Appendix~\ref{app:lat_setup} specifies the single-H200 latency benchmark, and Appendix~\ref{app:detection_resources} reports detector access and peak GPU memory requirements.

\Needspace{5\baselineskip}
\item {\bf Code of ethics}
    \item[] Question: Does the research conducted in the paper conform, in every respect, with the NeurIPS Code of Ethics \url{https://neurips.cc/public/EthicsGuidelines}?
    \item[] Answer: \answerYes{}
    \item[] Justification: Appendix~\ref{app:impact} discusses conformity with the NeurIPS Code of Ethics and the responsible interpretation of detection results. No human subjects or new personal-data datasets are involved.

\Needspace{5\baselineskip}
\item {\bf Broader impacts}
    \item[] Question: Does the paper discuss both potential positive societal impacts and negative societal impacts of the work performed?
    \item[] Answer: \answerYes{}
    \item[] Justification: Appendix~\ref{app:impact} discusses provenance attribution, misuse accountability, and intellectual-property protection, alongside false attribution, missed detection, and the risks of using watermark evidence as the sole basis for punitive decisions or ownership claims.

\Needspace{5\baselineskip}
\item {\bf Safeguards}
    \item[] Question: Does the paper describe safeguards that have been put in place for responsible release of data or models that have a high risk for misuse (e.g., pre-trained language models, image generators, or scraped datasets)?
    \item[] Answer: \answerNA{}
    \item[] Justification: The public repository distributes code and configurations, with no model weights or datasets. Pretrained models and datasets are obtained separately from their original providers under the applicable access conditions; Appendix~\ref{app:impact} discusses responsible use of detection results.

\Needspace{5\baselineskip}
\item {\bf Licenses for existing assets}
    \item[] Question: Are the creators or original owners of assets (e.g., code, data, models), used in the paper, properly credited and are the license and terms of use explicitly mentioned and properly respected?
    \item[] Answer: \answerYes{}
    \item[] Justification: The datasets, pretrained models, and baseline methods are credited in Section~\ref{sec:exper} and Appendix~\ref{app:imp_details}. The repository's \texttt{THIRD\_PARTY.md} identifies upstream software sources and versions, and the bundled license files retain the Apache-2.0 terms for these components. Model weights and datasets are obtained from their original providers under their respective licenses and access conditions.

\Needspace{5\baselineskip}
\item {\bf New assets}
    \item[] Question: Are new assets introduced in the paper well documented and is the documentation provided alongside the assets?
    \item[] Answer: \answerYes{}
    \item[] Justification: The public repository includes a README with installation, data preparation, and training/evaluation instructions, together with model-specific configurations. The \texttt{LICENSE} and \texttt{THIRD\_PARTY.md} files document the MIT license for MarkTune and the source attributions and licenses of third-party components.

\Needspace{5\baselineskip}
\item {\bf Crowdsourcing and research with human subjects}
    \item[] Question: For crowdsourcing experiments and research with human subjects, does the paper include the full text of instructions given to participants and screenshots, if applicable, as well as details about compensation (if any)?
    \item[] Answer: \answerNA{}
    \item[] Justification: The paper does not involve crowdsourcing nor research with human subjects.

\Needspace{5\baselineskip}
\item {\bf Institutional review board (IRB) approvals or equivalent for research with human subjects}
    \item[] Question: Does the paper describe potential risks incurred by study participants, whether such risks were disclosed to the subjects, and whether Institutional Review Board (IRB) approvals (or an equivalent approval/review based on the requirements of your country or institution) were obtained?
    \item[] Answer: \answerNA{}
    \item[] Justification: The paper does not involve crowdsourcing nor research with human subjects.

\Needspace{5\baselineskip}
\item {\bf Declaration of LLM usage}
    \item[] Question: Does the paper describe the usage of LLMs if it is an important, original, or non-standard component of the core methods in this research? Note that if the LLM is used only for writing, editing, or formatting purposes and does not impact the core methodology, scientific rigor, or originality of the research, declaration is not required.
    \item[] Answer: \answerNA{}
    \item[] Justification: LLM-based assistants were used for writing, literature review, and preparation of the released code, as disclosed in Appendix~\ref{app:ai_use}. These assistants are not components of the watermarking or evaluation algorithms.
    \item[] Guidelines:
    \begin{itemize}
        \item The answer NA means that the core method development in this research does not involve LLMs as any important, original, or non-standard components.
        \item Please refer to our LLM policy in the NeurIPS handbook for what should or should not be described.
    \end{itemize}

\end{enumerate}

\end{document}